\documentclass{article}

\usepackage[main, final]{neurips_2026}

\usepackage[utf8]{inputenc} \usepackage[T1]{fontenc}    \usepackage{hyperref}       \usepackage{url}            \usepackage{booktabs}       \usepackage{amsfonts}       \usepackage{nicefrac}       \usepackage{microtype}      \usepackage{xcolor}         
\usepackage{bbm}
\usepackage{amssymb}
\usepackage{mathtools}
\usepackage{amsthm}
\usepackage{todonotes}
\usepackage{subfig}

\usepackage{xspace}
\newcommand{\eg}{\textit{e.g.\@}\xspace}
\newcommand{\ie}{\textit{i.e.\@}\xspace}

\usepackage[bb=boondox]{mathalfa}
\usepackage{enumitem}
\usepackage{listings}

\hypersetup{
  colorlinks=true,
  breaklinks=true,
  urlcolor=blue,   linkcolor=blue,   citecolor=blue,
}

\usepackage[capitalize,noabbrev]{cleveref}

\usepackage{aliascnt}
\theoremstyle{plain}
\newtheorem{theorem}{Theorem}[section]
\newaliascnt{lemma}{theorem}
\newtheorem{lemma}[lemma]{Lemma}
\aliascntresetthe{lemma}
\newaliascnt{proposition}{theorem}

\aliascntresetthe{proposition}
\newaliascnt{corollary}{theorem}

\aliascntresetthe{corollary}
\theoremstyle{definition}
\newaliascnt{definition}{theorem}

\aliascntresetthe{definition}
\newaliascnt{assumption}{theorem}
\newtheorem{assumption}[assumption]{Assumption}
\aliascntresetthe{assumption}
\theoremstyle{remark}
\newaliascnt{remark}{theorem}

\aliascntresetthe{remark}

\crefname{theorem}{Theorem}{Theorems}
\Crefname{theorem}{Theorem}{Theorems}
\crefname{proposition}{Proposition}{Propositions}
\Crefname{proposition}{Proposition}{Propositions}
\crefname{lemma}{Lemma}{Lemmas}
\Crefname{lemma}{Lemma}{Lemmas}
\crefname{corollary}{Corollary}{Corollaries}
\Crefname{corollary}{Corollary}{Corollaries}
\crefname{definition}{Definition}{Definitions}
\Crefname{definition}{Definition}{Definitions}
\crefname{assumption}{Assumption}{Assumptions}
\Crefname{assumption}{Assumption}{Assumptions}
\crefname{remark}{Remark}{Remarks}
\Crefname{remark}{Remark}{Remarks}

\crefname{section}{Sec.}{Secs.}
\crefname{algorithm}{Alg.}{Algs.}
\crefname{appendix}{App.}{Apps.}
\crefname{definition}{Def.}{Defs.}
\crefname{table}{Table}{Tables}
\crefname{equation}{Eq.}{Eqs.}

\definecolor{color0}{HTML}{F6F6F4} \definecolor{color1}{HTML}{C27A7A} \definecolor{color2}{HTML}{B3DE69} \definecolor{color3}{HTML}{FC8D59} \definecolor{color4}{HTML}{FFD92F} \definecolor{color5}{HTML}{8C3B3B} \definecolor{color6}{HTML}{EEEEEE} \definecolor{color7}{HTML}{96B1C9} \definecolor{color8}{HTML}{4E6E8E} \definecolor{blueC}{HTML}{7896B3}
\definecolor{blueD}{HTML}{86A3BE}
\definecolor{blueE}{HTML}{96B1C9}
\definecolor{blueF}{HTML}{A4BDD1}
\usepackage{tikz}
\usetikzlibrary{shapes.geometric, positioning, calc}
\usepackage{tikz-3dplot}

\newcommand{\stopgrad}{\tikz[baseline]{
  \draw[thick,black] (-3pt,-1pt) -- (3pt,5pt);
  \draw[thick,black] (-3pt,-4pt) -- (3pt,2pt);}}

\newlength{\nodedist}
\newlength{\nodedistv}

\definecolor{cube}{HTML}{f7f75e}

\title{Temporal Consistency Improves Generalization in Contextual Offline Meta Reinforcement Learning}

\author{  Mohammadreza Nakhaei \\
    Aalto University \thanks{Part of this work was conducted during a research visit to Vrije Universiteit Amsterdam.}\\
  \texttt{mohammadreza.nakhaei@aalto.fi} \\
    \And
  Aidan Scannell \\
  Toyota Research Institute \thanks{Work done whilst at University of Edinburgh.} \\
  \texttt{aidan.scannell@tri.global} \\
  \AND
  Kevin Sebastian Luck \\
  Vrije Universiteit Amsterdam \\
    \texttt{k.s.luck@vu.nl} \\
  \And
  Joni Pajarinen \\
    Aalto University \\
              \texttt{joni.pajarinen@aalto.fi} \\
}

\begin{document}

\maketitle

\begin{abstract}
    Offline meta-reinforcement learning seeks to learn a policy that generalizes to new related tasks online. 
    Context-based methods infer a task representation from transition histories, yet learning an effective task representation without supervision remains challenging. 
    Existing methods relying on contrastive learning learn discriminative task representations, but fail to identify task-specific dynamics, while relying on reconstruction can be insufficient to model long-horizon dependencies, limiting generalization to new tasks.    
    We investigate the impact of temporal consistency in latent space on task representation learning, showing that enforcing multi-step predictions in latent space encourages task representations that are able to capture task-dependent dynamics while preventing representation collapse.    
    We provide theoretical analysis characterizing sources of error in value estimation and show through extensive experiments on MuJoCo, Contextual DeepMind Control, and MetaWorld benchmarks that temporal consistency significantly improves both zero-shot and few-shot generalization. Code is available at \href{https://github.com/MohammadrezaNakhaei/SPC}{https://github.com/MohammadrezaNakhaei/SPC}.
\end{abstract}

\section{Introduction}
One of the core challenges of reinforcement learning (RL) is \emph{generalization}, where a policy trained on one task typically performs poorly when applied to a related task.
Meta-reinforcement learning (meta-RL) addresses this issue by training on a distribution of tasks and learning a policy that can adapt quickly to new tasks \citep{maml, varibad, metarl_tutorial}.
However, most meta-RL methods require online interaction with numerous training tasks, which can be impractical in real-world scenarios.
Offline meta-RL (OMRL) instead leverages access to offline datasets of related tasks and aims to learn a generalizable policy with limited additional environment interaction at test time.

A common approach in OMRL is \emph{context encoding} \citep{focal, csro, gentle, unicorn, ertrl, idac, scrutinize}.
In these methods, a context encoder maps a history of transitions (the context) to a latent vector termed the \emph{task representation}.
Policy and value functions are conditioned on this task representation, which serves as an implicit task identifier without requiring explicit task labels or knowledge of the underlying task variation factors.
As a result, the quality of the task representation plays a central role in generalization to unseen tasks.

Many existing methods rely solely on contrastive learning to learn the task representation \citep{focal, corro, dora}.
While contrastive learning encourages task discrimination, it does not explicitly enforce predictive structure over time.
As a result, the context encoder may fail to encode the underlying task variation factors, reflected in the reward function and transition dynamics, into the task representation, thereby limiting generalization.
UNICORN \citep{unicorn} partially addresses this issue by including a one-step reconstruction objective as well, but one-step predictions might be insufficient to capture long-horizon, task-dependent dynamics.

Multi-step temporal consistency in latent space \citep{tcrl} forms a robust self-supervised learning signal for representation learning. 
In contrast to approaches that learn dynamics directly in observation or state space, world models project observations into a latent state representation and model dynamics within this latent manifold. 
They have been successfully used for decision-time planning and policy learning via imagination in model-based RL \citep{planet, muzero, tdmpc, hansen2024tdmpc2, dcmpc} and for representation learning in model-free RL \citep{tcrl, td7, general_modelfree}. 
Across these settings, a core theme is that self-supervised representation learning based on self-predictive RL \citep{schwarzer2021data} -- also known as temporal consistency -- provides a strong signal for learning latent state representations.
However, directly applying latent dynamics in the OMRL setting is challenging, as the model must infer task-dependent variations without access to explicit task labels. 

\paragraph{Contributions:} In this work, we show that enforcing temporal consistency through self-predictive representations by jointly training the context encoder and world model yields better task representations than existing methods relying on contrastive learning and reconstruction, aka, one-step prediction. 
We study the quality of the learned context encoder through the lens of disentanglement and representation metrics, including feature rank, matrix rank, and dormant neuron ratio.

Our main contributions are:
\begin{itemize}
    \item[\textbf{C1}] \textbf{Temporal consistency for task inference.}
    We show that enforcing latent temporal consistency during context encoding yields task representations that capture task variation factors more effectively than reconstruction-based objectives based on disentanglement scores.
    \item[\textbf{C2}] \textbf{Theoretical analysis.} 
    We formally characterize the sources of error in value estimation and motivate temporal consistency on the latent space for task representation learning. 
    \item[\textbf{C3}] \textbf{Extensive empirical evaluation.}
    Through experiments on MuJoCo, Contextual DeepMind Control, and Meta-World benchmarks, we analyze the task representation quality and show that incorporating temporal consistency into the context encoder training objective significantly enhances few-shot and zero-shot generalization.
\end{itemize}

\section{Background}
\paragraph{Context-based Offline Meta-RL}
In offline meta-RL, we consider a set of training tasks, each modeled as a Markov Decision Process (MDP)
$\mathcal{M}_i = \langle \mathcal{S}, \mathcal{A}, R_i, P_i, \gamma, \rho_0 \rangle$,
where all tasks share the same state space $\mathcal{S}$, action space $\mathcal{A}$, discount factor $\gamma \in [0,1]$, and initial state distribution $\rho_0$, but differ in their reward functions $R_i : \mathcal{S} \times \mathcal{A} \to \mathbb{R}$ and transition dynamics $P_i : \mathcal{S} \times \mathcal{A} \to \mathcal{S}$.
Each task $\mathcal{M}_i$ is associated with an offline dataset $\mathcal{D}_i$.

The goal of offline meta-RL is to learn a meta-policy $\pi$ that generalizes to unseen tasks, \ie, maximizes the expected return over a distribution of test tasks:
\begin{align}
\label{eq:objective}
J(\pi) = \mathbb{E}_{\mathcal{M}_i \sim p_{\text{test}}(\mathcal{M})}
\left[\mathbb{E}_{\substack{\mathbf{s}_0 \sim \rho_0 \\ \mathbf{a}_t\sim \pi \\ \mathbf{s}_{t+1} \sim P_i(\mathbf{s}_t, \mathbf{a}_t)}}\left[\sum_{t=0}^T \gamma^t R_i(\mathbf{s}_t, \mathbf{a}_t)\right]\right].
 \end{align}
Context-based methods address task generalization by learning a \emph{context encoder} that infers the underlying task from a small set of transitions.
Specifically, a context encoder $E_{\theta}: \mathcal{S} \times \mathcal{A} \times \mathbb{R} \times \mathcal{S} \rightarrow \mathcal{Z}$ maps transitions $(\mathbf{s}_j, \mathbf{a}_j, r_j, \mathbf{s}_j')$ to a task representation $\mathbf{z}$.
The inferred task representation is then used to condition the policy $\pi(\mathbf{a}_t \mid \mathbf{s}_t, \mathbf{z})$, the value function $Q(\mathbf{s}_t, \mathbf{a}_t, \mathbf{z})$, or a learned dynamics model, enabling adaptation to different tasks without explicit task labels.

\paragraph{Self-predictive RL}
A state encoder $F_\phi : \mathcal{S} \rightarrow \mathcal{X}$ maps states $\mathbf{s} \in \mathcal{S}$ to latent states $\mathbf{x} \in \mathcal{X}$, also referred to as abstractions. These latent states are subsequently used by downstream components such as the policy, value function, or world model. Different methods train the encoder under different objectives and architectural couplings \citep{schwarzer2021data,bridging,iqrl}.  

Self-supervised RL methods exploit (auxiliary) learning signals such as latent temporal consistency and latent reward prediction to train the encoder without relying on state reconstruction. Such an encoder should fulfill the following condition for learning optimal behavior \citep{bridging}:

\begin{align*}
\text{RP:} \hspace{0.4em} 
\exists\, \hat{R}: \mathcal{X} \times \mathcal{A} \rightarrow \mathbb{R}
\hspace{0.4em} \text{s.t.} \hspace{0.4em} 
\mathbb{E}[r_t \mid \mathbf{s}_t, \mathbf{a}_t] = \hat{R}(\mathbf{x}_t, \mathbf{a}_t), \\
\text{ZP:} \hspace{0.4em} 
\exists\, \hat{P}: \mathcal{X} \times \mathcal{A} \rightarrow \Delta \mathcal{X}
\hspace{0.4em} \text{s.t.} \hspace{0.4em} 
\mathbb{E}[\mathbf{x}_{t+1} \mid \mathbf{s}_t, \mathbf{a}_t] = \hat{P}(\mathbf{x}_t, \mathbf{a}_t),
\end{align*}
where $\mathbf{x}_t = F_\phi(\mathbf{s}_t)$.

The \textbf{RP} condition ensures that the latent state $\mathbf{x}$ contains sufficient information to predict the reward function, while the \textbf{ZP} condition enforces temporal consistency in the latent space $\mathcal{X}$. However, satisfying the \textbf{ZP} condition alone admits a trivial solution in which the encoder maps all states to a constant vector, resulting in complete representation collapse \citep{iqrl}. When both \textbf{RP} and \textbf{ZP} conditions are satisfied, the encoder retains the necessary information for predicting returns.

\begin{figure}[t]
\small
\centering

\begin{minipage}[t]{0.15\textwidth}
  \centering\vspace{0pt}
    \begin{tikzpicture}[trim left=0pt, trim right=0pt,
        node distance=1.1\nodedistv,
        line width=1pt,
                mynode/.style={fill=black!5,draw=black!80,rounded corners=1pt,
                       font=\scriptsize,inner sep=3pt,align=center,line width=.6pt},
        arr/.style={->,line width=1pt},
        cyl/.style={mynode, shape=cylinder, aspect=0.25,
                    minimum height=2.0em, minimum width=4.8em},
        ctxenc/.style={mynode, fill=color8, trapezium,
                       trapezium left angle=-70, trapezium right angle=-70,
                       minimum width=5.2em, minimum height=1.6em},
        tuple/.style={mynode, fill=black!10, rounded corners=2pt},
        znode/.style={mynode, fill=color7, minimum width=3.2em, minimum height=1.6em}
    ]

    \node[cyl] (D1) at (-1.0, 0) {$\mathcal{D}^{(1)}$};
    \node[cyl] (D2) at (0.0,0) {$\mathcal{D}^{(2)}$};
    \node[cyl] (D3) at (1.0,0) {$\mathcal{D}^{(N)}$};

    \node[font=\scriptsize, anchor=south, yshift=0\nodedist, xshift=-0.33\nodedist] at (D3.north) {Task datasets};
    
                    \node[below=0.5cm of D2] (tau)
    {$\{\mathbf{s}_j,\mathbf{a}_j,r_j,\mathbf{s}_{j+1}\}$};

    \draw[arr] (D2) -- (tau);
        
                \node[ctxenc, below=0.5cm of tau] (g)
    {\textcolor{white}{Context encoder $E_{\theta}$}};
    
    \draw[arr] (tau.south) -- (g.north);
    
                \node[znode, below=0.5cm of g] (z)
    {\textcolor{white}{$\mathbf{z}$}};
    \draw[arr] (g.south) -- (z.north);
    \end{tikzpicture}\end{minipage}\hspace{0.04\textwidth}\begin{minipage}[t]{0.55\textwidth}
  \centering\vspace{0pt}
  \resizebox{\linewidth}{!}{    \begin{tikzpicture}[line width=1pt]
           \pgfdeclarelayer{background}
  \pgfsetlayers{background,main}

    \newcommand{\sub}[0]{\scalebox{.8}{\ensuremath t}}
  \newcommand{\subs}[1]{\scalebox{.8}{\ensuremath t{+}#1}}

    \tikzstyle{mynode}=[fill=black!5,draw=black!80,rounded corners=1pt,font=\scriptsize,inner sep=0,align=center,line width=.6pt]
  \tikzstyle{trap}=[mynode,fill=color5,trapezium,text width=4em, minimum height=1.4em,
                    trapezium left angle=-70, trapezium right angle=-70,inner sep=2pt]
  \tikzstyle{blob}=[mynode,circle,minimum width=2.2em,minimum height=2.2em,align=center]
  \tikzstyle{polval}=[mynode,minimum width=6em,minimum height=2em,align=center,text=white,fill=black!60,inner sep=2pt]
  \tikzstyle{arr}=[line width=1pt,black,->]
  \tikzstyle{darr}=[line width=1pt,black,<->,densely dotted]
  \tikzstyle{dlc}=[align=center,font=\scriptsize,text width=5em]
  \tikzstyle{main}=[fill=none,draw=none,rectangle,anchor=south,minimum height=5em,minimum width=4em,fill=color0]
  \tikzstyle{dyn}=[fill=none,draw=none,rectangle,anchor=south,minimum height=2em,minimum width=4em,fill=color6]
    \draw[fill=color0,draw=color0!60!black,rounded corners=4pt] (-.5\nodedist,-1\nodedistv) rectangle (3.7\nodedist,1.5\nodedistv);

    \node[blob,main] (z0) at (0\nodedist,-1.5) {Latent state\\$\mathbf{x}_{0}\vphantom{\hat{\mathbf{x}}_{{1}}}$};
  \node[blob,dyn] (z1) at (1\nodedist,-1.) {$D_{\phi}({\mathbf{x}}_{0},\mathbf{a}_{0}, \mathbf{z})$};
  \node[blob,main] (z2) at (2\nodedist,-1.5) {Latent state\\$\hat{\mathbf{x}}_{{1}}$};
  \node[blob,dyn] (z3) at (3\nodedist,-1.0) {$D_{\phi}(\hat{\mathbf{x}}_{{1}},\mathbf{a}_{{1}}, \mathbf{z})$};
  
  \node[trap] (e0) at (0\nodedist,\nodedistv) {\textcolor{white}{Encoder $F_\phi$}};
  \node[trap] (e1) at (1\nodedist,\nodedistv) {\textcolor{white}{Encoder $F_\phi$}};
  \node[trap] (e2) at (3\nodedist,\nodedistv) {\textcolor{white}{Encoder $F_\phi$}};

  \node[blob,fill=none,draw=none,text width=4em,rectangle] (zh1) at ($(z1)!0.5!(e1)$) {Target latent state $\mathbf{x_1}$};
  \node[blob,fill=none,draw=none,text width=4em,rectangle] (zh2) at ($(z3)!0.5!(e2)$) {Target latent state $\mathbf{x_2}$};

    \coordinate (m0) at ($(z0)!.5!(z1)$);
  \node[blob,fill=color1] (a0) at ($(m0)!(zh1)!(m0)$) {\textcolor{white}{$\mathbf{a}_{0}$}};
  \coordinate (m1) at ($(z2)!.5!(z3)$);
  \node[blob,fill=color1] (a1) at ($(m1)!(zh2)!(m1)$) {\textcolor{white}{$\mathbf{a}_{{1}}$}};
    
    \node[blob,fill=color1] (r0) at ($(a0) + (0,-3.5)$) {\textcolor{white}{$r_{0}$}};
  \node[blob,fill=color1] (r1) at ($(a1) + (0,-3.5)$) {\textcolor{white}{$r_{{1}}$}};

  \node[mynode,minimum width=5em,minimum height=4em,outer sep=0,node distance=6em,path picture={\node at (path picture bounding box.center){\includegraphics[height=5.5em]{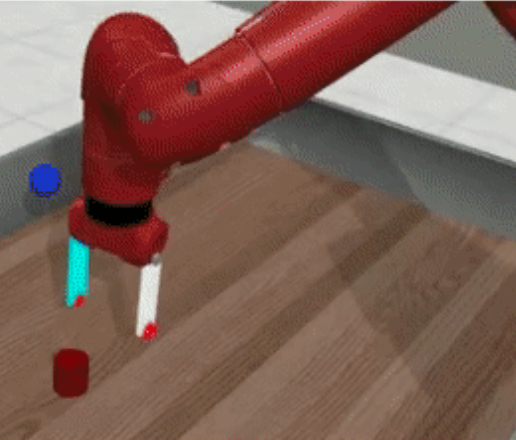}};}] (i0) at (0\nodedist,1.4\nodedistv) {};

  \node[mynode,minimum width=5em,minimum height=4em,outer sep=0,node distance=6em,path picture={\node at (path picture bounding box.center){\includegraphics[height=5.5em]{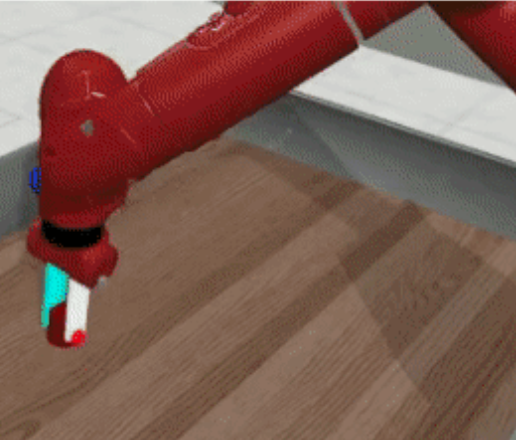}};}] (i1) at (1\nodedist,1.4\nodedistv) {};

  \node[mynode,minimum width=5em,minimum height=4em,outer sep=0,node distance=6em,path picture={\node at (path picture bounding box.center){\includegraphics[height=5.5em]{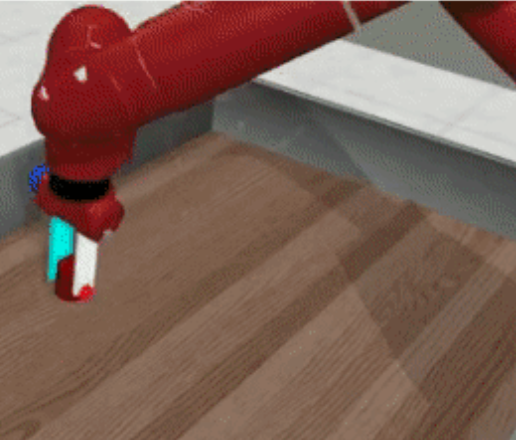}};}] (i2) at (3\nodedist,1.4\nodedistv) {};

  \node[anchor=north,yshift=-1.6em] (o0) at (i0.north) {\color{white}$\mathbf{s}_{0}$};
  \node[anchor=north,yshift=-1.6em] (o1) at (i1.north) {\color{white}$\mathbf{s}_{{1}}$};
  \node[anchor=north,yshift=-1.6em] (o2) at (i2.north) {\color{white}$\mathbf{s}_{{2}}$};

    \draw[arr] (a0) to[bend right=20] (r0);
  \draw[arr] (a1) to[bend right=20] (r1);
  \draw[arr] (z0) to[bend right=40] (r0);
  \draw[arr] (z2) to[bend right=40] (r1);
  \node[fill=color0,minimum width=1cm,minimum height=.6cm,opacity=.9] at ($(z0)!.5!(z1) + (0,-4pt)$) {};
  \node[fill=color0,minimum width=1cm,minimum height=.6cm,opacity=.9] at ($(z2)!.5!(z3) + (0,-4pt)$) {};
  \draw[arr] (i0) -- (e0);
  \draw[arr] (i1) -- (e1);
  \draw[arr] (i2) -- (e2);
  \draw[arr] (z0) -- node[below,font=\scriptsize,xshift=0em] {Dynamics} ++(2,0);
  \draw[arr] (z1) -- node[below,text width=5em,font=\scriptsize,xshift=0em,align=center] {} ++(2,0);
  \draw[arr] (z2) -- node[below,font=\scriptsize,xshift=0em] {Dynamics} ++(2,0);
    \draw[arr] (e0) -- node[right,font=\scriptsize,rotate=90,yshift=6pt] {} (z0);
  \draw[arr] (e1) -- (zh1);
  \draw[arr] (e2) -- (zh2);
  \draw[arr] (a0) -- (z1);
  \draw[arr] (a1) -- (z3);
  \draw[darr] (zh1) -- node[right,font=\scriptsize,align=left]{Loss $\| 
  \|^2_2$} (z1);
  \draw[darr] (zh2) -- node[right,font=\scriptsize,align=left]{Loss $\| 
  \|^2_2$} (z3);
  
    \node[circle,fill=color0,inner sep=0, scale=.7] at ($(e0)!.5!(z0)$) {};
  \node[circle,fill=color0,inner sep=0, scale=.7] (c1) at ($(e1)!.5!(zh1)$) {};
  \node[circle,fill=color0,inner sep=0, scale=.7] (c2) at ($(e2)!.5!(zh2)$) {};
  \node at ($(e1)!.5!(c1)$) {\stopgrad};
  \node at ($(e2)!.5!(c2)$) {\stopgrad};

    \end{tikzpicture}  }
\end{minipage}\hspace{0.02\textwidth}\begin{minipage}[t]{0.15\textwidth}
  \centering\vspace{0pt}
  \resizebox{1.2\linewidth}{!}{    \begin{tikzpicture}[line width=1pt]
                \tikzstyle{mynode}=[fill=black!5,draw=black!80,rounded corners=1pt,font=\scriptsize,inner sep=2pt,align=center,line width=.6pt]
          \tikzstyle{trap}=[mynode,fill=color5,trapezium,text width=4.6em, minimum height=1.4em,
                            trapezium left angle=-70, trapezium right angle=-70,inner sep=2pt]
          \tikzstyle{arr}=[line width=1pt,black,->]
          \tikzstyle{polval}=[mynode,minimum width=6.2em,minimum height=2.2em,align=center,text=white,fill=black!60,inner sep=2pt]
          \tikzstyle{dlc}=[align=center,font=\scriptsize,text width=5.5em]
        
                    \node[mynode,minimum width=5em,minimum height=4em,outer sep=0,
                path picture={\node at (path picture bounding box.center){\includegraphics[height=5.5em]{figs/mw-1}};}] (s) at (0,1.4\nodedistv) {};
          \node[anchor=north,yshift=-1.6em] at (s.north) {\color{white}$\mathbf{s}_{t}$};
        
          \node[trap] (e) at (0,\nodedistv) {\textcolor{white}{Encoder $F_\phi$}};
        
                    \node[mynode,fill=color0,minimum width=7.2em,minimum height=2.6em] (c) at (0,0) { Latent state\\ $\mathbf{x}_{t}$};
        
          \draw[arr] (s) -- (e);
          \draw[arr] (e) -- node[right,font=\scriptsize,rotate=90,yshift=6pt] {} (c);
        
                    \node[circle,fill=color0,inner sep=0, scale=.7] at ($(e)!.5!(c)$) {};
        
                    \node[polval] (pi) at (-0.4\nodedist,-0.85\nodedistv) {$\pi_{\eta}(\mathbf{x}_{t},\mathbf{z})$};
          \node[polval] (Q)  at ( 0.4\nodedist,-0.85\nodedistv) {$Q_{\psi}(\mathbf{x}_{t},\mathbf{a}_{t},\mathbf{z})$};
        
                    \coordinate (split) at ($(c.south)+(0,-0.25\nodedistv)$);
          \draw[arr] (c.south) -- (split);
          \draw[arr] (split) -| (pi.north);
          \draw[arr] (split) -| (Q.north);

                                  \end{tikzpicture}  }\end{minipage}
\caption{
    \textbf{Contextual World Model.}
    \textbf{Left:} A context encoder $E_\theta$ maps transitions from each task to a task representation $\mathbf{z}$, which serves as an implicit task identifier.
    \textbf{Middle:} An observation encoder maps observations $\mathbf{s}_t$ to latent vectors $\mathbf{x}_t$.
    A task-conditioned latent dynamics model {$D_\phi$} predicts future latent states given the current latent state, action, and task representation.
    The world model is trained using multi-step temporal consistency.
    \textbf{Right:} An offline policy is trained using the latent states $\mathbf{x}_t$ and task representation $\mathbf{z}$.}
\label{fig:overview}
\end{figure}

\section{Joint Training of Context Encoder and World Model}
To investigate temporal consistency for task representation learning in the OMRL setting, we condition a world model on the task representation (computed with the context encoder) and jointly train the world model and context encoder with self-supervision. 
The policy, conditioned on latent state and task representation, is optimized with offline RL. 
\cref{fig:overview} provides a high-level description of the joint training. 
\paragraph{Components} The contextual world model and offline RL consist of the following components:
\begin{subequations}
\begin{align}
&\text{Context encoder: } 
& \mathbf{z} &= \mathbb{E}[E_\theta(\mathbf{s}_t, \mathbf{a}_t, r_t, \mathbf{s}_{t+1})] \label{eq:context_encoder} \\
&\text{Obs. encoder: } 
& \mathbf{x}_t &= F_\phi(\mathbf{s}_t) \label{eq:obs_encoder} \\
&\text{Latent dynamics: } 
& \hat{\mathbf{x}}_{t+1} &= D_\phi(\mathbf{x}_t, \mathbf{a}_t, \mathbf{z}) \label{eq:dynamics} \\
&\text{Reward model: } 
& \hat{r}_t &= R_\phi(\mathbf{x}_t, \mathbf{a}_t, \mathbf{z}) \label{eq:reward} \\
&\text{Q-function: } 
& q &= Q_\psi(\mathbf{x}_t, \mathbf{a}_t, \mathbf{z}) \label{eq:critic} \\
&\text{Value function: } 
& v &= V_\omega(\mathbf{x}_t, \mathbf{z}) \label{eq:value} \\
&\text{Policy: } 
& \mathbf{a}_t &\sim \pi_\eta(\mathbf{a}_t \mid \mathbf{x}_t, \mathbf{z}) \label{eq:actor}
\end{align}
\end{subequations}
The world model consists of the observation encoder $F_\phi$, latent dynamics $D_\phi$, and reward model $R_\phi$.
We denote the encoding as $\Phi(\mathbf{s}) \coloneqq F_{\phi}(\mathbf{s})$.
The policy $\pi_\eta$, value function $V_\omega$, and Q-function $Q_\psi$ are conditioned on both the latent state $\mathbf{x}_t$ and the task representation $\mathbf{z}$.
Unlike prior work, we jointly train the world model and the context encoder, allowing task inference to benefit directly from temporal consistency.
Importantly, we avoid observation reconstruction and benefit from the powerful representation learning provided by a self-predictive loss.

\subsection{Contextual World Model}
In context-based  OMRL, a context encoder $E_{\theta}$ infers the task from a small set of transitions.
The observation encoder $F_{\phi}$ is then shared across all tasks, while the latent dynamics $D_\phi$ and reward model $R_\phi$ are conditioned on the inferred task representation $\mathbf{z}$.
This reflects the fact that tasks differ in dynamics and reward functions.

\paragraph{Task representation.}
During training, we sample a meta-batch of tasks.
For each task $i$ with dataset $\mathcal{D}_i$, the task representation is given by
\begin{equation}
\mathbf{z}^i = \mathbb{E}_{(\mathbf{s}_t, \mathbf{a}_t, r_t, \mathbf{s}_{t+1}) \sim \mathcal{D}_i}
\left[E_\theta(\mathbf{s}_t, \mathbf{a}_t, r_t, \mathbf{s}_{t+1})\right].
\label{eq:expectation_z}
\end{equation}

\paragraph{World model.}
Observations from all tasks are mapped to continuous latent vectors using \cref{eq:obs_encoder}. 
The latent dynamics model predicts next latent state $\hat{\mathbf{x}}_{t+1}$ conditioned on the current latent state $\mathbf{x}_t$, action $\mathbf{a}_t$, and task representation $\mathbf{z}$. 
Rewards are also predicted by a task-conditioned reward model.

We train the context encoder jointly with the world model -- observation encoder, latent dynamics, and reward function -- using a multi-step temporal consistency (TC) objective, given by
\begin{equation}
\mathcal{L}_{\text{TC}}(\theta, \phi) =
\sum_{h=0}^{H-1} \gamma^h \Big(
\|(D_\phi(\hat{\mathbf{x}}_{t+h}, \mathbf{a}_{t+h}, \mathbf{z})- \mathbf{x}_{t+h+1}) \|_2^2 
+ \| R_\phi(\hat{\mathbf{x}}_{t+h}, \mathbf{a}_{t+h}, \mathbf{z}) - r_{t+h} \|_2^2
\Big),
\label{eq:worldmodel_objective}
\end{equation}
where $\hat{\mathbf{x}}_t = F_\phi(\mathbf{s}_t)$ is the initial latent state,
$\hat{\mathbf{x}}_{t+h+1}$ is computed from the latent dynamics model,
and $\mathbf{x}_{t+h+1} = \mathrm{sg}(F_{\bar{\phi}}(\mathbf{s}_{t+h+1}))$ is the target latent state computed using an exponential moving average (EMA) encoder.
Here, $H$ is the prediction horizon, $\gamma$ is the discount factor, and $\bar{\phi}$ denotes EMA parameters. Additionally, contrastive learning, e.g., the InfoNCE loss (\cref{eq:infonce} in  \cref{sec:contrastive}), enhances the discriminability of task representations. We investigate the effects of training the context encoder using temporal consistency, contrastive learning, and their combination on task representation learning and downstream generalization.

\subsection{Meta Policy Optimization} \label{sec:offline_rl_IQL}
After learning the world model, we train a policy using offline reinforcement learning.
We condition the policy and value functions on the latent state and task representation.
To address distribution shift in offline RL, we adopt Implicit Q-Learning (IQL \citep{iql}).
In \cref{sec:offline_rl}, we show evidence that this investigation is compatible with other offline RL methods. 

IQL avoids out-of-distribution actions by learning a value function using expectile regression
\begin{equation*}
\mathcal{L}_V(\omega) =
\mathbb{E}\big[\mathcal{L}_2^\tau(Q_{\bar{\psi}}(\mathbf{x}_t, \mathbf{a}_t, \mathbf{z}) - V_\omega(\mathbf{x}_t, \mathbf{z}))\big],
\end{equation*}
where $\mathcal{L}_2^\tau(x) = (\tau - \mathbb{1}[x < 0])x^2$ and $\bar{\psi}$ denotes EMA parameters.
The Q-function is trained using standard temporal-difference learning.

The policy is optimized using advantage-weighted regression \citep{awr}, given by
\begin{equation*}
\mathcal{L}_\pi(\eta) =
\mathbb{E} \left[ \exp(A(\mathbf{x}_t, \mathbf{a}_t, \mathbf{z}) / \mathcal{B})
\log \pi_\eta(\mathbf{a}_t \mid \mathbf{x}_t, \mathbf{z})
\right],
\end{equation*}
where
$A(\mathbf{x}_t, \mathbf{a}_t, \mathbf{z}) = Q_\psi(\mathbf{x}_t, \mathbf{a}_t, \mathbf{z}) - V_\omega(\mathbf{x}_t, \mathbf{z})$ 
is the advantage and $\mathcal{B}$ is a temperature parameter.

It is worth noting that the world model is used solely for representation learning and not for planning. 
Specifically, observations are encoded into a latent space, and both the policy and value functions are learned directly over these latent states. 
Consequently, learning proceeds within the latent MDP induced by the encoder.

\subsection{Theoretical Analysis}
\label{sec:theory_main}

Our meta policy optimization operates in the learned latent space $(\mathbf{x}_t,\mathbf{z}_t)$ without reconstructing observations.
Here we build on the well-known simulation lemma~\citep{kearnsNearOptimalReinforcementLearning2002}
and bound the value error incurred by {\em (i)} latent abstraction via $\Phi(\mathbf{s})$, {\em (ii)} learned world-model error, and {\em (iii)} task inference via the context encoder.
All proofs are deferred to Appendix~\ref{app:theory}.

\paragraph{Setup.}
For task $i$ with MDP $\mathcal{M}_i=\langle \mathcal{S},\mathcal{A},R_i,P_i,\gamma,\rho_0\rangle$ 
and bounded rewards
$|R_i(\mathbf{s}, \mathbf{a})| \le R_{\max}$.
Let the learned latent state be $\mathbf{x} = \Phi(\mathbf{s})\coloneqq F_\phi(\mathbf{s}) \in \mathcal{X}$.
Given $(\mathbf{x}, \mathbf{a}, \mathbf{z})$, the learned dynamics model induces a kernel
$\widehat{P}_\phi(\cdot \mid \mathbf{x}, \mathbf{a}, \mathbf{z})$ over $\mathcal{X}$ and the learned reward model outputs $\widehat{R}_\phi(\mathbf{x}, \mathbf{a}, \mathbf{z})$.
We use the infinite-horizon discounted value in MDP $\mathcal{M}_i$ as
\begin{align*}
V_{\mathcal{M}_i}^\pi(\mathbf{s}) =
\mathbb{E} \bigg[\sum_{t\ge 0}\gamma^t R_i(\mathbf{s}_t,\mathbf{a}_t)
\mid \mathbf{s}_0=\mathbf{s},\ \mathbf{a}_t\sim \pi(\cdot\mid \mathbf{s}_t)\bigg]. 
\end{align*}
\paragraph{Sources of error.}
We quantify three sources of error:

\emph{(1) Latent abstraction error.}
Let $P_i^\Phi(\mathbf{x}'\mid \mathbf{s},\mathbf{a})$ be the pushforward transition induced by $\Phi$ (Appendix~\ref{app:theory}).
We assume there exists a Markov kernel $\bar{P}_i^\Phi(\cdot\mid \mathbf{x},\mathbf{a})$ and latent reward function $\bar{R}_i^\Phi(\mathbf{x},\mathbf{a})$ such that
\begin{subequations}
\begin{align}
\epsilon_{\mathrm{abs}}^{P} &\coloneq
\sup_{\mathbf{s},\mathbf{a}}
\big\|P_i^\Phi(\cdot \mid \mathbf{s}, \mathbf{a}) - \bar{P}_i^\Phi(\cdot \mid \Phi(\mathbf{s}), \mathbf{a})\big\|_1, \label{eq:abs_error} \\
\epsilon_{\mathrm{abs}}^{R} 
&\coloneq \sup_{\mathbf{s},\mathbf{a}}
\big|R_i(\mathbf{s},\mathbf{a})-\bar{R}_i^\Phi(\Phi(\mathbf{s}),\mathbf{a})\big|.
\end{align}
\end{subequations}
\emph{(2) World model approximation error (true task).}
When conditioned on the true task $i$, define
\begin{subequations}
\begin{align}
\epsilon_{\mathrm{model}}^{P}
&\coloneq \sup_{\mathbf{x}, \mathbf{a}}
\big\|\widehat{P}_\phi(\cdot\mid \mathbf{x}, \mathbf{a},i)-\bar{P}_i^\Phi(\cdot\mid \mathbf{x}, \mathbf{a})\big\|_1, \label{eq:model_error} \\
\epsilon_{\mathrm{model}}^{R}
&\coloneq \sup_{\mathbf{x}, \mathbf{a}}
\big|\widehat{R}_\phi(\mathbf{x}, \mathbf{a},i)-\bar{R}_i^\Phi(\mathbf{x}, \mathbf{a})\big|.
\end{align}
\end{subequations}
\emph{(3) Task inference error (conditioning on $\mathbf{z}$).}
Let $\mathbf{z}=g_\theta(\mathcal{H})$ be the task representation inferred from history $\mathcal{H}$ collected in task $i$.
Define
\begin{subequations}
\begin{align}
\epsilon_{\mathrm{task}}^{P}
&\coloneq \sup_{\mathbf{x},\mathbf{a}}
\big\|\widehat{P}_\phi(\cdot\mid \mathbf{x}, \mathbf{a}, \mathbf{z})-\widehat{P}_\phi(\cdot\mid \mathbf{x}, \mathbf{a}, i)\big\|_1, \label{eq:task_error} \\
\epsilon_{\mathrm{task}}^{R}
&\coloneq \sup_{\mathbf{x},\mathbf{a}}
\big|\widehat{R}_\phi(\mathbf{x}, \mathbf{a}, \mathbf{z})-\widehat{R}_\phi(\mathbf{x}, \mathbf{a}, i)\big|.
\end{align}
\end{subequations}
\paragraph{Value error bound.}
Let $\pi(\mathbf{a}\mid \mathbf{x}, \mathbf{z})$ be any latent policy and define the lifted policy
$\pi^\Phi(\mathbf{a}\mid \mathbf{s}, \mathbf{z})\coloneqq \pi(\mathbf{a}\mid \Phi(\mathbf{s}), \mathbf{z})$.
Let $V_{\mathcal{M}_i}^{\pi^\Phi}$ be its value in the original MDP and let $V_{\widehat{\mathcal{M}}(\mathbf{z})}^{\pi}$ be its value in the learned latent MDP induced by $(\widehat{P}_\phi(\cdot\mid \mathbf{x}, \mathbf{a}, \mathbf{z}), \widehat{R}_\phi(\mathbf{x}, \mathbf{a}, \mathbf{z}))$.

\begin{theorem}[Value error from abstraction, model learning, and task inference]
\label{thm:main_value_bound}
For any task $i$ and any latent policy $\pi$,
\begin{subequations}
\begin{align}
\sup_{\mathbf{s}\in\mathcal{S}} \Big|
V_{\mathcal{M}_i}^{\pi^\Phi}(\mathbf{s})
&- V_{\widehat{\mathcal{M}}(\mathbf{z})}^{\pi}(\Phi(\mathbf{s})) \Big|
\le \frac{\epsilon^{R}}{1-\gamma}
+ \frac{\gamma R_{\max}}{(1-\gamma)^2} \epsilon^{P},
\label{eq:main_value_bound}
\nonumber \\
\text{with } \qquad
\epsilon^{R} &= \epsilon_{\mathrm{abs}}^{R}
+ \epsilon_{\mathrm{model}}^{R}
+ \epsilon_{\mathrm{task}}^{R}, \\
\epsilon^{P} &= \epsilon_{\mathrm{abs}}^{P}
+ \epsilon_{\mathrm{model}}^{P}
+ \epsilon_{\mathrm{task}}^{P}.
\end{align}
\end{subequations}
\end{theorem}
\paragraph{Interpretation.}
The bound decomposes value error into: 
{\em (i)} latent abstraction error (how non-Markov the latent state $\mathbf{x}$ is),
{\em (ii)} world model error (how accurately $D_\phi$ and $R_\phi$ match the induced latent MDP), and
{\em (iii)} task inference error (how well $\mathbf{z}$ substitutes for the true task ID for predicting latent dynamics and rewards).

\paragraph{Implications for representation learning.}
Crucially, the bound in \cref{thm:main_value_bound} does not require reconstructing observations from the latent state.
Unlike reconstruction-based world models, which aim to preserve all information in $\mathbf{s}_t$, our analysis only requires that the learned representations $(\mathbf{x}, \mathbf{z})$ preserve information for predictive control.

Concretely, the abstraction error $(\epsilon_{\mathrm{abs}}^{P},\epsilon_{\mathrm{abs}}^{R})$ is governed by how well the observation encoder $\Phi$ groups states with similar future latent dynamics and rewards.
Our temporal consistency loss directly targets this by enforcing that $\mathbf{x}_{t+1}$ is predictable from $(\mathbf{x}_t, \mathbf{a}_t, \mathbf{z})$, rather than from raw observations.
The world model error $(\epsilon_{\mathrm{model}}^{P},\epsilon_{\mathrm{model}}^{R})$ is reduced by training the latent dynamics and reward models to accurately predict future latent states and rewards under the inferred task.
Finally, the task inference error $(\epsilon_{\mathrm{task}}^{P},\epsilon_{\mathrm{task}}^{R})$ is controlled by jointly training the context encoder with the world model, encouraging $\mathbf{z}$ to capture precisely the task-dependent factors needed for prediction.

Together, these objectives yield a latent representation that is sufficient for planning and control, without requiring reconstruction of $\mathbf{s}_t$.
We provide a more detailed theoretical analysis in Appendix~\ref{app:theory}. 

\section{Experiments}
We evaluate temporal consistency for task representation learning on multi-task benchmarks from MuJoCo~\citep{mujoco}, Contextual DeepMind Control (DMC)~\citep{deepmindcontrol, hyperzero}, and Meta-World~\citep{metaworld}.
Our experiments address the following questions:
\begin{itemize}
    \item[\textbf{RQ1}] How does learning the context encoder jointly with the world model's self-supervised temporal consistency loss affect the task representation?
    \item[\textbf{RQ2}] Does temporal consistency for task representation learning reduce different sources of error described in our theoretical analysis (\cref{sec:theory_main})?
    \item[\textbf{RQ3}] How does incorporating temporal consistency into training the context encoder impact the generalization to unseen tasks compared to prior methods? 
\end{itemize}

\paragraph{Baselines.}
We compare against the following context-based OMRL methods (more details in \cref{sec:baseline_details}):
\begin{itemize}
    \item \textbf{FOCAL}~\citep{focal} trains the context encoder using a distance-based contrastive objective that pulls representations from the same task together and pushes different tasks apart. \looseness-1
    \item \textbf{CSRO}~\citep{csro} extends FOCAL by reducing context distribution shift using CLUB as a mutual information regularizer \citep{club}. 
        \item \textbf{UNICORN}~\citep{unicorn} trains a context encoder using task-conditioned dynamics and reward models with one-step prediction. \textit{UNICORN-SS} further augments this with a contrastive objective, whereas \textit{UNICORN-SUP} relies exclusively reconstruction.  
    \item \textbf{AutoRegression} where we extend \textit{UNICORN-SS} to multi-step prediction by performing backpropagation through time in the state space. 
    Comparing this baseline with the contextual world model highlights the role of multi-step temporal consistency in latent space for task representation learning.
\end{itemize}

\paragraph{Datasets.} We generate offline datasets using Dropout Q-learning (DroQ)~\citep{droq}, training a separate agent for each task.
Each dataset contains 1000 trajectories collected at different stages of training, ranging from random behavior to near-expert performance over 1M environment steps. \looseness-1

\paragraph{Benchmarks.} For MuJoCo and Contextual-DMC, we use 20 training tasks, 10 in-distribution test tasks, and 10 out-of-distribution test tasks.
Out-of-distribution tasks differ in the ranges of underlying variation factors.

In Meta-World, we follow standard Meta-RL settings (ML1, ML10, ML45).
ML1 evaluates generalization across goals within a single environment.
ML10 and ML45 evaluate generalization to unseen environments without task identifiers, using 10 and 45 training environments, respectively, and 5 held-out test environments.

\begin{figure}[t]
    \centering
    \includegraphics[width=\linewidth]{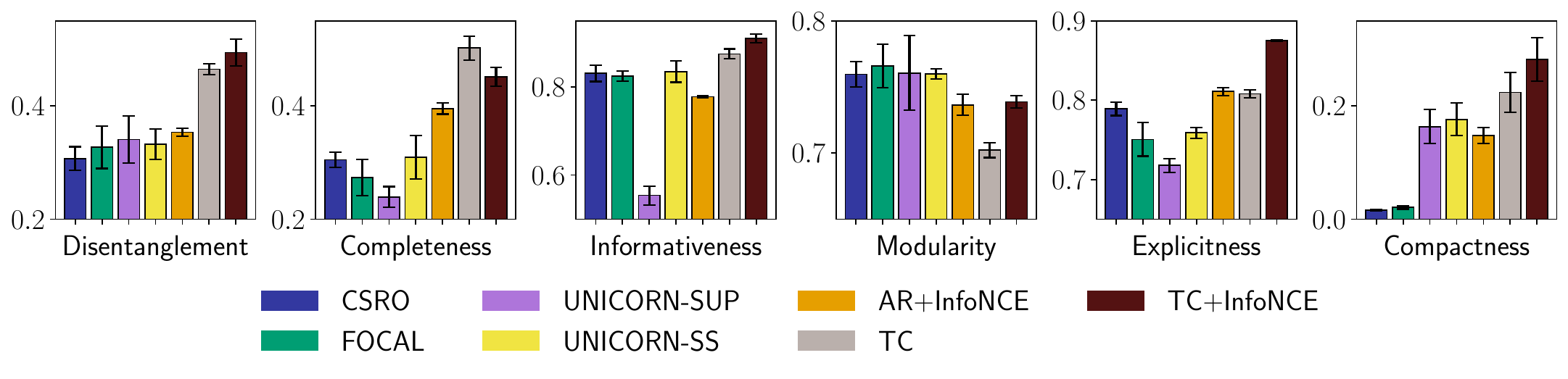}
    \caption{
    Disentanglement metrics (DCI, InfoMEC) for the Cheetah-length-speed (LS) environment. 
    \textbf{Temporal consistency disentangles the variation factors more effectively, while contrastive learning enhances task distinguishability. }
    {TC} denotes training the context encoder solely with the temporal consistency
    objective, and {AR} denotes auto-regression. 
    Averaged over 6 random seeds.
    }
    \label{fig:DCI}
\end{figure}

\begin{figure}[t]
    \centering
    \includegraphics[width=\linewidth]{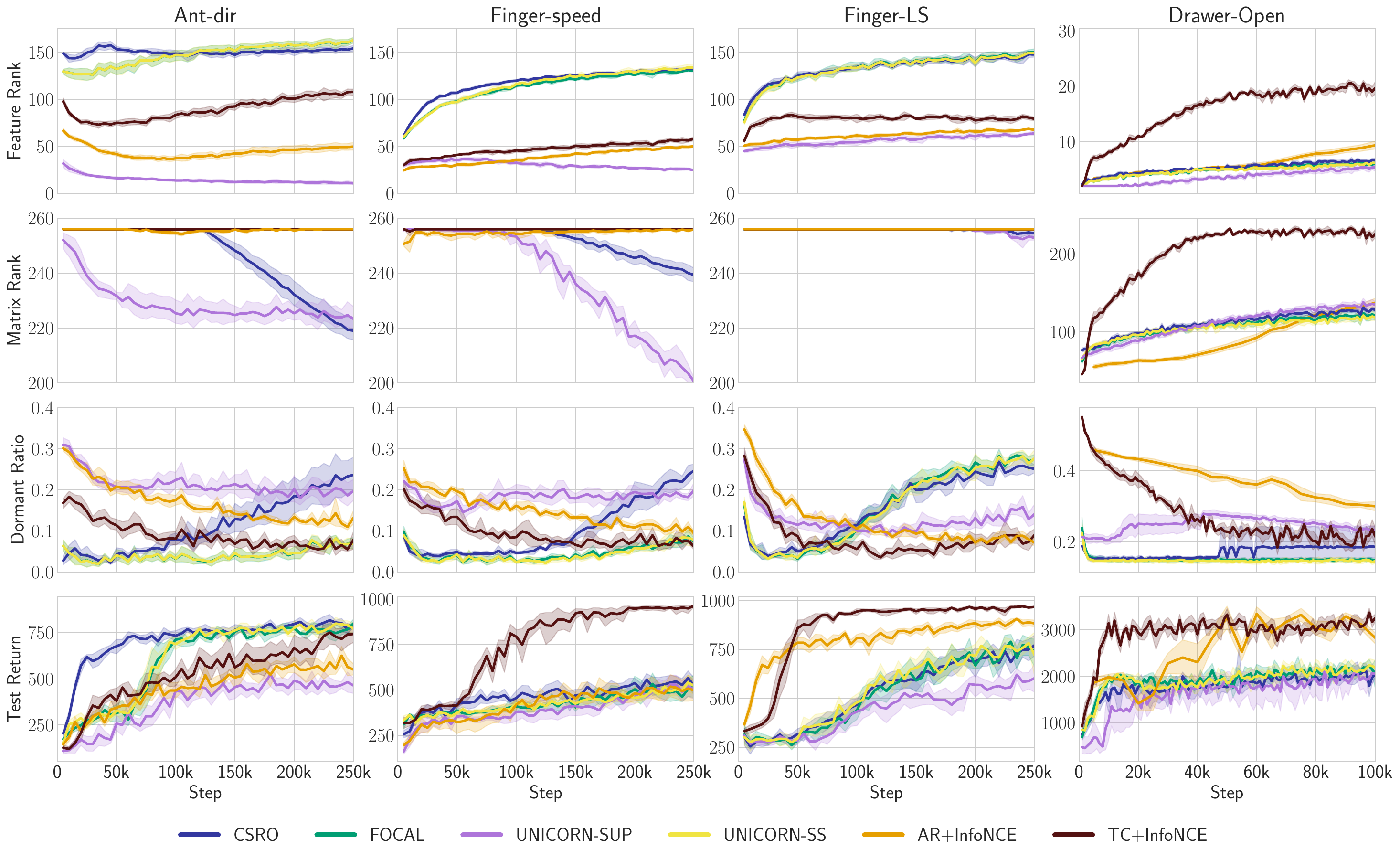}
    \caption{
        Representation metrics (Feature Rank, Rank, and Dormant ratio) of the context encoder. \textbf{Temporal consistency (TC) maintains low dormant neuron ratio and high matrix rank} compared to pure reconstruction (UNICORN-SUP) and auto-regression (AR). 
        The shaded area represents 95\% confidence interval across 6 random seeds.
    }
    \label{fig:ce_metrics} 
\end{figure}

\subsection{Investigating Task Representation Learning}

We first study how different training objectives affect the quality of the learned task representations.
Previous work often relies on t-SNE visualizations to assess task separation, but such visualizations are qualitative and sensitive to hyperparameters.
We instead use disentanglement metrics, which measure how well latent dimensions align with underlying task variation factors.
We report DCI~\citep{DCI} metrics (disentanglement, completeness, informativeness) and InfoMEC~\citep{InfoMEC} metrics (modularity, explicitness, compactness).
Together, these metrics capture whether a task representation is structured, informative, and aligned with true task factors.

\Cref{fig:DCI} reports results for the Cheetah-LS environment, which varies desired speed and torso length.
Training the context encoder using latent temporal consistency (\textit{TC}) leads to better disentanglement than reconstruction-based training (UNICORN-SUP) and autoregression (\textit{AR}).
Adding contrastive objectives improves task distinguishability, reflected in higher informativeness and explicitness.
Combining both objectives yields the strongest overall representations.
\Cref{sec:appendix_disentanglement} presents disentanglement metrics for more environments, demonstrating consistent behavior.

We analyze task representation learning using feature rank, matrix rank, and dormant neuron ratio (\cref{fig:ce_metrics}).
The matrix rank \citep{iqrl} quantifies correlations among learned features, while the dormant neuron ratio \citep{dormant, plasticity_loss} measures the proportion of neurons with near-zero activations, which limits the expressivity of the network.
Given the representation matrix -- defined as the activations of the last hidden layer of the context encoder -- the feature rank \citep{feature_rank_rl, sharpness_aware_low_rank, representation_collapse_ppo, simplicial_embeddings_improve_sample} measures feature diversity by identifying the smallest subspace that preserves $(1-\varepsilon)\%$ of the total variance; throughout our experiments, we set $\varepsilon = 0.01$.
Self-prediction learns less correlated features than reconstruction (UNICORN SUP), maintains higher matrix rank, and significantly reduces dormant neurons, suggesting that latent temporal consistency leads to more expressive and robust task representations \citep{plasticity_loss}.
\cref{fig:extra_rep_metrics} provides representation metrics for more environments.

\subsection{Error Decomposition}
In this section, we conduct experiments to evaluate the theoretical analysis in \cref{sec:theory_main}. 
To quantify each error source, we train separate networks using privileged information, the true task ID $i$, unavailable in the offline meta-RL setting. 
\Cref{sec:error_details} provides more detail. 
\begin{itemize}[]
\item \textbf{Abstraction error} A network $P_{\text{abs}}(\mathbf{x}_{t+1} \mid \mathbf{s}_t, \mathbf{a}_t, i)$ is trained to predict the next latent from the state, action, and task ID. The error is the discrepancy between this prediction and the encoded next state $\Phi(\mathbf{s})$ (\cref{eq:abs_error}). \looseness-1
\item \textbf{World modeling error} A network $P_{\text{model}}(\mathbf{x}_{t+1} \mid \mathbf{x}_t, \mathbf{a}_t, i)$ is trained to predict the next latent from the current latent, action, and task ID. The error is measured against the encoded next state $\Phi(\mathbf{s})$ (\cref{eq:model_error}). \looseness-1
\item \textbf{Task error} We compare predictions of $P_{\text{model}}$ with those of the contextual world model, which conditions on the learned task representation $\mathbf{z}$ instead of task ID $i$ (\cref{eq:task_error}). \looseness-1
\end{itemize}
\cref{fig:errors} illustrates the evolution of these errors during training. 
All networks are trained and evaluated simultaneously on the same transitions. 
Jointly training the context encoder and the world model reduces all three error terms, and their sum upper-bounds the temporal consistency loss. 
Task error is the dominant term in most experiments, indicating that the main source of error is due to estimating the underlying task from previous transitions with the context encoder. 

\begin{figure}[t]
    \centering
    \includegraphics[width=\linewidth]{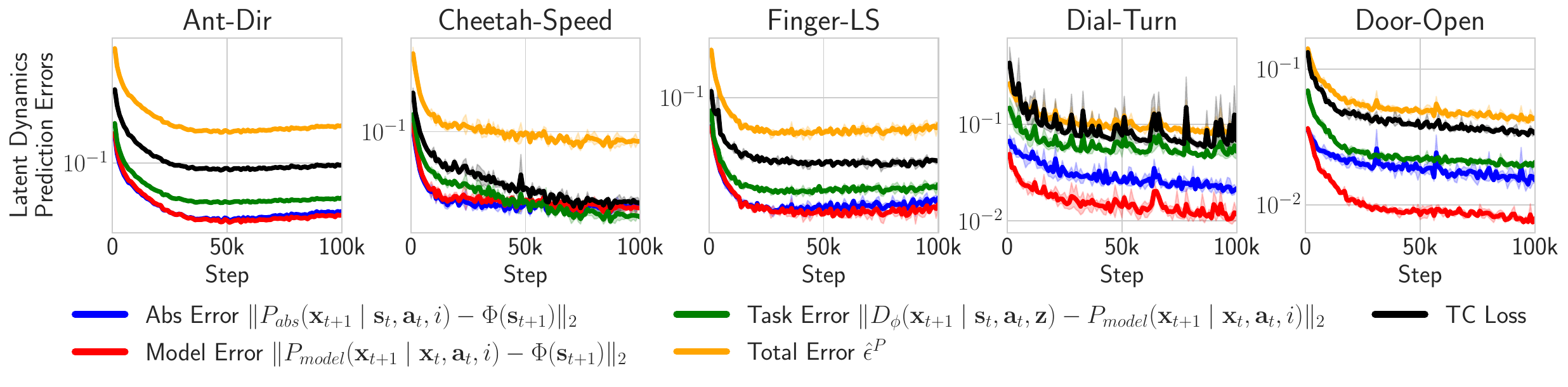}
    \caption{
    Estimating different error sources.
    \textbf{Joint training of the context encoder and world model reduces all error terms, with task error dominating; their sum upper-bounds the temporal consistency loss.}
    Abstraction error arises from mapping states to a latent space, model error reflects world model accuracy, and task error measures how well the context encoder identifies the task. 
        Shaded regions denote 95\% confidence intervals over three random seeds. \looseness-1
    }
    \label{fig:errors}
\end{figure}

\subsection{Generalization to New Tasks}

We evaluate whether improved task representations lead to better generalization.
During meta-testing, the agent is not given task identifiers and must infer the task online from interaction data.
The task representation is initialized to zero and updated at each step using collected transitions.\looseness-2

\Cref{fig:aggregate} reports few-shot performance on in-distribution tasks.
These results suggest that training the context encoder with both the latent temporal consistency and contrastive objectives improves both in-distribution performance and generalization. 
\Cref{sec:appendix_generalization} reports few-shot performance across benchmarks as well as zero-shot results on MuJoCo and Contextual-DMC.
We also evaluate generalization to unseen environments in Meta-World ML10 and ML45 (\cref{tab:ML_new_env} in \cref{sec:appendix_generalization}).
Temporal consistency with contrastive learning achieves higher success rates on both training and test environments, although generalization to entirely new environments remains challenging for all methods.
Increasing the diversity of training environments improves test performance.

\begin{figure}[t]
    \centering
    \subfloat[\centering MuJoCo and Contextual DMC]{{\includegraphics[width=0.45\linewidth]{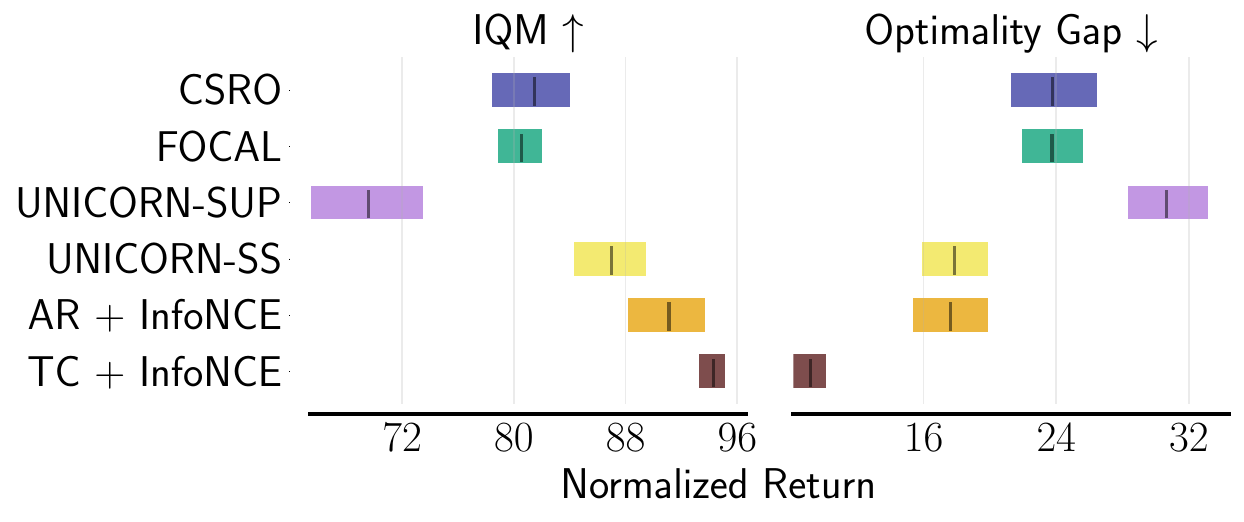} }}    \qquad
    \subfloat[\centering MetaWorld]{{\includegraphics[width=0.45\linewidth]{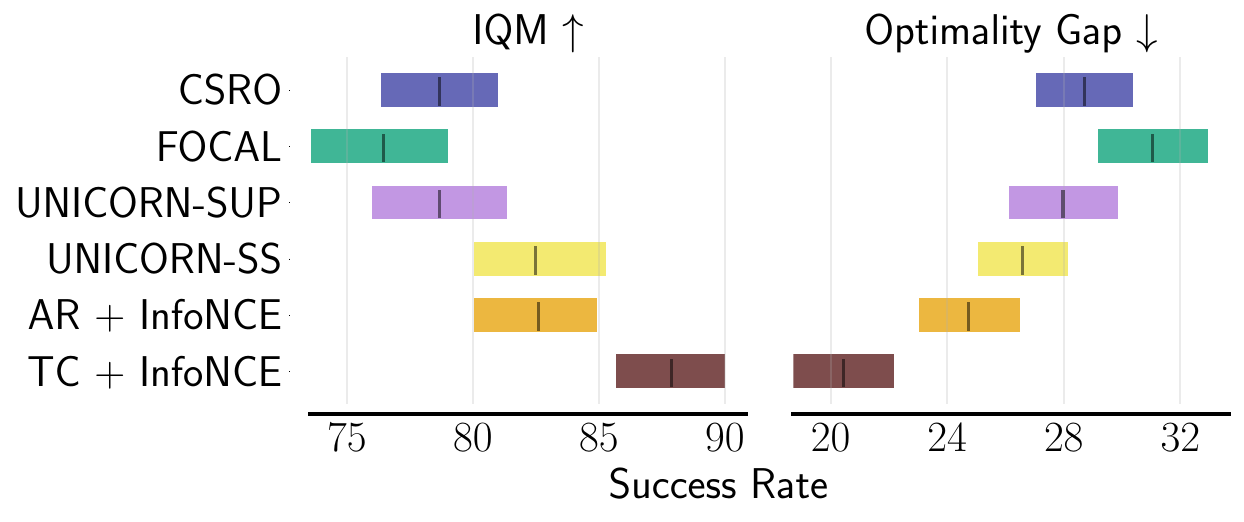} }}    \caption{Few-shot in-distribution performance. 
    \textbf{Temporal consistency in the latent space improves generalization compared to reconstruction and auto-regression}
    (9 MuJoCo and Contextual DMC environments and 30 MetaWorld environments, each experiment with 6 random seeds).
        }
    \label{fig:aggregate}\end{figure}

\subsection{Further Results}
\Cref{sec:furthere_results} in the appendix presents additional empirical results and ablation studies.
\Cref{sec:plan} demonstrates that the learned world model can be used for online planning, where increasing the planning horizon generally improves zero-shot performance at the cost of increased computation.
Additionally, incorporating a conservative value function improves planning performance and reduces the need for long-horizon planning.
\Cref{sec:dreamerv3} compares the proposed framework to DreamerV3 \citep{dreamerv3} and MetaDT \citep{wang2024metadt}, demonstrating better generalization of the proposed method.
\cref{sec:ce_architecture} investigates alternative context-encoder architectures, including RNN- and transformer-based architectures, but does not find a significant benefit on the evaluated benchmarks.

\Cref{sec:latent_formulation} provides an ablation study of different latent-space formulations, indicating that a discrete latent space with temporal consistency defined as next-code classification improves performance, consistent with the findings of \citep{dcmpc}.
\Cref{sec:contrastive} further analyzes the impact of contrastive learning and latent temporal consistency objectives on generalization, demonstrating that incorporating temporal consistency into the context encoder objective improves generalization, particularly on out-of-distribution tasks.
\Cref{sec:multi_step_tc} presents an ablation study on the training horizon in \cref{eq:worldmodel_objective}, showing that the proposed framework is robust to this hyperparameter, while longer horizons can negatively affect training stability, potentially due to backpropagation through time. 
\Cref{sec:scaling} demonstrates the performance with different model sizes, indicating that the proposed method scales effectively. 
\Cref{sec:offline_rl} provides an ablation study of the offline RL backbone used for policy optimization, showing that the proposed framework is relatively agnostic to the choice of offline RL algorithm. 
\cref{sec:bounding_z} shows that bounding the task representation $\mathbf{z}$ improves performance across the evaluated bounding functions.
Finally, \cref{sec:computation_cost} compares the computational cost of the proposed framework with the baselines, showing that although each training iteration is more expensive for the proposed framework, it requires fewer iterations to converge, resulting in faster overall training.

\section{Related Work} 
We review world models, self-predictive RL, and context-based offline meta-RL, highlighting how our method bridges gaps between these directions.

\paragraph{World models and self-predictive RL.} World models learn compact state representations by mapping observations to a latent state space and learning to predict future latent states and rewards.
They have been successfully used in model-based RL \citep{planet, muzero, tdmpc, hansen2024tdmpc2, dcmpc}, and for representation learning in model-free RL \citep{tcrl, td7, iqrl, general_modelfree}.
Across both settings, a unifying principle is \emph{self-predictive RL} \citep{schwarzer2021data,tcrl}, which provides a strong self-supervised signal for learning latent state spaces.
Recent works have also shown that reconstruction is unnecessary and that enforcing multi-step latent consistency leads to improved performance \citep{tdmpc, dcmpc, iqrl}.
However, these methods are typically task-specific or assume access to a known task identifier, and therefore do not address generalization across tasks or environments.

\paragraph{Context-based offline meta-RL.}
Context-based offline meta-RL (OMRL) methods aim to generalize to unseen tasks by learning a task representation from a context of past transitions.
Prior work trains context encoders using contrastive objectives \citep{focal, corro}, but such approaches are sensitive to context distribution shift between training and testing.
CSRO \citep{csro} and ER-TRL \citep{ertrl} mitigate this shift by regularizing the dependence between task representations and behavior policies.
UNICORN \citep{unicorn} addresses task inference via predictive objectives, but operates directly in the observation space and does not learn a latent state space for control.

We unify these two lines of work.
Unlike prior OMRL methods, we learn a latent state space using self-predictive world models without reconstructing observations.
By training the world model and context encoder together, both the latent state $\mathbf{x}$ and task representation $\mathbf{z}$ capture the task-dependent predictive structure required for control.

\section{Conclusion} This paper investigates temporal consistency for contextual OMRL, where we condition world models on learned task representations and jointly train the world model and context encoder without reconstruction.
The resulting self-supervised representations capture task-relevant variation while avoiding representation collapse, leading to improved generalization to unseen tasks.
We further provide value error bounds in the latent MDP induced by the observation encoder, showing that accurate control is possible without observation reconstruction.
Overall, our results demonstrate that predictive latent representations suffice for generalization in offline meta-RL, and that jointly learning the context encoder and world model is both principled and effective. \looseness-1

\paragraph{Limitations} 
Extending our experiments to accommodate different state and action spaces across tasks is a promising direction for future research.
Furthermore, we evaluate temporal consistency for task representation learning in environments with only state information; in principle, the observation encoder architecture could be modified to support visual observations.

\begin{ack}
We acknowledge CSC – IT Center for Science, Finland, for awarding this project access to the LUMI supercomputer, owned by the EuroHPC Joint Undertaking, hosted by CSC (Finland) and the LUMI consortium through CSC. We acknowledge the computational resources provided by the Aalto Science-IT project. M. Nakhaei was supported by Business Finland (BIOND4.0 - Data Driven Control for Bioprocesses and Aurora - Automated and Connected Machines). A. Scannell was supported by the UK Engineering and Physical Sciences Research Council (grant EP/W002876/1). K. S. Luck was supported by the project {ETAPE: Embodiment- and Task-Aware Parameter Embeddings for Robotic Foundation Models}, file number NGF.1609.243.075, under the research programme {National Growth Fund AiNed XS Europe 24-3}, financed by the Dutch Research Council (NWO).
\end{ack}

\phantomsection\addcontentsline{toc}{section}{References}
\begingroup
\small
\bibliographystyle{unsrt}
\bibliography{main}
\endgroup

\clearpage
\appendix
\section*{Appendices}
In this appendix, we provide implementation details in \cref{sec:appendix_details}, theoretical analysis in \cref{app:theory}, and additional results in \cref{sec:furthere_results}. 
\subsection*{Large Language Models}
We use large language models (LLMs) to assist with paper writing, including proofreading for typos and grammar errors. 
We also employ LLMs to generate scripts for visualizing results and creating figures.

\section{Implementation Details} \label{sec:appendix_details}
We used PyTorch \citep{pytorch} for implementation and the AdamW optimizer \citep{adamw} for training the world model and the Adam optimizer \citep{adam} for the other models. All neural networks are implemented as MLPs where each intermediate linear layer is followed by Layer Normalization (\ie pre-norm) \citep{ba2016layernormalization} and a Mish activation function \citep{misra2019mish}. 
Below we summarize the architecture of the world model and policy for the Cheetah-LS environment.
\begin{table}[h]
\caption{Hyperparameters for jointly training context encoder and world model.}
\label{tab:hyperparameters}
\begin{center}
\scriptsize
\begin{sc}
\begin{tabular}{lll}
\hline
Hyperparameter & Value & Description \\
\hline
\textbf{Data Collection} & & \\
Train steps & $10^6$ & \\
Random steps & $5 \times 10^4$ &  Num. random steps at start \\
Num. eval episodes & $50$ & Num. trajectories in evaluation \\ 
Eval. every steps & $5\times 10^4$ & \\
Policy MLP dims & $[512, 512]$ & \\
Value Function MLP dims & $[512, 512]$ & \\
Dropout ratio & 0.1 & \\
Learning rate & $10^{-4}$ &  Value function, entropy coeff \\
 & $3\times 10^{-4}$ & Policy \\ 
Target Entropy & $-\|\mathcal{A}\|_1$ & \\
Batch size & $1024$ & \\
Discount factor $\gamma$ & $0.99$ & \\
Momentum coef  & $0.005$ & Soft update target networks \\
\hline
\textbf{Contextual World Model} & & \\
Observation Encoder MLP dims & $[256]$ &  \\
Context Encoder MLP dims & $[256, 256]$ &  \\
Latent Dynamics and Reward MLP dims  & $[512, 512]$ &  \\
Task representation dim & 10& \\
Latent dim & 512& \\
Consistency coeff & 1.0 & \\
Reward Coeff & 1.0 & \\
Discount Factor $\gamma$ & 0.99& \\
Training Horizon $H$ &  5& \\
Learning Rate & $10^{-4}$& \\
Contrastive Obj Weight $\beta$ & 1.0 & \\

Momentum coef  & $0.005$ & Soft update target encoder \\
\hline 
\textbf{Offline Meta-RL} & & \\
Meta batch size & $16$ & \\
Batch size & $256$ & \\
Num. train task & 20 & mujoco \& dmc \\
                & 40 & metaworld \\
Num. eval task & 10 &  \\
Num. OOD task & 10 & mujoco \& dmc \\
                & 0 & metaworld \\
Context size & $64$ & \\
Buffer size & $2\times 10^5$ &  for each task\\
Discount factor $\gamma$ & $0.99$ & \\
(Q-)Value Function MLP dims & $[256, 256]$ &  \\
Num. Q Functions & 2 &  \\
Policy MLP dims & $[256, 256]$ &  \\
Learning rate & $3 \times 10^{-4}$ & Policy and (Q-)Value functions\\
Expectile regression $\tau$ & 0.8 & \\
Inverse temperature $\mathcal{B}$ & 3.0 & In policy optimization \\ 
Momentum coef. & $0.005$ & soft update target networks \\ 
Num. test trajectories & 3 & k few-shot\\
\hline
\end{tabular}
\end{sc}
\end{center}
\label{tab:hyperparameters}
\end{table}

\begin{lstlisting}[basicstyle=\ttfamily\scriptsize]
Context Encoder: MLP(
  (net): Sequential(
    (0): NormedLinear(in_features=41, out_features=256, bias=True, act=Mish)
    (1): NormedLinear(in_features=256, out_features=256, bias=True, act=Mish)
    (2): NormedLinear(in_features=256, out_features=256, bias=True, act=Mish)
    (3): Linear(in_features=256, out_features=10, bias=True)
  )
)
World Model: ContextualWorldModel(
  (Encoder): MLP(
    (net): Sequential(
      (0): NormedLinear(in_features=17, out_features=256, bias=True, act=Mish)
      (1): Linear(in_features=256, out_features=512, bias=True)
    )
  )
  (Encoder_tar): MLP(
    (net): Sequential(
      (0): NormedLinear(in_features=17, out_features=256, bias=True, act=Mish)
      (1): Linear(in_features=256, out_features=512, bias=True)
    )
  )
  (Trans): MLP(
    (net): Sequential(
      (0): NormedLinear(in_features=534, out_features=512, bias=True, act=Mish)
      (1): NormedLinear(in_features=512, out_features=512, bias=True, act=Mish)
      (2): Linear(in_features=512, out_features=512, bias=True)
    )
  )
  (Reward): Vectorized [MLP(
    (net): Sequential(
      (0): NormedLinear(in_features=534, out_features=512, bias=True, act=Mish)
      (1): NormedLinear(in_features=512, out_features=512, bias=True, act=Mish)
      (2): Linear(in_features=512, out_features=1, bias=True)
    )
  )]
)
Policy: MLP(
  (net): Sequential(
    (0): NormedLinear(in_features=522, out_features=256, bias=True, act=Mish)
    (1): NormedLinear(in_features=256, out_features=256, bias=True, act=Mish)
    (2): Linear(in_features=256, out_features=12, bias=True)
  )
)
Q-Functions: Vectorized [MLP(
  (net): Sequential(
    (0): NormedLinear(in_features=534, out_features=256, bias=True, act=Mish)
    (1): NormedLinear(in_features=256, out_features=256, bias=True, act=Mish)
    (2): Linear(in_features=256, out_features=1, bias=True)
  )
), MLP(
  (net): Sequential(
    (0): NormedLinear(in_features=534, out_features=256, bias=True, act=Mish)
    (1): NormedLinear(in_features=256, out_features=256, bias=True, act=Mish)
    (2): Linear(in_features=256, out_features=1, bias=True)
  )
)]
Value Function: MLP(
  (net): Sequential(
    (0): NormedLinear(in_features=522, out_features=256, bias=True, act=Mish)
    (1): NormedLinear(in_features=256, out_features=256, bias=True, act=Mish)
    (2): Linear(in_features=256, out_features=1, bias=True)
  )
)
\end{lstlisting}

\paragraph{Hardware} 
We used AMD Instinct MI250X GPUs to run our experiments. 
All experiments have been run on a single GPU with 2 CPU workers and 32GB of RAM.

\paragraph{Hyperparameters}
\cref{tab:hyperparameters} illustrates the hyperparameters for our experiments. 
We use the same hyperparameters for all of the experiments. 
For a fair comparison, we use the same network architecture for all the baselines.

\subsection{Environments}
We evaluated all methods on 3 MuJoCo \citep{mujoco} environments, 6 Contextual DeepMind Control \citep{deepmindcontrol, hyperzero} environments, and 50 Meta-World ML1 \cite{metarl_tutorial} environments. 
\cref{tab:envs} provides details of the environments we used, including the dimensionality of the observation and action spaces, and the distribution of variation factors. Below we summarize the environments:
\begin{itemize} [left=1em]
    \item \textbf{Ant-direction:} an ant (quadruped) robot moving in different desired directions in different tasks.
    \item \textbf{Hopper-mass:} a hopper (one-legged robot) must move as fast as it can, while the mass is different for different tasks.
    \item \textbf{Walker-friction:} a walker (bi-legged) robot must move as fast as it can, while the friction coefficient is different for different tasks. 
    \item \textbf{Cheetah-speed:} a cheetah robot moving forwards/backwards with different desired speeds in different tasks. 
    \item \textbf{Finger-speed:} a planar finger robot rotating a body on an unactuated hinge with different desired angular speeds (both directions) for different tasks. 
    \item \textbf{Walker-speed:} a walker robot moving forwards/backwards with different desired speeds in different tasks. 
    \item \textbf{Cheetah-length-speed (LS):} a cheetah robot moving forwards, where torso length (change in morphology) and/or desired speed are different for each task.
    \item \textbf{Finger-length-speed (LS):} a planar finger robot rotating a body on an unactuated hinge, while the length of the link and/or desired angular speed differ in each task. 
    \item \textbf{Walker-length-speed (LS):} a walker robot moving forwards, where torso length (change in morphology) and/or desired speed are different for each task.
    \item \textbf{Meta-World ML1}: consists of 50 robotic manipulation environments featuring a Sawyer arm with various everyday objects.  Each environment consists of 50 different tasks where the position of objects and goals is different for each task. 
    \item {\textbf{Meta-World ML10:} evaluates generalization to new environments. Similar to the ML1 setting, it consists of robotic manipulation environments, where 10 environments are used for training and 5 environments are reserved for testing generalization capabilities. The testing environments share structural similarities with the training environments. During testing, no prior information about the environment (such as environment ID) is provided, and agents must identify and adapt to the environment solely based on interaction data.}
    \item {\textbf{Meta-World ML45:} evaluates generalization to new environments, similar to the ML10 setting, but with a larger and more diverse set of 45 training environments.}
\end{itemize}

\begin{table*}[h]
\label{tab:envs}
\caption{Environment used for evaluation of different methods.}
\begin{center}
\begin{sc}
\scriptsize
\begin{tabular}{lllll}
\hline
\textbf{Environment} & \textbf{Obs dim} & \textbf{Action dim}  & \textbf{ID Variation Factors}  & \textbf{OOD Variation Factors}\\
\hline
ant-direction & 29 & 8 & $\theta \sim [-\pi, \pi]$ & $\theta \sim [-1.5\pi, -\pi] \cup [\pi, 1.5 \pi]$ \\
hopper-mass & 11 & 3 & $\log f_m \sim [-1.5, 1.5]$ & $\log f_m \sim [-2, -1.5] \cup [1.5, 2]$\\
walker-friction & 17 & 6 & $\log f_f \sim [-1.5, 1.5]$ & $\log f_f \sim [-2, -1.5] \cup [1.5, 2]$\\
\hline 
cheetah-speed & 17 & 6 & $v \sim [-10, -6]\cup[-2, 2]\cup[6,10]$  &  $v \sim [-6, -2]\cup[2,6]$ \\
finger-speed & 17 & 6 & $v \sim [-15, -9]\cup[-3, 3]\cup[9,15]$  &  $v \sim [-9, -3]\cup[3,9]$ \\
walker-speed & 24 & 6 & $v \sim [-5, -3]\cup[-1, 1]\cup[3,5]$  &  $v \sim [-3, -1]\cup[1,3]$ \\
cheetah-length-speed & 17 & 6 & $ v \sim [3, 8]$  & $v\in \{1,2,9,10 \}$ \\
                    &     &   & $L \sim [0.4, 0.6]$ & $L \in \{ 0.3, 0.35, 0.65, 0.7 \}$\\ 
finger-length-speed & 9 & 2 & $ v \sim [5, 10]$   & $v\in \{3,4,11,12 \}$ \\
                    &   &   & $L \sim [0.15, 0.25]$ & $L \in \{ 0.1, 0.12, 0.27, 0.3 \}$\\ 
walker-length-speed & 24 & 6 & $ v \sim [2, 4.5]$   & $v\in \{1,1.5,5,5.5 \}$ \\
                    &    &   & $L \sim [0.2, 0.4]$ & $L \in \{ 0.1, 0.15, 0.45, 0.5 \}$\\ 
\hline
Meta-World  & 39 & 4 & &\\
\hline
\end{tabular}
\end{sc}
\end{center}
\label{tab:envs}
\end{table*}

\newpage
\subsection{Baselines} \label{sec:baseline_details}
In this section, we provide further details of the baselines we compare against.

\paragraph{FOCAL}
\citep{focal} utilizes distance metric learning to train the context encoder. 
Intuitively, the context encoder should map contexts from the same task closely in the latent task representation $\mathbf{z}$ and contexts from different tasks far apart. 
They show that pairwise contrastive learning \citep{contrastive_original} can lead to degenerate task representations, where embeddings from multiple tasks collapse into shared clusters.
To encourage discriminative task representations, FOCAL proposes a negative-power variant of the contrastive loss:
\begin{equation*} 
    \mathrm{L}_{\text{FOCAL}}(\phi) = \mathbb{1} \{ i=j \} \| \mathbf{z}^i - \mathbf{z}^j\|^2_2 
    + \mathbb{1} \{ i \ne j \} \frac{\beta}{\| \mathbf{z}^i - \mathbf{z}^j\|^2_2 + \epsilon_0},
\end{equation*}
wwhere $\phi$ denotes the parameters of the context encoder, $\mathbf{z}^i$ is the task representation corresponding to task $i$, and $\beta$ is a hyperparameter.
Although FOCAL improves over approaches that directly couple task representation learning with Q-learning, it does not explicitly model task-specific dynamics or rewards.
Moreover, during inference, it assumes access to a context set sampled from a distribution similar to the training data, which limits applicability under distribution shift.

\paragraph{CSRO} \citep{csro} utilizes adversarial training to mitigate context distribution shift \citep{smac_omrl}.
Since the learned policy may differ from the behavior policies used to collect the offline data, the context distribution encountered during evaluation can differ substantially from the training distribution.
To address this issue, CSRO minimizes the mutual information between the task representation $\mathbf{z}$ and the behavior-dependent components $(\mathbf{s}, \mathbf{a})$.
Specifically, it employs the CLUB \citep{club} upper bound on mutual information
$\mathcal{I}_{\text{club}} = \mathbb{E}_i [\log p(\mathbf{z}^i \mid \mathbf{s}^i, \mathbf{a}^i)] -\mathbb{E}_j [\log p(\mathbf{z}^j \mid \mathbf{s}^i, \mathbf{a}^i)]$.
Since $p(\mathbf{z}^i \mid \mathbf{s}^i, \mathbf{a}^i)$ is unavailable, CSRO approximates it using a variational distribution estimator trained via maximum likelihood:
\begin{equation*}
    \mathcal{L}_{\text{VD}} (\psi) =  -\mathbb{E} [\log q_\psi(\mathbf{z}^i\mid \mathbf{s}^i, \mathbf{a}^i)],
\end{equation*}
where $\psi$ denotes the parameters of the variational estimator.
To minimize the mutual information and, subsequently, the context distribution shift, the context encoder is trained to minimize
\begin{equation*}
    \mathcal{L}_{\min\text{MI}}(\phi) = \mathbb{E}_i [\log q_\psi(\mathbf{z}^i \mid \mathbf{s}^i, \mathbf{a}^i)] -\mathbb{E}_j [\log q_\psi(\mathbf{z}^j \mid \mathbf{s}^i, \mathbf{a}^i)].
\end{equation*}
CSRO trains the context encoder with $\mathcal{L}(\phi) = \mathcal{L}_{\text{FOCAL}} +  \mathcal{L}_{\min\text{MI}}$ forming an adversarial objective with variational distribution estimator $q_\psi$. 
Although CSRO improves generalization to unseen tasks compared to FOCAL, it still does not explicitly model task-dependent dynamics or rewards.

\paragraph{UNICORN} \citep{unicorn} studies context-based OMRL from an information-theoretic perspective, showing that reconstruction (predicting reward and next state) maximizes the $\mathcal{I}(Z; X_t\mid X_b)$ and contrastive learning in FOCAL maximizes $\mathcal{I}(Z; X)$, where $X_t$ and $X_b$ denote the task-dependent and behavior-dependent components of the context, respectively.
By maximizing $\mathcal{I}(Z; X_t\mid X_b)$, the reconstruction objective encourages the task representation $\mathbf{z}$ to capture task-specific dynamics and rewards, thereby reducing context distribution shift:
\begin{equation} \label{eq:reconstruction}
    \mathcal{L}_\text{Reconstruction}(\phi) = \|\mathbf{s}_{t+1}-D_\phi(\mathbf{s}_t, \mathbf{a}_t, \mathbf{z}) \|^2_2 + \lambda \|r_t-R_\phi(\mathbf{s}_t, \mathbf{a}_t, \mathbf{z}) \|^2_2,
\end{equation}
where $D_\phi$ and $R_\phi$ denote the dynamics and reward models jointly trained with the context encoder.
\textit{UNICORN-SUP} trains the context encoder using only $\mathcal{L}_{\text{Reconstruction}}$,
whereas \textit{UNICORN-SS} augments the reconstruction objective with contrastive learning:
$\mathcal{L}_{\text{Reconstruction}} + \lambda_{\text{FOCAL}} \mathcal{L}_{\text{FOCAL}}$.

\paragraph{AutoRegression:}
We extend \cref{eq:reconstruction} to perform multi-step prediction using backpropagation through time:
\begin{equation*}
    \mathcal{L}_\text{AutoRegression}(\phi) = \sum_{h=0}^{H-1} \|\mathbf{s}_{t+h+1}-D_\phi(\hat{\mathbf{s}}_{t+h}, \mathbf{a}_{t+h}, \mathbf{z}) \|^2_2 + \lambda \|r_t-R_\phi(\hat{\mathbf{s}}_{t+h}, \mathbf{a}_{t+h}, \mathbf{z}) \|^2_2,
\end{equation*}
where $\hat{\mathbf{s}}_t=\mathbf{s}_t$ and $\hat{\mathbf{s}}_{t+h+1}=D_\phi(\hat{\mathbf{s}}_{t+h}, \mathbf{a}_{t+h}, \mathbf{z})$. 
Comparing this baseline with single-step reconstruction isolates the effect of multi-step prediction on task representation learning.
Comparing it with our jointly trained world model and context encoder further highlights the importance of latent temporal consistency in task representation learning.

\clearpage

\section{Theoretical Analysis} \label{app:theory}
Our actor-critic operates in the learned latent space $(\mathbf{x}_t,\mathbf{z}_t)$ without reconstructing observations.
We give conditions under which planning/control in $(\mathbf{x},\mathbf{z})$ is sufficient and derive value-error bounds.
Our bounds build on the \emph{simulation lemma} (\eg, \citep{kearnsNearOptimalReinforcementLearning2002}), which relates value differences between two MDPs to discrepancies in reward and transition kernels.
We extend it in two steps: {\em (i)} from the original task MDP to the latent MDP induced by the observation encoder $\Phi$, and {\em (ii)} from conditioning on the true task to conditioning on the learned context representation $\mathbf{z}$.

\subsection{Preliminaries and notation}

For each task $i$, let $\mathcal{M}_i=\langle \mathcal{S},\mathcal{A},R_i,P_i,\gamma,\rho_0\rangle$ with bounded rewards
$|R_i(\mathbf{s},\mathbf{a})|\le R_{\max}$ for all $\mathbf{s}$ and $\mathbf{a}$.
We use the infinite-horizon discounted value
\begin{align}
V_i^\pi(\mathbf{s}) =
\mathbb{E} \bigg[\sum_{t\ge 0}\gamma^t R_i(\mathbf{s}_t,\mathbf{a}_t)
\mid \mathbf{s}_0=\mathbf{s},\ \mathbf{a}_t\sim \pi(\cdot\mid \mathbf{s}_t)\bigg].
\end{align}

\paragraph{Latent state and task representation.}
Our latent state is
\begin{align}
\mathbf{x}=\Phi(\mathbf{s}) \coloneqq F_\phi(\mathbf{s}) \in \mathcal{X},
\end{align}
and the context encoder maps a context set $\mathcal{H}_t$ of transitions to a task representation
\begin{align}
\mathbf{z}_t = g_\theta(\mathcal{H}_t)
\coloneqq
\mathbb{E}_{(\mathbf{s},\mathbf{a},r,\mathbf{s}')\sim \mathcal{H}_t}
\big[E_\theta(\mathbf{s},\mathbf{a},r,\mathbf{s}')\big].
\end{align}
At test time $\mathbf{z}_t$ may change as more context is accumulated.
In the bounds below, we treat $\mathbf{z}$ as fixed after a context set is collected (matching few-shot evaluation with a fixed context window).

\paragraph{World model (task-conditioned).}
Given $(\mathbf{x},\mathbf{a},\mathbf{z})$, the latent dynamics $D_\phi$ defines a distribution
$\widehat{P}_\phi(\cdot\mid \mathbf{x},\mathbf{a},\mathbf{z})$ over $\mathcal{X}$,
and the reward model defines $\widehat{R}_\phi(\mathbf{x},\mathbf{a},\mathbf{z})$.

\paragraph{Norms.}
For distributions $p,q$ over a finite set, $\|p-q\|_1=\sum_x |p(x)-q(x)|$.
For bounded functions $V:\mathcal{S}\to\mathbb{R}$, $\|V\|_\infty=\sup_{\mathbf{s}} |V(\mathbf{s})|$.

\paragraph{Simulation lemma.}
We use a standard simulation-lemma style inequality in total variation (TV).
\begin{lemma}[Simulation lemma (TV form)]
\label{lem:sim_lemma_tv}
Let $M=\langle \mathcal{S},\mathcal{A},R,P,\gamma\rangle$ and
$\widehat{M}=\langle \mathcal{S},\mathcal{A},\widehat{R},\widehat{P},\gamma\rangle$
share the same state space $\mathcal{S}$ and action space $\mathcal{A}$, with
$|R(\mathbf{s}, \mathbf{a})| \le R_{\max}$ and $|\widehat{R}(\mathbf{s}, \mathbf{a})|\le R_{\max}$ for all $(\mathbf{s}, \mathbf{a})$.
Assume for all $(\mathbf{s}, \mathbf{a})$:
\begin{align}
|R(\mathbf{s},\mathbf{a}) - \widehat{R}(\mathbf{s}, \mathbf{a})| \le \epsilon_R, \qquad
\|P(\cdot\mid \mathbf{s}, \mathbf{a}) - \widehat{P}(\cdot \mid \mathbf{s}, \mathbf{a})\|_1 \le \epsilon_P.
\end{align}
Then for any stationary policy $\pi(\mathbf{a}\mid \mathbf{s})$,
\begin{align}
\|V_M^\pi - V_{\widehat{M}}^\pi\|_\infty
\le \frac{\epsilon_R}{1-\gamma}
+ \frac{\gamma R_{\max}}{(1-\gamma)^2}\,\epsilon_P.
\label{eq:sim_tv_bound}
\end{align}
\end{lemma}

\begin{proof}
See \citep{kearnsNearOptimalReinforcementLearning2002}.
\end{proof}

\subsection{Step 1: From the original MDP to the latent MDP induced by $\Phi$}

Fix a task $i$ and suppress $i$ in notation when clear.

\paragraph{Pushforward (latent) transition.}
Given $P_i$ on $\mathcal{S}$ and $\Phi:\mathcal{S}\to\mathcal{X}$, we define the pushforward transition  for any $(\mathbf{s},\mathbf{a})$ as
\begin{align}
P_i^\Phi(\mathbf{x}' \mid \mathbf{s}, \mathbf{a})
\coloneqq
\sum_{\mathbf{s}' \in \mathcal{S}} P_i(\mathbf{s}' \mid \mathbf{s}, \mathbf{a})\,\mathbbm{1}\{\Phi(\mathbf{s}')=\mathbf{x}'\}.
\label{eq:pushforward}
\end{align}

\paragraph{Latent Markov approximation.}
Because $\Phi$ may be many-to-one, the induced process on $\mathbf{x}$ need not be Markov.
As such, we define a Markov approximation using conditional expectation within preimages of $\Phi$, given by
\begin{align}
\bar{P}_i^\Phi(\mathbf{x}' \mid \mathbf{x}, \mathbf{a})
&\coloneqq
\mathbb{E} \left[P_i^\Phi(\mathbf{x}' \mid \mathbf{s},\mathbf{a}) \mid \Phi(\mathbf{s})=\mathbf{x}\right],
\label{eq:latent_markov_P} \\
\bar{R}_i^\Phi(\mathbf{x}, \mathbf{a})
&\coloneqq \mathbb{E} \left[
R_i(\mathbf{s}, \mathbf{a}) \mid \Phi(\mathbf{s})=\mathbf{x}
\right].
\label{eq:latent_markov_R}
\end{align}

\begin{assumption}(Approximately Markov).
There exist $\epsilon_{\mathrm{abs}}^{P},\epsilon_{\mathrm{abs}}^{R}\ge 0$ such that for all $(\mathbf{s}, \mathbf{a})$ with $\mathbf{x} = \Phi(\mathbf{s})$ we have
\begin{align}
\|P_i^\Phi(\cdot \mid \mathbf{s}, \mathbf{a})-\bar{P}_i^\Phi(\cdot \mid \mathbf{x}, \mathbf{a})\|_1
\le \epsilon_{\mathrm{abs}}^{P}, \qquad
|R_i(\mathbf{s}, \mathbf{a}) - \bar{R}_i^\Phi(\mathbf{x}, \mathbf{a})|
\le \epsilon_{\mathrm{abs}}^{R}.
\end{align}
\label{eq:assump_A}
\end{assumption}
Intuitively, states that map to the same $\mathbf{x}$ should induce similar one-step latent transitions and rewards.
Essentially, $\epsilon_{\mathrm{abs}}^{P}$ represents errors in the transition dynamics due to the encoder's $\Phi$ latent space breaking the Markov property.
Similarly, $\epsilon_{\mathrm{abs}}^{R}$, quantifies how much reward-relevant information is lost by $\Phi$.

\paragraph{Latent policy lifting.}
Any latent policy $\pi(\cdot \mid \mathbf{x},\cdot)$ induces a state-space policy
$\pi^\Phi(\mathbf{a} \mid \mathbf{s}, \cdot)=\pi(\mathbf{a} \mid \Phi(\mathbf{s}), \cdot)$.

\begin{lemma}[Original MDP $\rightarrow$ latent Markov MDP]
\label{lem:orig_to_latent}
Fix task $i$ and latent policy $\pi(\mathbf{a} \mid \mathbf{x}, i)$.
Under \cref{eq:assump_A}, the value of $\pi^\Phi$ in $\mathcal{M}_i$ is close to the value of $\pi$ in the latent Markov MDP
$\bar{\mathcal{M}}_i^\Phi=\langle \mathcal{X},\mathcal{A},\bar{R}_i^\Phi,\bar{P}_i^\Phi,\gamma\rangle$:
\begin{align}
\sup_{\mathbf{s} \in \mathcal{S}} \Big|
V_{\mathcal{M}_i}^{\pi^\Phi}(\mathbf{s})
- V_{\bar{\mathcal{M}}_i^\Phi}^{\pi}(\Phi(\mathbf{s}))
\Big| \le \frac{\epsilon_{\mathrm{abs}}^{R}}{1-\gamma}
+ \frac{\gamma R_{\max}}{(1-\gamma)^2}\,\epsilon_{\mathrm{abs}}^{P}.
\label{eq:orig_to_latent_bound}
\end{align}
\end{lemma}

\begin{proof}
Consider two (latent) MDPs under the same policy $\pi$:
one uses the (state-dependent) pushforward kernel $P_i^\Phi(\cdot\mid \mathbf{s},\mathbf{a})$ and reward $R_i(\mathbf{s},\mathbf{a})$,
and the other uses $\bar{P}_i^\Phi(\cdot\mid \mathbf{x},\mathbf{a})$ and $\bar{R}_i^\Phi(\mathbf{x},\mathbf{a})$.
\cref{eq:assump_A}, bounds the one-step reward and transition discrepancies uniformly for any $\mathbf{s}$ with latent state $\mathbf{x}$.
Applying \cref{lem:sim_lemma_tv} with $\epsilon_R=\epsilon_{\mathrm{abs}}^{R}$ and $\epsilon_P=\epsilon_{\mathrm{abs}}^{P}$ yields \eqref{eq:orig_to_latent_bound}.
\end{proof}

\subsection{Step 2: From the latent MDP to the learned world model (true task ID)}

\begin{assumption}(Latent model error with true task).
There exist $\epsilon_{\mathrm{model}}^P,\epsilon_{\mathrm{model}}^R\ge 0$ such that for all $(\mathbf{x},\mathbf{a})$ we have
\begin{align}
\|\widehat{P}_\phi(\cdot \mid \mathbf{x}, \mathbf{a}, i)-\bar{P}_i^\Phi(\cdot \mid \mathbf{x}, \mathbf{a})\|_1
\le \epsilon_{\mathrm{model}}^P, \qquad
|\widehat{R}_\phi(\mathbf{x}, \mathbf{a}, i)-\bar{R}_i^\Phi(\mathbf{x}, \mathbf{a})|
\le \epsilon_{\mathrm{model}}^R.
\end{align}
\label{eq:assump_B}
\end{assumption}

\begin{lemma}[Latent Markov MDP $\rightarrow$ learned world model (true task)]
\label{lem:latent_to_model_true_task}
Fix task $i$ and latent policy $\pi(\mathbf{a}\mid \mathbf{x}, i)$.
Under \cref{eq:assump_B},
\begin{align}
\|V_{\bar{\mathcal{M}}_i^\Phi}^{\pi} - V_{\widehat{\mathcal{M}}_i}^{\pi}\|_\infty
\le \frac{\epsilon_{\mathrm{model}}^R}{1-\gamma}
+ \frac{\gamma R_{\max}}{(1-\gamma)^2} \epsilon_{\mathrm{model}}^P,
\label{eq:latent_to_model_bound}
\end{align}
where $\widehat{\mathcal{M}}_i=\langle \mathcal{X}, \mathcal{A}, \widehat{R}_\phi(\cdot, \cdot, i),\widehat{P}_\phi(\cdot \mid \cdot, \cdot, i),
\gamma \rangle$.
\end{lemma}

\begin{proof}
Apply \cref{lem:sim_lemma_tv} on $\mathcal{S}=\mathcal{X}$ with $\epsilon_R=\epsilon_{\mathrm{model}}^R$ and $\epsilon_P=\epsilon_{\mathrm{model}}^P$.
\end{proof}

\begin{theorem}[Original task MDP vs learned world model (true task)]
\label{thm:orig_vs_model_true_task}
Fix task $i$ and latent policy $\pi(\mathbf{a}\mid \mathbf{x}, i)$.
Under \cref{eq:assump_A,eq:assump_B},
\begin{align}
\sup_{\mathbf{s}\in\mathcal{S}}
\Big|
V_{\mathcal{M}_i}^{\pi^\Phi}(\mathbf{s})
-
V_{\widehat{\mathcal{M}}_i}^{\pi}(\Phi(\mathbf{s}))
\Big| \le \frac{\epsilon_{\mathrm{abs}}^{R} + \epsilon_{\mathrm{model}}^{R} }{1-\gamma}
+ \frac{\gamma R_{\max}}{(1-\gamma)^2} (\epsilon_{\mathrm{abs}}^{P} + \epsilon_{\mathrm{model}}^{P}).
\label{eq:taskID_main_bound}
\end{align}
\end{theorem}

\begin{proof}
By the triangle inequality,
\begin{align}
\Big|
V_{\mathcal{M}_i}^{\pi^\Phi}(\mathbf{s})-V_{\widehat{\mathcal{M}}_i}^{\pi}(\Phi(\mathbf{s}))
\Big| \le \Big| V_{\mathcal{M}_i}^{\pi^\Phi}(\mathbf{s})-V_{\bar{\mathcal{M}}_i^\Phi}^{\pi}(\Phi(\mathbf{s}))
\Big| + \Big|
V_{\bar{\mathcal{M}}_i^\Phi}^{\pi}(\Phi(\mathbf{s}))-V_{\widehat{\mathcal{M}}_i}^{\pi}(\Phi(\mathbf{s})) \Big|.
\end{align}
Bound the first term using \cref{lem:orig_to_latent} and the second term using \cref{lem:latent_to_model_true_task}.
\end{proof}

\subsection{Step 3: Adding the context encoder (unknown task at test time)}

So far, the latent model is conditioned on the true task ID $i$.
At meta-test time $i$ is unknown and the agent instead conditions 
on $\mathbf{z}$ inferred via our context encoder.

\begin{assumption}(Task-inference sufficiency in latent space).
There exist $\epsilon_{\mathrm{task}}^{P},\epsilon_{\mathrm{task}}^{R}\ge 0$ such that for any task $i$ and any $(\mathbf{x},\mathbf{a})$,
if $\mathbf{z}=g_\theta(\mathcal{H})$ is computed from history collected in task $i$, then
\begin{align}
\|\widehat{P}_\phi(\cdot \mid \mathbf{x}, \mathbf{a}, \mathbf{z})-\widehat{P}_\phi(\cdot \mid \mathbf{x}, \mathbf{a}, i)\|_1
\le \epsilon_{\mathrm{task}}^{P},
\qquad |\widehat{R}_\phi(\mathbf{x}, \mathbf{a}, \mathbf{z})-\widehat{R}_\phi(\mathbf{x}, \mathbf{a}, i)| 
\le \epsilon_{\mathrm{task}}^{R}.
\end{align}
\label{eq:assump_C}
\end{assumption}
This is an approximate sufficiency requirement.
That is, we assume that conditioning on $\mathbf{z}$ is nearly as informative as conditioning on $i$ for predicting latent dynamics and rewards.

\begin{lemma}[Extra value error from task inference in the learned latent MDP]
\label{lem:inference_gap}
Fix task $i$ and latent policy $\pi(\mathbf{a}\mid \mathbf{x},\mathbf{z})$.
Under \cref{eq:assump_C},
\begin{align}
\sup_{\mathbf{x}\in\mathcal{X}} \Big|
V_{\widehat{\mathcal{M}}(\mathbf{z})}^{\pi}(\mathbf{x})
- V_{\widehat{\mathcal{M}}_i}^{\pi_i}(\mathbf{x})
\Big| \le \frac{\epsilon_{\mathrm{task}}^{R}}{1-\gamma}
+ \frac{\gamma R_{\max}}{(1-\gamma)^2}\,\epsilon_{\mathrm{task}}^{P},
\label{eq:inference_bound}
\end{align}
where $\widehat{\mathcal{M}}(\mathbf{z})$ uses $(\widehat{P}_\phi(\cdot \mid \mathbf{x}, \mathbf{a}, \mathbf{z}),\widehat{R}_\phi(\mathbf{x}, \mathbf{a}, \mathbf{z}))$ and
$\pi_i(\mathbf{a} \mid \mathbf{x}) \coloneqq \pi(\mathbf{a}\mid \mathbf{x},\mathbf{z})$ is the realized policy after inferring $\mathbf{z}$.
\end{lemma}

\begin{proof}
Apply \cref{lem:sim_lemma_tv} on $\mathcal{S}=\mathcal{X}$ between MDP$_1=\widehat{\mathcal{M}}_i$ and MDP$_2=\widehat{\mathcal{M}}(\mathbf{z})$,
using \cref{eq:assump_C} to bound the one-step reward and transition discrepancies.
\end{proof}

\paragraph{Final combined bound (unknown task, learned context).}
Combining \cref{thm:orig_vs_model_true_task} and \cref{lem:inference_gap} gives, for any latent policy $\pi$,
\begin{align}
\sup_{\mathbf{s} \in \mathcal{S}} \Big|
V_{\mathcal{M}_i}^{\pi^\Phi}(\mathbf{s})
- V_{\widehat{\mathcal{M}}(\mathbf{z})}^{\pi}(\Phi(\mathbf{s}))
\Big| \le \frac{ \epsilon_{\mathrm{abs}}^{R} + \epsilon_{\mathrm{model}}^{R} + \epsilon_{\mathrm{task}}^{R}}{1-\gamma}
+ \frac{\gamma R_{\max}}{(1-\gamma)^2}
(\epsilon_{\mathrm{abs}}^{P} + \epsilon_{\mathrm{model}}^{P} + \epsilon_{\mathrm{task}}^{P}).
\label{eq:final_combined_bound}
\end{align}

\paragraph{Discussion.}
The bounds depend on:
\begin{enumerate}[leftmargin=1.2em,itemsep=0.2em]
    \item how Markov the latent state space $\mathbf{x}=\Phi(\mathbf{s})$ is for predicting future latent states and rewards (\cref{eq:assump_A}),
    \item the learned latent model error (\cref{eq:assump_B}),
    \item the quality/sufficiency of task inference via $\mathbf{z}$ (\cref{eq:assump_C}).
\end{enumerate}
Importantly, none of these require reconstructing $\mathbf{s}$ from $\mathbf{x}$.
Instead, they require that $(\mathbf{x},\mathbf{z})$ preserves predictive information for control.

\clearpage

\section{Further Results} \label{sec:furthere_results}
\Cref{sec:appendix_generalization} provides additional comparisons with baselines, evaluating few-shot and zero-shot adaptation in both in-distribution and out-of-distribution task settings.
\Cref{sec:error_details} provides more detail on estimating different error sources and reports them for additional environments. 
\Cref{sec:extra_metrics} shows task representation metrics for additional environments.  
\Cref{sec:appendix_disentanglement} reports detailed disentanglement metrics across methods, highlighting the effects of temporal consistency and contrastive learning objectives on task representation learning. \Cref{sec:plan} reports online planing with the trained world model. 
\Cref{sec:dreamerv3} compares our contextual world model with DreamerV3 \citep{dreamerv3}, a state-of-the-art model-based RL method based on recurrent state-space models (RSSMs), to assess generalization to unseen tasks from history alone. 
It also compares temporal consistency for task representation to MetaDT \citep{wang2024metadt}, a state-of-the-art transformer-based OMRL method. 

Ablation studies are presented in 
\cref{sec:ce_architecture} (different architectures for the context encoder), \cref{sec:latent_formulation} (latent space formulations), \cref{sec:contrastive} (contrastive learning and task representation $\mathbf{z}$),\Cref{sec:multi_step_tc} (Training horizon), \cref{sec:scaling} (model size and parameter count), \cref{sec:offline_rl} (choice of offline RL algorithm), and \cref{sec:bounding_z} (bounding strategies for $\mathbf{z}$).

Finally, \cref{sec:computation_cost} analyzes the computational cost of jointly training the context encoder and world model relative to prior work, demonstrating that the contextual world model is computationally efficient while achieving improved performance.

\subsection{Generalization to New Tasks} \label{sec:appendix_generalization}
\Cref{tab:result_metaworld} summarizes the results across all environments in the Meta-World ML1 benchmark.
We collected datasets for these environments at varying difficulty levels using the same RL algorithm (DroQ) with identical hyperparameters.
In certain environments, such as Assembly, the datasets lack successful trajectories.
Consequently, various OMRL methods fail to learn the corresponding tasks.
In the majority of environments, temporal consistency for joint training of the context encoder and world model outperforms the baselines, achieving higher success rates during few-shot adaptation. 

\cref{tab:zeroshot_id} summarize zero-shot adaptation performance for in-distribution tasks and \cref{tab:zeroshot_ood} summarize zero-shot adaptation performance for out-of-distribution tasks. 
For environments where some variation factors in the out-of-distribution tasks interpolate those seen during training (\eg, [Cheetah, Finger, Walker]-speed), the performance on out-of-distribution tasks is relatively close to in-distribution performance.
In contrast, when out-of-distribution generalization requires extrapolating beyond the training variation factors (\eg, [Cheetah, Finger, Walker]-LS, Ant-Dir), a larger performance gap between in-distribution and out-of-distribution tasks is observed.
Temporal consistency performs more consistently compared to baselines when generalizing to out-of-distribution tasks, where it significantly outperforms the baselines (two-sided t-test, $p<0.05$) for both in-distribution and out-of-distribution zero-shot adaptation, respectively, indicating improved generalization to OOD tasks.
These results indicate that latent temporal consistency can improve performance on in-distribution tasks by converting the observation space to a latent space and can increase generalization to out-of-distribution tasks by encouraging the context encoder to capture latent dynamics. 

\Cref{tab:ML_new_env} reports generalization to new environments in ML10 and ML45 settings. Temporal consistency improves generalization on both training and test environments, although generalization to entirely new environments remains challenging for all methods.
Increasing the diversity of training environments improves test performance.

\clearpage
\begin{table}[t]
    \caption{
        Few-shot in-distribution performance on all environments in Meta-world benchmarks. Average success rate over 6 random seeds, $\pm$ represents $95\%$ confidence intervals. 
        Since we use DroQ with the same hyperparameters to collect datasets across all environments, some environments may lack successful trajectories.  
        \textbf{Bold} indicates the highest value with statistical significance according to the t-test with p-value $< 0.05$.
    }
    \begin{center}
    \scriptsize 
    \begin{tabular}{lllllll}
    \hline
     Env                       & FOCAL           & CSRO            & UNICORN-SUP     & UNICORN-SS      & AR + InfoNCE    & TC + InfoNCE    \\
    \hline
     Assembly                  & $0.0 \pm 0.0$   & $1.7 \pm 3.3$   & $0.0 \pm 0.0$   & $0.0 \pm 0.0$   & $0.0 \pm 0.0$   & $0.0 \pm 0.0$   \\
     Basketball                & $0.0 \pm 0.0$   & $0.0 \pm 0.0$   & $0.0 \pm 0.0$   & $0.0 \pm 0.0$   & $0.0 \pm 0.0$   & $\mathbf{5.0} \pm 4.4$   \\
     Bin-picking               & $1.7 \pm 3.3$   & $6.7 \pm 4.1$   & $1.7 \pm 3.3$   & $0.0 \pm 0.0$   & $0.0 \pm 0.0$   & $0.0 \pm 0.0$   \\
     Box-close                 & $5.0 \pm 4.4$   & $16.7 \pm 10.9$ & $8.3 \pm 6.0$   & $3.3 \pm 6.5$   & $5.0 \pm 6.7$   & $21.7 \pm 6.0$  \\
     Button-press              & $93.3 \pm 6.5$  & $96.7 \pm 4.1$  & $91.7 \pm 10.6$ & $93.3 \pm 9.7$  & $86.7 \pm 10.9$ & $100.0 \pm 0.0$ \\
     Button-press-topdown      & $78.3 \pm 19.9$ & $78.3 \pm 7.9$  & $75.0 \pm 19.4$ & $90.0 \pm 7.2$  & $80.0 \pm 13.4$ & $\mathbf{100.0} \pm 0.0$ \\
     Button-press-topdown-wall & $21.7 \pm 6.0$  & $30.0 \pm 10.1$ & $43.3 \pm 13.1$ & $40.0 \pm 8.8$  & $33.3 \pm 21.3$ & $\mathbf{93.3} \pm 6.5$  \\
     Button-press-wall         & $91.7 \pm 6.0$  & $85.0 \pm 8.4$  & $80.0 \pm 16.0$ & $88.3 \pm 3.3$  & $85.0 \pm 4.4$  & $98.3 \pm 3.3$  \\
     Coffee-button             & $98.3 \pm 3.3$  & $98.3 \pm 3.3$  & $98.3 \pm 3.3$  & $98.3 \pm 3.3$  & $98.3 \pm 3.3$  & $100.0 \pm 0.0$ \\
     Coffee-pull               & $2.5 \pm 4.9$   & $0.0\pm0.0$     & $1.7 \pm 3.3$   & $0.0 \pm 0.0$   & $0.0 \pm 0.0$   & $5.0 \pm 4.4$   \\
     Coffee-push               & $15.0 \pm 12.1$ & $16.7 \pm 8.3$  & $11.7 \pm 7.9$  & $20.0 \pm 18.2$ & $11.7 \pm 6.0$  & $30.0 \pm 8.8$  \\
     Dial-turn                 & $90.0 \pm 0.0$  & $93.3 \pm 6.5$  & $75.0 \pm 9.8$  & $88.3 \pm 11.8$ & $81.7 \pm 6.0$  & $91.7 \pm 6.0$  \\
     Disassemble               & $5.0 \pm 4.4$   & $16.7 \pm 6.5$  & $18.0 \pm 7.3$  & $13.3 \pm 6.5$  & $28.3 \pm 6.0$  & $25.0 \pm 6.7$  \\
     Door-close                & $100.0 \pm 0.0$ & $100.0 \pm 0.0$ & $96.7 \pm 6.5$  & $95.0 \pm 9.8$  & $93.3 \pm 9.7$  & $100.0 \pm 0.0$ \\
     Door-lock                 & $90.0 \pm 5.1$  & $78.3 \pm 19.9$ & $86.7 \pm 4.1$  & $86.7 \pm 9.7$  & $70.0 \pm 13.4$ & $88.3 \pm 9.4$  \\
     Door-open                 & $96.7 \pm 6.5$  & $96.7 \pm 4.1$  & $60.0 \pm 18.9$ & $93.3 \pm 4.1$  & $78.3 \pm 12.8$ & $100.0 \pm 0.0$ \\
     Door-unlock               & $91.7 \pm 6.0$  & $95.0 \pm 4.4$  & $85.0 \pm 4.4$  & $90.0 \pm 8.8$  & $91.7 \pm 6.0$  & $100.0 \pm 0.0$ \\
     Drawer-close              & $98.3 \pm 3.3$  & $100.0 \pm 0.0$ & $98.3 \pm 3.3$  & $98.3 \pm 3.3$  & $96.7 \pm 4.1$  & $100.0 \pm 0.0$ \\
     Drawer-open               & $35.0 \pm 8.4$  & $28.3 \pm 11.8$ & $55.0 \pm 21.9$ & $26.7 \pm 12.0$ & $36.7 \pm 13.1$ & $56.7 \pm 8.3$  \\
     Faucet-close              & $81.7 \pm 9.4$  & $90.0 \pm 7.2$  & $71.7 \pm 11.8$ & $96.7 \pm 4.1$  & $71.7 \pm 18.5$ & $98.3 \pm 3.3$  \\
     Faucet-open               & $81.7 \pm 13.8$ & $88.3 \pm 10.6$ & $76.7 \pm 23.6$ & $91.7 \pm 7.9$  & $88.3 \pm 7.9$  & $86.7 \pm 9.7$  \\
     Hammer                    & $35.0 \pm 12.1$ & $31.7 \pm 11.8$ & $38.3 \pm 14.7$ & $36.7 \pm 14.9$ & $38.3 \pm 11.8$ & $25.0 \pm 11.0$ \\
     Hand-insert               & $18.3 \pm 9.4$  & $23.3 \pm 12.0$ & $20.0 \pm 8.8$  & $11.7 \pm 6.0$  & $21.7 \pm 6.0$  & $26.7 \pm 9.7$  \\
     Handle-press              & $91.7 \pm 6.0$  & $93.3 \pm 6.5$  & $93.3 \pm 8.3$  & $93.3 \pm 4.1$  & $91.7 \pm 6.0$  & $96.7 \pm 4.1$  \\
     Handle-press-side         & $96.7 \pm 4.1$  & $83.3 \pm 6.5$  & $95.0 \pm 4.4$  & $96.7 \pm 6.5$  & $90.0 \pm 7.2$  & $100.0 \pm 0.0$ \\
     Handle-pull               & $48.3 \pm 12.8$ & $40.0 \pm 10.1$ & $51.7 \pm 14.7$ & $43.3 \pm 12.0$ & $48.3 \pm 14.7$ & $\mathbf{73.3} \pm 4.1$  \\
     Handle-pull-side          & $71.7 \pm 13.8$ & $73.3 \pm 15.7$ & $55.0 \pm 11.0$ & $73.3 \pm 17.3$ & $45.0 \pm 13.1$ & $80.0 \pm 7.2$  \\
     Lever-pull                & $15.0 \pm 11.0$ & $30.0 \pm 7.2$  & $21.7 \pm 9.4$  & $21.7 \pm 10.6$ & $25.0 \pm 11.0$ & $30.0 \pm 8.8$  \\
     Peg-insert-side           & $23.3 \pm 14.9$ & $15.0 \pm 4.4$  & $35.0 \pm 12.1$ & $23.3 \pm 13.1$ & $23.3 \pm 12.0$ & $31.7 \pm 11.8$ \\
     Peg-unplug-side           & $66.7 \pm 9.7$  & $80.0 \pm 15.2$ & $68.3 \pm 7.9$  & $65.0 \pm 6.7$  & $56.7 \pm 16.5$ & $73.3 \pm 10.9$ \\
     Pick-out-of-hole          & $23.3 \pm 6.5$  & $30.0 \pm 16.0$ & $1.7 \pm 3.3$   & $20.0 \pm 11.3$ & $13.3 \pm 8.3$  & $6.7 \pm 6.5$   \\
     Pick-place                & $0.0 \pm 0.0$   & $0.0 \pm 0.0$   & $1.7 \pm 2.2$   & $0.0 \pm 0.0$   & $0.0 \pm 0.0$   & $1.7 \pm 2.2$   \\
     Pick-place-wall           & $0.0 \pm 0.0$   & $3.3 \pm 4.1$   & $0.0 \pm 0.0$   & $1.7 \pm 3.3$   & $0.0 \pm 0.0$   & $0.0 \pm 0.0$   \\
     Plate-slide               & $58.3 \pm 13.8$ & $50.0 \pm 7.2$  & $50.0 \pm 11.3$ & $58.3 \pm 13.8$ & $55.0 \pm 18.8$ & $56.7 \pm 6.5$  \\
     Plate-slide-back          & $10.0 \pm 7.2$  & $10.0 \pm 7.2$  & $10.0 \pm 8.8$  & $26.7 \pm 23.6$ & $26.7 \pm 10.9$ & $36.7 \pm 24.1$ \\
     Plate-slide-back-side     & $74.0 \pm 14.7$ & $70.0 \pm 29.9$ & $70.0 \pm 10.1$ & $55.0 \pm 18.8$ & $75.0 \pm 19.4$ & $58.3 \pm 13.8$ \\
     Plate-slide-side          & $55.0 \pm 18.8$ & $61.7 \pm 7.9$  & $70.0 \pm 14.3$ & $58.3 \pm 12.8$ & $56.7 \pm 4.1$  & $67.5 \pm 14.7$ \\
     Push                      & $11.7 \pm 6.0$  & $15.0 \pm 8.4$  & $8.3 \pm 7.9$   & $10.0 \pm 5.1$  & $21.7 \pm 10.6$ & $\mathbf{35.0} \pm 8.4$  \\
     Push-back                 & $16.7 \pm 4.4$  & $18.3 \pm 8.6$  & $16.7 \pm 5.6$  & $16.7 \pm 7.4$  & $15.0 \pm 4.4$  & $13.3 \pm 2.8$  \\
     Push-wall                 & $26.7 \pm 9.7$  & $30.0 \pm 18.9$ & $40.0 \pm 14.3$ & $26.7 \pm 12.0$ & $8.3 \pm 9.4$   & $\mathbf{70.0} \pm 11.3$ \\
     Reach                     & $6.7 \pm 6.5$   & $1.7 \pm 3.3$   & $5.0 \pm 4.4$   & $1.7 \pm 3.3$   & $3.3 \pm 4.1$   & $6.7 \pm 6.5$   \\
     Reach-wall                & $8.3 \pm 6.0$   & $10.0 \pm 5.1$  & $1.7 \pm 3.3$   & $2.0 \pm 3.9$   & $3.3 \pm 4.1$   & $10.0 \pm 7.2$  \\
     Shelf-place               & $8.3 \pm 6.0$   & $0.0 \pm 0.0$   & $6.7 \pm 6.5$   & $6.7 \pm 6.5$   & $5.0 \pm 4.4$   & $10.0 \pm 7.2$  \\
     Soccer                    & $18.3 \pm 7.9$  & $11.7 \pm 12.8$ & $16.7 \pm 9.7$  & $21.7 \pm 9.4$  & $23.3 \pm 10.9$ & $\mathbf{45.0} \pm 11.0$ \\
     Stick-pull                & $0.0 \pm 0.0$   & $0.0 \pm 0.0$   & $3.3 \pm 6.5$   & $0.0 \pm 0.0$   & $0.0 \pm 0.0$   & $3.3 \pm 4.1$   \\
     Stick-push                & $1.7 \pm 3.3$   & $8.3 \pm 7.9$   & $23.3 \pm 10.9$ & $6.7 \pm 6.5$   & $11.7 \pm 9.4$  & $0.0 \pm 0.0$   \\
     Sweep                     & $63.3 \pm 13.1$ & $76.7 \pm 13.1$ & $85.0 \pm 9.8$  & $68.3 \pm 11.8$ & $63.3 \pm 8.3$  & $85.0 \pm 9.8$  \\
     Sweep-into                & $66.7 \pm 14.0$ & $70.0 \pm 5.1$  & $80.0 \pm 14.3$ & $63.3 \pm 8.3$  & $75.0 \pm 12.1$ & $83.3 \pm 6.5$  \\
     Window-close              & $98.3 \pm 3.3$  & $91.7 \pm 12.8$ & $98.3 \pm 3.3$  & $95.0 \pm 6.7$  & $93.3 \pm 6.5$  & $100.0 \pm 0.0$ \\
     Window-open               & $85.0 \pm 6.7$  & $83.3 \pm 4.1$  & $88.3 \pm 7.9$  & $78.3 \pm 11.8$ & $68.3 \pm 12.8$ & $96.7 \pm 6.5$  \\
    \hline
    \end{tabular}
    \end{center}
    \label{tab:result_metaworld}
\end{table}

\begin{table}[t]
    \caption{
        Few-shot in-distribution performance on MuJoCo and Contextual-DMC benchmarks.
        \textbf{Temporal consistency for task representation learning improves in-distribution generalization}. 
        Average returns over 6 random seeds, $\pm$ represents $95\%$ confidence intervals. 
        \textbf{Bold} indicates the highest value with statistical significance according to the t-test with p-value $< 0.05$.
    }
    \label{tab:result_mujoco}
    \begin{center}
    \scriptsize
\begin{tabular}{lllllll}
\hline
 Env             & FOCAL            & CSRO             & UNICORN-SUP       & UNICORN-SS       & AR + InfoNCE      & TC + InfoNCE     \\
\hline
 Ant-dir         & $804.0 \pm 35.0$ & $798.0 \pm 39.3$ & $429.0 \pm 30.6$  & $812.9 \pm 24.5$ & $653.5 \pm 52.5$  & $\mathbf{863.1} \pm 36.2$ \\
 Cheetah-LS      & $852.2 \pm 26.4$ & $831.2 \pm 60.0$ & $832.0 \pm 52.9$  & $795.9 \pm 39.4$ & $930.3 \pm 9.4$   & $\mathbf{944.8} \pm 4.9$  \\
 Cheetah-speed   & $515.7 \pm 62.7$ & $576.3 \pm 78.2$ & $586.4 \pm 40.7$  & $554.3 \pm 71.6$ & $591.3 \pm 109.1$ & $\mathbf{764.1} \pm 39.2$ \\
 Finger-LS       & $880.8 \pm 39.2$ & $869.2 \pm 46.6$ & $753.2 \pm 56.3$  & $885.6 \pm 14.1$ & $913.0 \pm 15.2$  & $\mathbf{968.0} \pm 5.5$  \\
 Finger-speed    & $609.5 \pm 25.0$ & $631.6 \pm 56.1$ & $526.9 \pm 36.8$  & $515.9 \pm 25.5$ & $528.3 \pm 129.9$ & $\mathbf{967.4} \pm 2.0$  \\
 Hopper-mass     & $572.7 \pm 13.0$ & $476.4 \pm 68.7$ & $442.5 \pm 119.2$ & $540.9 \pm 36.5$ & $551.2 \pm 12.4$  & $\mathbf{590.6} \pm 3.5$  \\
 Walker-LS       & $875.0 \pm 51.5$ & $899.2 \pm 41.8$ & $914.2 \pm 23.0$  & $880.7 \pm 54.7$ & $917.2 \pm 19.1$  & $934.6 \pm 20.1$ \\
 Walker-friction & $532.3 \pm 46.7$ & $521.8 \pm 34.4$ & $539.1 \pm 13.7$  & $485.5 \pm 57.8$ & $526.1 \pm 47.8$  & $563.6 \pm 33.5$ \\
 Walker-speed    & $768.9 \pm 30.1$ & $771.2 \pm 20.0$ & $518.6 \pm 55.3$  & $730.7 \pm 48.6$ & $452.1 \pm 55.9$  & $\mathbf{835.7} \pm 37.3$ \\
\hline
\end{tabular}
    \end{center}
\end{table}

\begin{table}[t]
    \caption{
        Zero-shot in-distribution performance on MuJoCo and Contextual-DMC benchmarks. 
        \textbf{Temporal consistency for task representation learning improves in-distribution generalization}. 
        Average returns over 6 random seeds, $\pm$ represents $95\%$ confidence intervals.
        \textbf{Bold} indicates the highest value with statistical significance according to the t-test with p-value $< 0.05$.
    }
    \begin{center}
    \scriptsize
\begin{tabular}{lllllll}
\hline
 Env             & FOCAL            & CSRO              & UNICORN-SUP       & UNICORN-SS       & AR + InfoNCE      & TC + InfoNCE     \\
\hline
 Ant-dir         & $678.4 \pm 40.3$ & $699.3 \pm 26.6$  & $366.6 \pm 41.6$  & $668.1 \pm 46.4$ & $502.5 \pm 86.5$  & $\mathbf{755.7} \pm 43.1$ \\
 Cheetah-LS      & $825.3 \pm 36.6$ & $828.3 \pm 34.9$  & $841.9 \pm 50.9$  & $794.5 \pm 44.5$ & $860.8 \pm 25.3$  & $\mathbf{935.0} \pm 11.9$ \\
 Cheetah-speed   & $447.3 \pm 73.0$ & $556.5 \pm 31.3$  & $447.9 \pm 49.1$  & $490.2 \pm 82.1$ & $532.5 \pm 148.2$ & $\mathbf{706.4} \pm 33.1$ \\
 Finger-LS       & $863.0 \pm 58.3$ & $897.8 \pm 41.8$  & $824.0 \pm 30.6$  & $885.3 \pm 49.6$ & $931.3 \pm 8.7$   & $\mathbf{972.0} \pm 5.0$  \\
 Finger-speed    & $746.7 \pm 49.3$ & $773.7 \pm 48.5$  & $614.4 \pm 45.0$  & $671.9 \pm 54.6$ & $638.7 \pm 112.5$ & $\mathbf{943.3} \pm 8.4$  \\
 Hopper-mass     & $535.0 \pm 33.3$ & $450.8 \pm 79.2$  & $491.6 \pm 145.6$ & $533.7 \pm 38.3$ & $536.2 \pm 16.1$  & $566.0 \pm 13.5$ \\
 Walker-LS       & $898.9 \pm 30.3$ & $882.5 \pm 100.3$ & $889.9 \pm 36.9$  & $900.0 \pm 49.2$ & $871.9 \pm 32.6$  & $937.2 \pm 9.9$  \\
 Walker-friction & $522.0 \pm 32.8$ & $503.7 \pm 39.5$  & $476.7 \pm 32.4$  & $522.1 \pm 21.6$ & $510.1 \pm 81.6$  & $578.2 \pm 13.6$ \\
 Walker-speed    & $653.1 \pm 99.4$ & $767.1 \pm 31.7$  & $513.1 \pm 48.7$  & $598.6 \pm 54.4$ & $506.5 \pm 43.3$  & $785.7 \pm 50.5$ \\
\hline
\end{tabular}
    \end{center}
    \label{tab:zeroshot_id}
\end{table}

\begin{table}[t]
    \caption{
        Zero-shot out-of-distribution performance on MuJoCo and Contextual-DMC benchmarks. 
        \textbf{Temporal consistency for task representation learning improves out-of-distribution generalization}.  
        Average returns over 6 random seeds, $\pm$ represents $95\%$ confidence intervals.
        \textbf{Bold} indicates the highest value with statistical significance according to the t-test with p-value $< 0.05$.
    }    
    \begin{center}
    \scriptsize
\begin{tabular}{lllllll}
\hline
 Env             & FOCAL             & CSRO             & UNICORN-SUP        & UNICORN-SS        & AR + InfoNCE      & TC + InfoNCE     \\
\hline
 Ant-dir         & $368.8 \pm 64.2$  & $399.2 \pm 63.9$ & $-211.4 \pm 195.5$ & $405.9 \pm 38.5$  & $284.8 \pm 103.4$ & $410.7 \pm 36.9$ \\
 Cheetah-LS      & $826.6 \pm 20.6$  & $813.9 \pm 28.1$ & $795.8 \pm 44.3$   & $806.1 \pm 28.6$  & $849.3 \pm 23.6$  & $865.5 \pm 20.4$ \\
 Cheetah-speed   & $607.8 \pm 160.4$ & $603.5 \pm 96.5$ & $554.8 \pm 75.9$   & $598.8 \pm 103.4$ & $622.8 \pm 155.7$ & $756.0 \pm 92.5$ \\
 Finger-LS       & $786.8 \pm 32.6$  & $762.7 \pm 57.9$ & $691.5 \pm 54.8$   & $816.1 \pm 47.7$  & $846.1 \pm 18.4$  & $\mathbf{886.7} \pm 11.8$ \\
 Finger-speed    & $771.3 \pm 39.0$  & $822.8 \pm 32.3$ & $675.3 \pm 66.0$   & $709.6 \pm 43.6$  & $598.3 \pm 87.6$  & $\mathbf{948.1} \pm 9.3$  \\
 Hopper-mass     & $547.3 \pm 13.3$  & $543.6 \pm 32.1$ & $463.9 \pm 133.6$  & $550.7 \pm 11.3$  & $423.1 \pm 22.6$  & $583.4 \pm 4.1$  \\
 Walker-LS       & $658.0 \pm 41.0$  & $611.5 \pm 31.6$ & $649.9 \pm 50.7$   & $657.6 \pm 41.7$  & $705.7 \pm 36.2$  & $\mathbf{788.0} \pm 27.8$ \\
 Walker-friction & $473.7 \pm 36.5$  & $475.1 \pm 22.5$ & $435.2 \pm 53.8$   & $484.6 \pm 25.0$  & $475.6 \pm 33.2$  & $474.1 \pm 30.3$ \\
 Walker-speed    & $659.6 \pm 120.1$ & $767.2 \pm 24.0$ & $535.5 \pm 44.2$   & $623.7 \pm 98.0$  & $576.9 \pm 44.0$  & $772.5 \pm 56.5$ \\
\hline
\end{tabular}
    \label{tab:zeroshot_ood}
    \end{center}
\end{table}

\begin{table}[t]
    \caption{{
    \textbf{Generalization to new environments:} Temporal consistency for task representation learning demonstrates better generalization to unseen environments.
    Increasing the number of training environments can improve generalization to testing environments. 
    Average success rate over 6 random seeds, $\pm$ represents $95\%$ confidence intervals. 
    \textbf{Bold} indicates the highest value with statistical significance according to the t-test with p-value $< 0.05$.
    }}
    \label{tab:ML_new_env}
    \begin{center}
    \begin{scriptsize}
    \begin{tabular}{lllllll}
    \hline
     Env        & CSRO           & FOCAL          & UNICORN-SUP      & UNICORN-SS    & AR + InfoNCE   & TC + InfoNCE   \\
    \hline
     ML10-Train & $43.3 \pm 7.6$ & $43.3 \pm 6.5$ & $25.4 \pm 3.7$ & $43.3 \pm 3.0$ & $44.4 \pm 4.2$ & $\mathbf{83.5} \pm 6.1$ \\
     ML10-Test  & $3.3 \pm 4.1$  & $5.0 \pm 5.1$  & $2.5 \pm 3.3$  & $4.2 \pm 4.7$  & $5.5 \pm 4.9$  & $\mathbf{15.0} \pm 3.9$  \\
     ML45-Train & $32.8 \pm 2.4$ & $32.2 \pm 2.9$ & $30.7 \pm 1.7$ & $35.9 \pm 2.4$ & $42.4 \pm 2.1$ & $\mathbf{58.7} \pm 4.0$ \\
     ML45-Ood   & $8.3 \pm 6.0$  & $3.3 \pm 4.1$  & $18.3 \pm 1.2$ & $18.3 \pm 2.6$ & $5.0 \pm 4.4$  & $\mathbf{24.0} \pm 6.7$  \\
    \hline
    \end{tabular}
    \end{scriptsize}
    \end{center}
\end{table}

\clearpage
\subsection{Error Decomposition} \label{sec:error_details}
To approximately quantify each error source, we train separate networks using privileged information (the true task ID $i$, which is unavailable in the offline meta-RL setting). 
Note that while the theoretical analysis in \cref{sec:theory_main} defines each error term as $\sup \|\cdot\|_1$, we approximate these quantities empirically using the $\ell_2$-norm $\|\cdot\|_2$.
\paragraph{Abstraction error} We train a network $P_{\text{abs}}(\mathbf{x}_{t+1} \mid \mathbf{s}_t, \mathbf{a}_t, i)$ to predict the next latent vector from the current state, action, and task ID. The abstraction error is estimated as the discrepancy between this prediction and the encoded next state, where $\mathbf{x}_{t+1}:= \Phi(\mathbf{s}_{t+1})$ is the next latent state induced by the observation encoder and $\bar{P}_i^{\phi}$ is the true latent dynamics:

\begin{equation*}
    \hat{\epsilon}^P_{\text{abs}} = \left\| \underbrace{P_{\text{abs}}(\cdot \mid \mathbf{s}_t, \mathbf{a}_t, i)}_{\approx P_i^{\Phi}(\cdot \mid \mathbf{s}_t, \mathbf{a}_t)} - \underbrace{\Phi(\mathbf{s}_{t+1})}_{\approx \bar{P}_i^{\phi}(\cdot \mid \Phi(\mathbf{s}_t), \mathbf{a}_t)} \right\|_2.
\end{equation*}

\paragraph{World modeling error} We train a second network $P_{\text{model}}(\mathbf{x}_{t+1} \mid \mathbf{x}_t, \mathbf{a}_t, i)$ to predict the next latent vector from the current latent vector, action, and task ID. The world modeling error is estimated as the discrepancy between the predicted next latent state and the induced next latent state by the observation encoder $\Phi(\mathbf{s}_{t+1})$:

\begin{equation*}
\hat{\epsilon}^P_{\text{model}} = \left\| \underbrace{P_{\text{model}}(\cdot \mid \mathbf{x}_t, \mathbf{a}_t, i)}_{=\hat{P}_{\phi}(\cdot \mid \mathbf{x}_t, \mathbf{a}_t, i)} - \underbrace{\Phi(\mathbf{s}_{t+1})}_{\approx \bar{P}_i^{\phi}(\cdot \mid \Phi(\mathbf{s}_t), \mathbf{a}_t)} \right\|_2.
\end{equation*}

\paragraph{Task error} Finally, we compare the predictions of the learned world model with those of our contextual world model, which conditions on the learned task representation $\mathbf{z}$ instead of the true task ID $i$:

\begin{equation*}
\hat{\epsilon}^P_{\text{task}} = \left\| \underbrace{D_\phi(\cdot \mid \mathbf{x}_t, \mathbf{a}_t, \mathbf{z})}_{=\hat{P}_{\phi}(\cdot \mid \mathbf{x}_t, \mathbf{a}_t, \mathbf{z})} - \underbrace{P_{\text{model}}(\cdot \mid \mathbf{x}_t, \mathbf{x}_t, i)}_{=\hat{P}_{\phi}(\cdot \mid \mathbf{x}_t, \mathbf{a}_t, i)} \right\|_2.
\end{equation*}
All networks are trained simultaneously on the same transitions, and all error terms are evaluated on the same transitions.  
\cref{fig:mdp_errors} illustrates the evolution of these errors during training. 
We observe the following:
\begin{itemize}
    \item Jointly training the context encoder and the world model generally reduces all three error terms over the course of training.
    \item The estimated task error is generally larger than both the estimated abstraction and modeling errors, highlighting the importance of task representation learning.
    \item The sum of the decomposed errors upper-bounds the temporal consistency loss
    \begin{equation*}
    \hat{\epsilon}^P_{\text{abs}} + \hat{\epsilon}^P_{\text{model}} + \hat{\epsilon}^P_{\text{task}} \;\ge\; \left\| \mathbf{x}_{t+1} - D_\phi(\cdot \mid \mathbf{x}_t, \mathbf{a}_t, \mathbf{z}) \right\|_2.
    \end{equation*}
\end{itemize}

\begin{figure}[t]
    \centering
    \includegraphics[width=\linewidth]{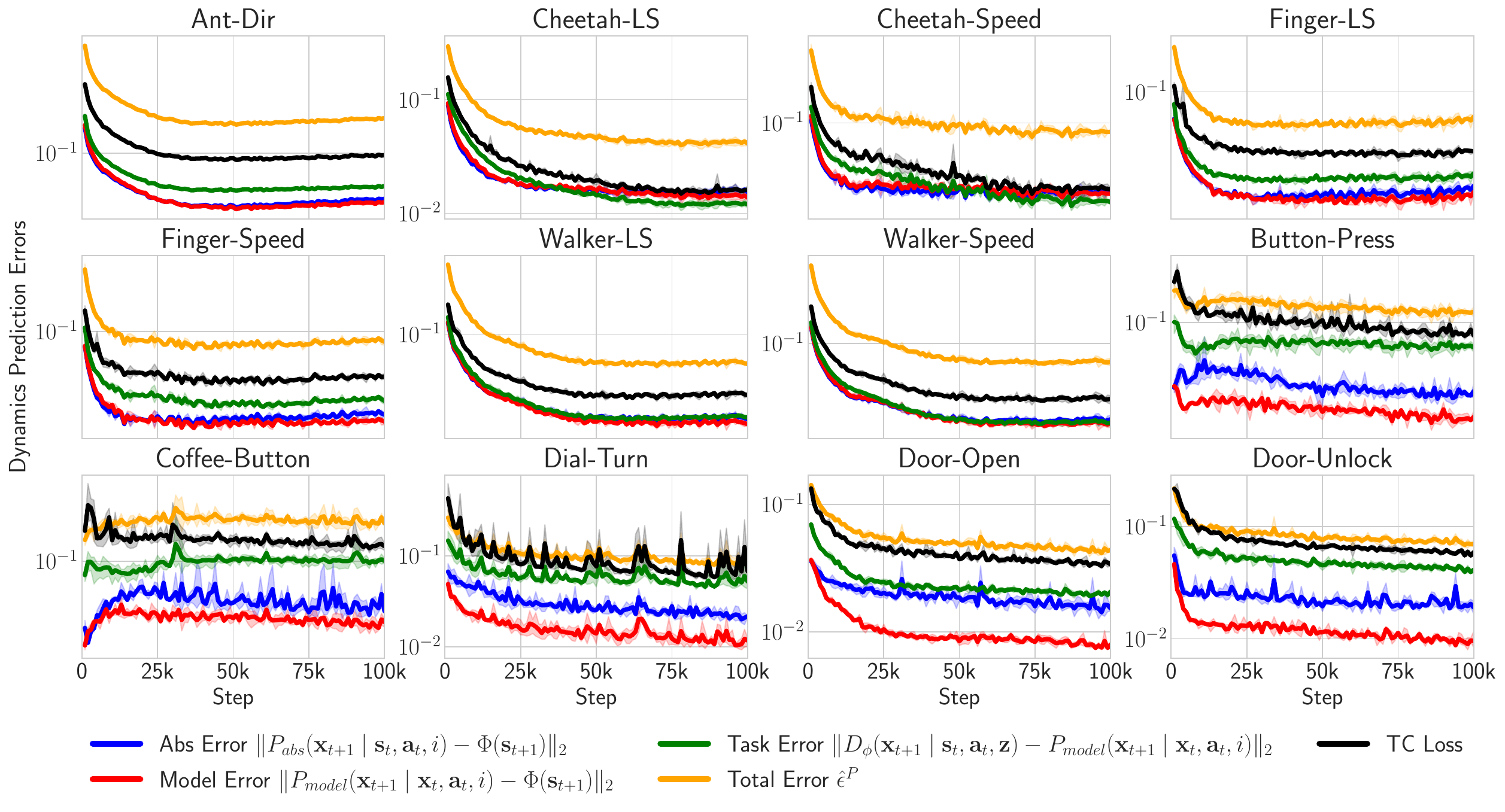}
    \caption{
    Estimating different error sources.
    \textbf{Joint training of the context encoder and world model reduces all error terms, with task error dominating; their sum upper-bounds the temporal consistency loss.}
    Abstraction error arises from mapping states to a latent space, model error reflects world model accuracy, and task error measures how well the context encoder identifies the task. 
    Separate networks ($P_{\text{abs}}$ and $P_\text{model}$) with privileged information (task ID $i$) were trained for quantifying these error terms. 
    Shaded regions denote 95\% confidence intervals over three random seeds. 
    }
    \label{fig:mdp_errors}
\end{figure}

\clearpage
\subsection{Task Representation Metrics} \label{sec:extra_metrics}
\cref{fig:extra_rep_metrics} illustrates representation metrics (feature rank, rank, and dormant ratio) along with performance on 10 MetaWorld environments. 
\begin{figure}[h]
    \centering
    \includegraphics[width=0.95\linewidth]{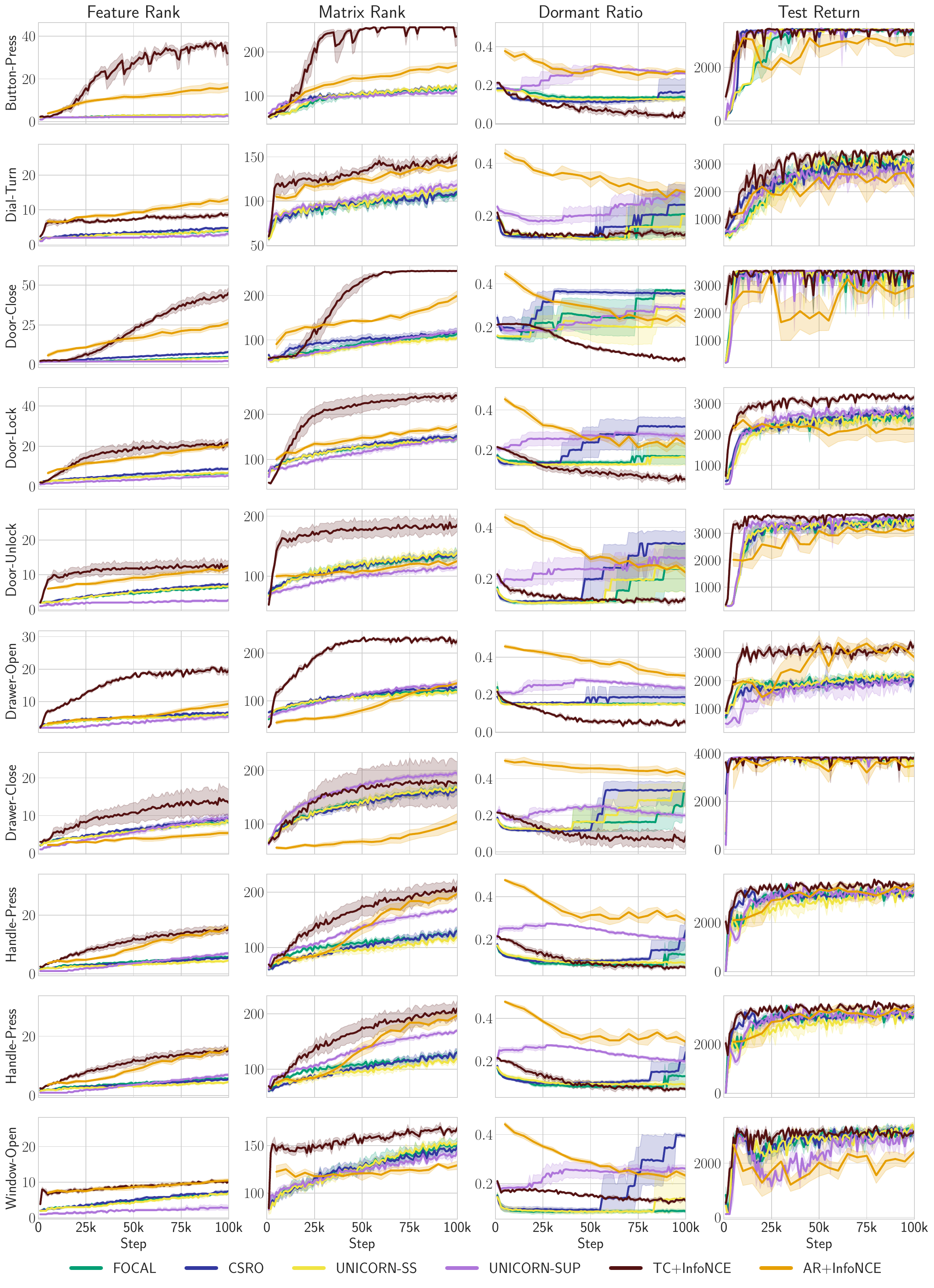}
    \caption{
        Representation metrics (feature rank, rank, and dormant ratio) of the context encoder across 10 environments from the Meta-World benchmark.
        \textbf{Temporal consistency maintains a low dormant neuron ratio and a high matrix rank while learning more diverse features (higher feature rank).}
        The shaded area represents the 95\% confidence interval across six random seeds.
    }
    \label{fig:extra_rep_metrics}
\end{figure}

\clearpage
\subsection{Disentanglement Metrics} \label{sec:appendix_disentanglement}
Previous work often relies on t-SNE visualizations to assess task separation, but such visualizations are qualitative and sensitive to hyperparameters.
We instead use disentanglement metrics, which measure how well latent dimensions align with underlying task variation factors.
We report DCI~\citep{DCI} metrics (disentanglement, completeness, informativeness) and InfoMEC~\citep{InfoMEC} metrics (modularity, explicitness, compactness).
Together, these metrics capture whether a task representation is structured, informative, and aligned with true task factors.

\cref{tab:DCI}, \cref{tab:disentanle_finger} and \cref{tab:disentanle_walker} illustrate disentanglement metrics for Cheetah-LS Finger-LS and Walker-LS environments, respectively. 
We observe similar trends as \cref{tab:DCI}, where training the context encoder based on temporal consistency (TC) results in higher disentanglement than training the context encoder with reconstruction (UNICORN-SUP), and including contrastive learning can improve task distinguishability (informativeness, explicitness). 
However, while the reconstruction objective  (UNICORN-SUP) results in higher disentanglement than contrastive objectives (FOCAL, InfoNCE) in Cheetah-LS (\cref{tab:DCI}), in [Finger, Walker]-LS reconstruction objective results in a lower disentanglement score.  

\begin{table}[ht]
    \caption{
        Disentanglement metrics (DCI, InfoMEC) for the Cheetah-length-speed (LS) environment. 
        \textbf{Latent temporal consistency disentangles the variation factors more effectively}, while \textbf{contrastive learning enhances task distinguishability}, reflected in informativeness and explicitness. 
        \textit{TC} denotes training the context encoder solely with temporal consistency~(\cref{eq:worldmodel_objective}); \textit{FOCAL} represents contrastive learning, \textit{UNICORN-SUP} indicates training with reconstruction (decoder), and \textit{AR} indicates auto-regression, training the context encoder with backpropagation through time in state space.
        Average metrics over 6 random seeds, $\pm$ represents $95\%$ confidence intervals.
        }
    \begin{center}
    \scriptsize
    \begin{tabular}{lllllll}
    \hline
                  & Disentanglement   & Completeness    & Informativeness   & Modularity      & Explicitness    & Compactness     \\
    \hline
     FOCAL        & $0.33 \pm 0.11$   & $0.25 \pm 0.08$ & $0.82 \pm 0.04$   & $0.77 \pm 0.04$ & $0.75 \pm 0.05$ & $0.02 \pm 0.01$ \\
     CSRO         & $0.31 \pm 0.05$   & $0.32 \pm 0.05$ & $0.83 \pm 0.04$   & $0.75 \pm 0.03$ & $0.78 \pm 0.03$ & $0.02 \pm 0.00$ \\
     UNICORN-SS   & $0.32 \pm 0.06$   & $0.29 \pm 0.09$ & $0.83 \pm 0.04$   & $0.76 \pm 0.01$ & $0.76 \pm 0.04$ & $0.03 \pm 0.01$ \\
     UNICORN-SUP  & $0.36 \pm 0.09$   & $0.23 \pm 0.06$ & $0.54 \pm 0.06$   & $0.76 \pm 0.05$ & $0.71 \pm 0.02$ & $0.18 \pm 0.07$ \\
     AR + InfoNCE & $0.41 \pm 0.04$   & $0.39 \pm 0.05$ & $0.78 \pm 0.01$   & $0.74 \pm 0.02$ & $0.81 \pm 0.01$ & $0.15 \pm 0.03$ \\
     \hline 
     TC           & $0.45 \pm 0.06$   & $0.49 \pm 0.07$ & $0.87 \pm 0.03$   & $0.70 \pm 0.02$ & $0.86 \pm 0.02$ & $0.23 \pm 0.02$ \\
     TC + InfoNCE & $0.50 \pm 0.05$   & $0.46 \pm 0.05$ & $0.89 \pm 0.02$   & $0.74 \pm 0.03$ & $0.87 \pm 0.01$ & $0.24 \pm 0.04$ \\
    \hline
    \end{tabular}
    \end{center}
    \label{tab:DCI}
\end{table}

\begin{table}[ht]
    \caption{
        Disentanglement metrics (DCI, InfoMEC) for the Finger Length/Speed environment. 
                Average metrics over 6 random seeds, $\pm$ represents $95\%$ confidence intervals.
    }
    \scriptsize
    \begin{center}
    \begin{tabular}{lllllll}
    \hline
                  & Disentanglement   & Completeness    & Informativeness   & Modularity      & Explicitness    & Compactness     \\
    \hline
     FOCAL        & $0.36 \pm 0.06$   & $0.30 \pm 0.07$ & $0.70 \pm 0.01$   & $0.83 \pm 0.02$ & $0.80 \pm 0.01$ & $0.04 \pm 0.01$ \\
     CSRO         & $0.36 \pm 0.03$   & $0.38 \pm 0.08$ & $0.71 \pm 0.02$   & $0.82 \pm 0.02$ & $0.79 \pm 0.01$ & $0.04 \pm 0.01$ \\
     UNICORN-SS   & $0.33 \pm 0.05$   & $0.31 \pm 0.08$ & $0.69 \pm 0.02$   & $0.79 \pm 0.02$ & $0.78 \pm 0.01$ & $0.05 \pm 0.01$ \\
     UNICORN-SUP  & $0.25 \pm 0.05$   & $0.34 \pm 0.06$ & $0.57 \pm 0.01$   & $0.85 \pm 0.02$ & $0.73 \pm 0.01$ & $0.10 \pm 0.04$ \\
     AR + InfoNCE & $0.23 \pm 0.05$   & $0.40 \pm 0.05$ & $0.63 \pm 0.02$   & $0.87 \pm 0.02$ & $0.74 \pm 0.01$ & $0.08 \pm 0.01$ \\
     \hline 
     TC           & $0.41 \pm 0.03$   & $0.43 \pm 0.02$ & $0.70 \pm 0.01$   & $0.82 \pm 0.02$ & $0.85 \pm 0.01$ & $0.13 \pm 0.01$ \\
     TC + InfoNCE & $0.46 \pm 0.06$   & $0.43 \pm 0.04$ & $0.82 \pm 0.03$   & $0.82 \pm 0.02$ & $0.87 \pm 0.02$ & $0.13 \pm 0.01$ \\
    \hline
    \end{tabular}
    \end{center}
    \label{tab:disentanle_finger}
\end{table}

\begin{table}[ht]
    \caption{
        Disentanglement metrics (DCI, InfoMEC) for the Walker Length/Speed environment. 
        Average metrics over 6 random seeds, $\pm$ represents $95\%$ confidence intervals.
    }
    \begin{center}
    \scriptsize
    \begin{tabular}{lllllll}
    \hline
                  & Disentanglement   & Completeness    & Informativeness   & Modularity      & Explicitness    & Compactness     \\
    \hline
     FOCAL        & $0.33 \pm 0.07$   & $0.31 \pm 0.06$ & $0.83 \pm 0.04$   & $0.69 \pm 0.03$ & $0.82 \pm 0.02$ & $0.03 \pm 0.01$ \\
     CSRO         & $0.42 \pm 0.08$   & $0.41 \pm 0.09$ & $0.79 \pm 0.03$   & $0.72 \pm 0.03$ & $0.82 \pm 0.02$ & $0.03 \pm 0.01$ \\
     UNICORN-SS   & $0.38 \pm 0.07$   & $0.31 \pm 0.05$ & $0.84 \pm 0.03$   & $0.73 \pm 0.03$ & $0.84 \pm 0.03$ & $0.04 \pm 0.01$ \\
     UNICORN-SUP  & $0.20 \pm 0.03$   & $0.18 \pm 0.07$ & $0.38 \pm 0.04$   & $0.80 \pm 0.04$ & $0.64 \pm 0.02$ & $0.06 \pm 0.02$ \\
     AR + InfoNCE & $0.30 \pm 0.04$   & $0.21 \pm 0.05$ & $0.46 \pm 0.04$   & $0.82 \pm 0.06$ & $0.67 \pm 0.02$ & $0.10 \pm 0.04$ \\
     \hline 
     TC           & $0.39 \pm 0.05$   & $0.27 \pm 0.06$ & $0.75 \pm 0.04$   & $0.77 \pm 0.05$ & $0.86 \pm 0.03$ & $0.14 \pm 0.02$ \\
     TC + InfoNCE & $0.40 \pm 0.09$   & $0.30 \pm 0.09$ & $0.84 \pm 0.04$   & $0.79 \pm 0.04$ & $0.88 \pm 0.02$ & $0.14 \pm 0.01$ \\
    \hline
    \end{tabular}
    \end{center}
    \label{tab:disentanle_walker}
\end{table}

\clearpage
\subsection{Planning with the World Model} \label{sec:plan}

We have shown that jointly training the context encoder and world model improves the performance of offline policy optimization. In the previous experiments, however, we used the world model solely for representation learning. In this section, we investigate whether the trained world model, conditioned on the context encoder, can also be used for online planning.

A key challenge in planning with a learned Q-function in the offline setting is value overestimation for out-of-distribution (OOD) actions. Since the actions selected by the planner may follow a different distribution from those observed in the offline dataset, the learned Q-function can assign erroneously high values to such actions. To mitigate this issue, we use conservative Q-learning (CQL \citep{cql}) to train the Q-function, encouraging more conservative value estimates for OOD actions during planning.

\cref{fig:planning_horizon} illustrates the zero-shot adaptation performance under different planning horizons using online planning. Similar to TD-MPC2 \citep{hansen2024tdmpc2} and DC-MPC \citep{dcmpc}, we use model-predictive path-integral (MPPI) control for planning while updating the task representation $\mathbf{z}$ at each step. We consider two planning variants: planning based solely on the predicted rewards and planning that additionally uses the Q-function to estimate the value of the final state at the end of the planning horizon.

The Q-function can complement short-horizon planning by providing an estimate of the long-term return beyond the planning horizon. In reward-only planning, candidate trajectories are evaluated based only on the rewards predicted within the finite planning horizon, which may favor actions that yield high immediate rewards but lead to unfavorable states in the longer term. In contrast, incorporating the Q-value of the final predicted state provides a bootstrap estimate of the future return beyond the planning horizon. This allows the planner to account for long-term consequences without explicitly extending the planning horizon, potentially improving both decision quality and computational efficiency.

Overall, increasing the planning horizon improves zero-shot adaptation performance, at the cost of increased computational complexity. Longer horizons allow the planner to account for future outcomes over a greater number of steps in the learned latent space. Furthermore, incorporating conservative Q-value estimates into the planner generally improves performance, particularly for shorter planning horizons, where the Q-function provides a more informative estimate of the long-term consequences that cannot be captured within the limited planning horizon.

\begin{figure}[h]
    \centering
    \includegraphics[width=0.95\linewidth]{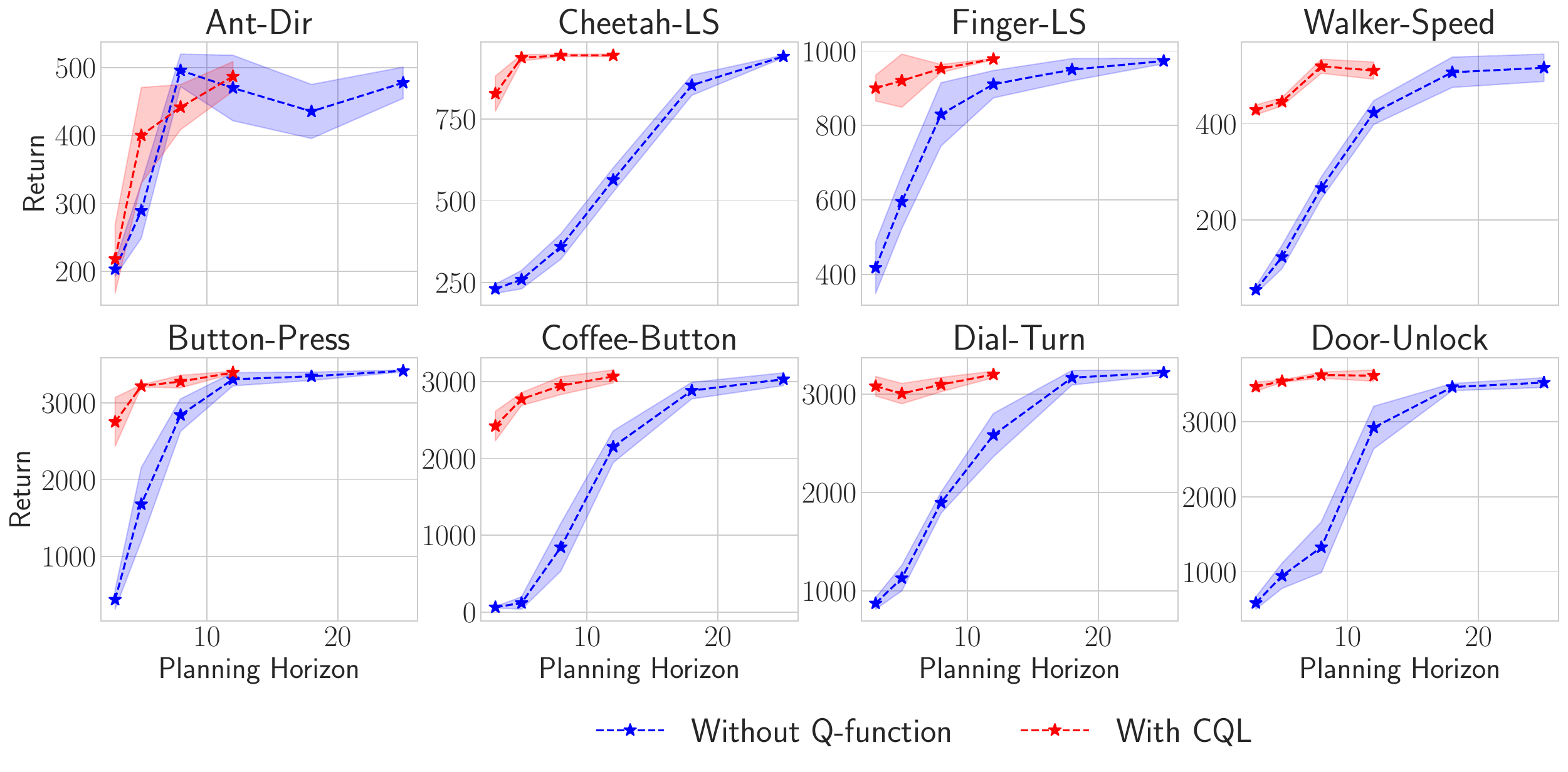}
    \caption{
        Planning with the trained world model and context encoder for zero-shot adaptation. \textbf{Increasing the planning horizon improves performance, while incorporating conservative Q-estimates further improves performance, particularly at shorter planning horizons.} Results are averaged over 6 random seeds; shaded areas represent the standard deviation. }
    \label{fig:planning_horizon}
\end{figure}

\newpage
\subsection{Comparison to DreamerV3 and MetaDT} \label{sec:dreamerv3}

DreamerV3 \citep{dreamerv3} is a model-based RL method that employs a Recurrent State-Space Model (RSSM \citep{ProbabilisticRSSM}) for latent dynamics while jointly predicting rewards, observations, and terminations. 
Its latent space includes both continuous and discrete variables.
The policy and value function are optimized within the world model. 
We use a public PyTorch implementation\footnote{https://github.com/NM512/dreamerv3-torch} with default hyperparameters.

MetaDT \citep{wang2024metadt} employs a transformer-based architecture to directly predict actions, without explicitly modeling environment dynamics or the reward function. 
It extends Prompt-DT \citep{prompting_dt} by conditioning on a sequence of task representations, where the task encoder is an RNN trained via reconstruction.

\cref{tab:dreamer} reports the zero-shot performance on in-distribution tasks. 
For DreamerV3, we reset the RSSM hidden state at the initial timestep during meta-testing. 
DreamerV3 struggles to generalize to new tasks in OMRL settings, particularly in environments where optimal policies differ significantly across tasks. 
For example, in [Cheetah, Finger, Walker]-speed environments, the agent must move both forward and backward at different speeds, and in Ant-dir environments, the agent must move in different directions. 
On the other hand, DreamerV3 shows better generalization in environments where optimal task-specific policies are more similar, such as Hopper-mass and Walker-friction. 
We also hypothesize that the policy may exploit inaccuracies in the world model, since the world model is trained solely on static datasets. 
The policy is optimized to maximize expected return under the world model’s predictions, without any penalty for acting in uncertain or poorly modeled regions. 
In online RL settings, the policy’s actions would be executed in the real environment, and the world model would be updated accordingly; however, this corrective mechanism is absent in the offline OMRL scenario.

Overall, jointly training context encoder and world model based on temporal consistency and contrastive learning outperforms Meta-DT in 5/9 environments, matches performance in 3/7, and is outperformed only in Walker-Speed. 

\begin{table}[ht]
    \caption{{
    \textbf{DreamerV3 fails to generalize in OMRL settings} and \textbf{our contextual world model outperforms MetaDT. 
    Zero-shot generalization to in-distribution tasks}. 
    Average returns/success rates over 3 random seeds, $\pm$ represents $95\%$ confidence intervals.
    }}
    \begin{center}
    \begin{small}
    \begin{tabular}{llll}
    \hline
     Environment     & TC + InfoNCE (ours)      & DreamerV3         & MetaDT            \\
    \hline
     Ant-dir         & $\mathbf{649.9 \pm 50.7}$ & $-3.6 \pm 3.0$    & $412.3 \pm 87.4$  \\
     Cheetah-LS      & $\mathbf{936.5 \pm 10.8}$ & $584.8 \pm 53.2$  & $533.8 \pm 113.8$ \\
     Cheetah-speed   & $\mathbf{664.4 \pm 51.7}$ & $178.7 \pm 20.8$  & $439.2 \pm 55.8$  \\
     Finger-LS       & $966.4 \pm 5.5$  & $438.6 \pm 85.3$  & $967.3 \pm 14.2$  \\
     Finger-speed    & $946.9 \pm 9.6$  & $187.0 \pm 72.7$  & $942.1 \pm 4.2$   \\
     Hopper-mass     & $\mathbf{579.9 \pm 9.5}$  & $555.1 \pm 3.6$   & $566.1 \pm 12.3$  \\
     Walker-friction & $\mathbf{580.6 \pm 4.7}$  & $523.6 \pm 41.1$  & $540.2 \pm 21.3$  \\
     Walker-LS       & $939.7 \pm 8.3$  & $643.6 \pm 46.0$  & $929.3 \pm 11.0$  \\
     Walker-speed    & $705.8 \pm 70.9$ & $149.9 \pm 10.3$  & $\mathbf{797.7 \pm 57.1}$  \\

    \hline
    \end{tabular}
    \end{small}
    \end{center}
    \label{tab:dreamer}
\end{table}

\newpage
\subsection{Ablation: Architecture of Context Encoder} \label{sec:ce_architecture}
Similar to previous OMRL methods, we conduct our experiments on deterministic MDPs where the reward function and transition dynamics differ across tasks. 
On these benchmarks, a common assumption is task-transition correspondence, where :
\begin{equation} \label{eq:task_assumption}
    \forall s, a \quad P_i(. \mid s, a)= P_j(.\mid s, a), R_i(. \mid s, a)=R_j(.\mid s, a) \iff \mathcal{M}_i =\mathcal{M}_j
\end{equation}  
\cref{eq:task_assumption} indicates that the context encoder can identify the task from a single transition. 
Based on this assumption, the context encoder identifies the underlying task by taking the expectation over the sequence of task representations according to \cref{eq:expectation_z}. 

\Cref{tab:architecture} compares the performance of different architectures in the context encoder. 
\textit{Attention} applies self-attention without positional encoding to compute the task representation 
via weighted averaging, operating on the same non-temporal context set as \textit{MLP}. 
By contrast, \textit(RNN) and \textit(Transformer) process sequential context batches (sequence length 5) during training, and leverage the full available sequence at inference time.

We have noticed that training an RNN-based context encoder with a temporal consistency objective of the world model can hinder performance. 
Using attention to perform weighted averaging did not provide a noticeable benefit, yet it reduced the performance in 2 environments. 
Using a transformer-based context encoder improved generalization in \textit{Finger-Speed}, we speculate that this is because when the planar robot is not in contact with the rotating body, its actions do not affect the angular velocity and thus do not influence the reward, making individual transitions less informative for task inference.

\begin{table}[h]
    \caption{
        Ablation study on the context encoder architecture for zero-shot in-distribution generalization. 
        While RNN-based encoders underperform, \textbf{transformer-based encoders perform comparably to MLPs and only show clear gains in Finger-Speed, suggesting that averaging representations from single transitions is sufficient in most environments.} Results are averaged over six seeds; $\pm$ denotes 95\% confidence intervals, and \textbf{bold} indicates statistically significant best results ($p < 0.05$, t-test).
    }
    \label{tab:architecture}
    \centering
    \scriptsize
    \begin{tabular}{lllll}
    \hline
    Env             & MLP                        & Attention         & GRU                & Transformer                \\
    \hline
    Ant-Dir         & $649.9 \pm 50.7$           & $453.5 \pm 106.1$ & $230.9 \pm 5.5$    & $593.1 \pm 60.2$           \\
    Cheetah-LS      & $936.5 \pm 10.8$           & $930.5 \pm 11.2$  & $755.4 \pm 0.0$    & $880.5 \pm 11.2$           \\
    Cheetah-Speed   & $664.4 \pm 51.7$           & $669.5 \pm 72.6$  & $427.8 \pm 25.0$   & $677.5 \pm 43.2$           \\
    Finger-LS       & $966.4 \pm 5.5$            & $967.0 \pm 10.3$  & $857.6 \pm 81.0$   & $970.0 \pm 13.1$           \\
    Finger-Speed    & $946.9 \pm 9.6$            & $693.6 \pm 146.7$ & $568.3 \pm 147.8$  & $\mathbf{970.6 \pm 3.1}$   \\
    Walker-LS       & $939.7 \pm 8.3$            & $941.5 \pm 9.5$   & $826.2 \pm 24.7$   & $937.5 \pm 8.4$            \\
    Walker-Speed    & $705.8 \pm 70.9$           & $698.7 \pm 49.2$  & $621.7 \pm 28.7$   & $715.4 \pm 53.4$           \\
    \hline
    Button-Press    & $100.0 \pm 0.0$            & $96.7 \pm 4.1$    & $35.0 \pm 29.4$    & $100.0 \pm 0.0$            \\
    Coffee-Button   & $100.0 \pm 0.0$            & $100.0 \pm 0.0$   & $88.3 \pm 6.0$     & $100.0 \pm 0.0$            \\
    Dial-Turn       & $\mathbf{100.0 \pm 0.0}$   & $83.3 \pm 10.9$   & $88.7 \pm 5.7$     & $90.0 \pm 7.1$             \\
    Door-Close      & $98.3 \pm 3.3$             & $98.3 \pm 3.3$    & $100.0 \pm 0.0$    & $98.3 \pm 3.3$             \\
    Door-Lock       & $91.7 \pm 10.6$            & $100.0 \pm 0.0$   & $90.0 \pm 0.0$     & $100.0 \pm 0.0$            \\
    Door-Open       & $98.3 \pm 3.3$             & $100.0 \pm 0.0$   & $95.0 \pm 7.0$     & $100.0 \pm 0.0$            \\
    Door-Unlock     & $100.0 \pm 0.0$            & $98.3 \pm 3.3$    & $100.0 \pm 0.0$    & $100.0 \pm 0.0$            \\
    Faucet-Close    & $98.3 \pm 3.3$             & $100.0 \pm 0.0$   & $100.0 \pm 0.0$    & $100.0 \pm 0.0$            \\
    Handle-Press    & $93.3 \pm 9.7$             & $98.3 \pm 3.3$    & $50.0 \pm 14.1$    & $100.0 \pm 0.0$            \\
    \hline
    \end{tabular}
\end{table}

\clearpage
\subsection{Ablation: Latent Space Formulation} \label{sec:latent_formulation}
We isolate the effect of different latent space formulations (classification vs.\ regression, continuous vs.\ discrete, bounded vs.\ unbounded) on the performance.
Prior work on RL uses continuous latent spaces with regression losses~\citep{tdmpc}, cosine similarity~\citep{tcrl}, or bounded and discretized representations~\citep{iqrl, hansen2024tdmpc2, dcmpc}.

\Cref{fig:ablation_worldmodels} shows that simply bounding or discretizing the latent space does not improve performance in our setting.
In contrast, formulating temporal consistency as a classification problem leads to better results.
This suggests that the cross-entropy loss, rather than discretization itself, is the main driver of performance gains as reported by \citep{dcmpc}.
Similarly, \citep{classify_value} demonstrates the advantages of classification compared to regression in training value functions.
We provide a detailed comparison in \cref{tab:world_model_ablation}.

\begin{figure}[ht]
    \centering
    \includegraphics[width=\linewidth]{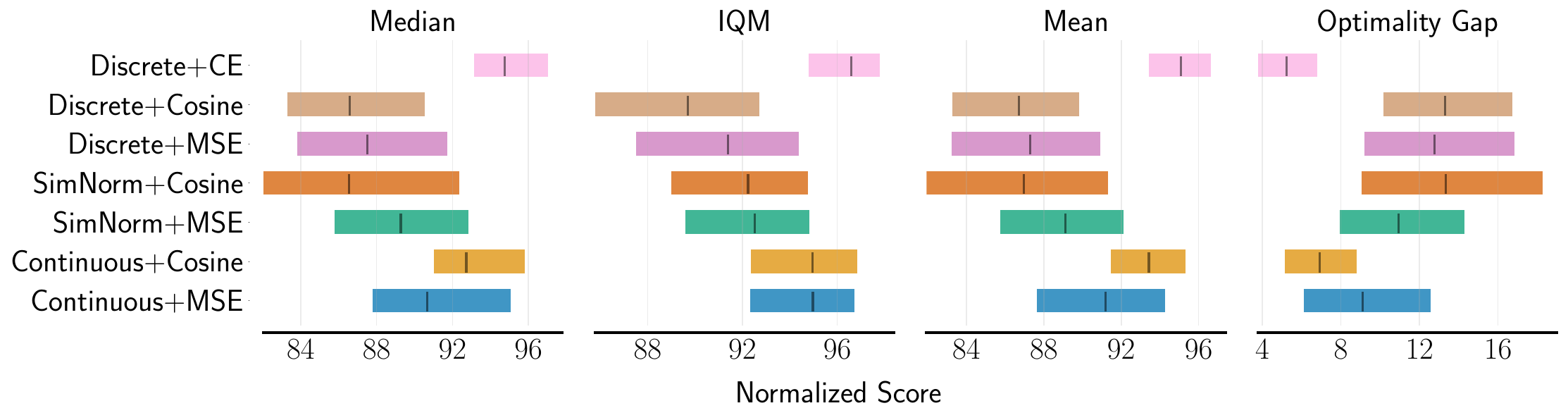}
    \caption{
        Different world model formulations: IQM and optimality gap of the normalized return for different world modeling methods, evaluated on 9 environments with 6 random seeds per environment. 
        \textbf{The main advantage of discretizing the latent space is due to classification} loss (cross-entropy). 
        Bounding or discretizing the latent space alone does not improve performance.     }
    \label{fig:ablation_worldmodels}
\end{figure}

\begin{table}[ht]
    \centering
    \caption{Ablation on latent space, different choice on latent space, and temporal consistency loss function. Average returns over 6 random seeds, $\pm$ represents $95\%$ confidence intervals. }
    \label{tab:latent_formulation}
    \scriptsize
    \begin{tabular}{llllllll}
    \hline
     Environment     & Cont MSE    & Cont Cosine   & SimNorm MSE      & SimNorm Cosine   & Dis MSE      & Dis Cosine   & Dis CE      \\
    \hline
     Ant-dir         & $841.6 \pm 25.9$  & $865.6 \pm 41.4$    & $851.5 \pm 30.9$ & $854.5 \pm 34.0$ & $784.0 \pm 54.0$  & $800.7 \pm 65.9$  & $863.1 \pm 36.2$ \\
     Cheetah-LS      & $938.0 \pm 9.2$   & $895.8 \pm 49.8$    & $898.3 \pm 25.1$ & $905.6 \pm 25.0$ & $920.0 \pm 15.2$  & $897.0 \pm 33.2$  & $944.8 \pm 4.9$  \\
     Cheetah-speed   & $645.3 \pm 127.1$ & $735.6 \pm 27.6$    & $572.2 \pm 49.6$ & $453.9 \pm 96.0$ & $647.2 \pm 63.3$  & $648.5 \pm 49.7$  & $764.1 \pm 39.2$ \\
     Finger-LS       & $972.2 \pm 6.1$   & $971.6 \pm 3.6$     & $971.6 \pm 4.8$  & $973.7 \pm 2.3$  & $974.7 \pm 2.2$   & $976.8 \pm 2.5$   & $968.0 \pm 5.5$  \\
     Finger-speed    & $949.4 \pm 2.9$   & $932.3 \pm 18.9$    & $915.0 \pm 63.1$ & $947.3 \pm 25.3$ & $929.8 \pm 39.7$  & $968.0 \pm 1.9$   & $967.4 \pm 2.0$  \\
     Hopper-mass     & $578.8 \pm 17.8$  & $557.0 \pm 31.8$    & $563.6 \pm 11.7$ & $513.2 \pm 71.7$ & $406.0 \pm 116.0$ & $429.5 \pm 61.0$  & $590.6 \pm 3.5$  \\
     Walker-friction & $559.9 \pm 20.6$  & $573.7 \pm 13.4$    & $502.1 \pm 35.5$ & $521.8 \pm 49.4$ & $549.4 \pm 28.0$  & $451.5 \pm 40.0$  & $563.6 \pm 33.5$ \\
     Walker-LS       & $927.8 \pm 23.2$  & $930.1 \pm 56.5$    & $942.0 \pm 9.1$  & $931.9 \pm 31.3$ & $944.9 \pm 4.6$   & $938.6 \pm 11.6$  & $934.6 \pm 20.1$ \\
     Walker-speed    & $813.0 \pm 30.9$  & $847.6 \pm 23.4$    & $843.1 \pm 20.6$ & $824.1 \pm 31.9$ & $814.6 \pm 28.7$  & $831.3 \pm 22.2$  & $835.7 \pm 37.3$ \\
    \hline
     Button-press  & $100.0 \pm 0.0$  & $98.3 \pm 3.3$      & $100.0 \pm 0.0$ & $96.7 \pm 6.5$   & $98.3 \pm 3.3$  & $98.3 \pm 3.3$    & $100.0 \pm 0.0$ \\
     Coffee-button & $100.0 \pm 0.0$  & $100.0 \pm 0.0$     & $100.0 \pm 0.0$ & $100.0 \pm 0.0$  & $100.0 \pm 0.0$ & $100.0 \pm 0.0$   & $100.0 \pm 0.0$ \\
     Door-close    & $100.0 \pm 0.0$  & $100.0 \pm 0.0$     & $100.0 \pm 0.0$ & $100.0 \pm 0.0$  & $100.0 \pm 0.0$ & $98.3 \pm 3.3$    & $100.0 \pm 0.0$ \\
     Door-open     & $100.0 \pm 0.0$  & $98.3 \pm 3.3$      & $100.0 \pm 0.0$ & $96.7 \pm 4.1$   & $98.3 \pm 3.3$  & $98.3 \pm 3.3$    & $100.0 \pm 0.0$ \\
     Door-unlock   & $95.0 \pm 6.7$   & $96.7 \pm 4.1$      & $98.3 \pm 3.3$  & $96.7 \pm 4.1$   & $98.3 \pm 3.3$  & $100.0 \pm 0.0$   & $100.0 \pm 0.0$ \\
     Drawer-close  & $98.3 \pm 3.3$   & $100.0 \pm 0.0$     & $100.0 \pm 0.0$ & $98.3 \pm 3.3$   & $100.0 \pm 0.0$ & $100.0 \pm 0.0$   & $100.0 \pm 0.0$ \\
     Drawer-open   & $45.0 \pm 8.4$   & $40.0 \pm 8.8$      & $63.3 \pm 10.9$ & $56.7 \pm 9.7$   & $48.3 \pm 9.4$  & $48.3 \pm 12.8$   & $51.7 \pm 6.0$  \\
     Faucet-close  & $91.7 \pm 6.0$   & $90.0 \pm 5.1$      & $98.3 \pm 3.3$  & $100.0 \pm 0.0$  & $100.0 \pm 0.0$ & $93.3 \pm 4.1$    & $100.0 \pm 0.0$ \\
     Faucet-open   & $71.7 \pm 17.8$  & $95.0 \pm 4.4$      & $86.7 \pm 12.0$ & $95.0 \pm 4.4$   & $90.0 \pm 5.1$  & $90.0 \pm 10.1$   & $91.7 \pm 6.0$  \\
     Handle-press  & $91.7 \pm 6.0$   & $93.3 \pm 4.1$      & $95.0 \pm 4.4$  & $96.7 \pm 6.5$   & $96.7 \pm 4.1$  & $98.3 \pm 3.3$    & $98.3 \pm 3.3$  \\
     Handle-pull   & $61.7 \pm 15.5$  & $70.0 \pm 14.3$     & $75.0 \pm 12.1$ & $73.3 \pm 17.3$  & $70.0 \pm 20.9$ & $86.7 \pm 8.3$    & $70.0 \pm 17.5$ \\
     Push-wall     & $53.3 \pm 17.3$  & $48.3 \pm 12.8$     & $46.7 \pm 18.0$ & $45.0 \pm 12.1$  & $48.3 \pm 19.2$ & $56.7 \pm 14.0$   & $63.3 \pm 6.5$  \\
     Sweep         & $68.3 \pm 11.8$  & $78.3 \pm 9.4$      & $73.3 \pm 9.7$  & $71.7 \pm 13.8$  & $81.7 \pm 9.4$  & $93.3 \pm 4.1$    & $85.0 \pm 8.4$  \\
     Sweep-into    & $73.3 \pm 12.0$  & $73.3 \pm 12.0$     & $71.7 \pm 15.5$ & $73.3 \pm 19.4$  & $71.7 \pm 19.9$ & $76.7 \pm 9.7$    & $80.0 \pm 11.3$ \\
     Window-close  & $100.0 \pm 0.0$  & $100.0 \pm 0.0$     & $100.0 \pm 0.0$ & $100.0 \pm 0.0$  & $98.3 \pm 3.3$  & $100.0 \pm 0.0$   & $100.0 \pm 0.0$ \\
     Window-open   & $100.0 \pm 0.0$  & $88.3 \pm 9.4$      & $96.7 \pm 6.5$  & $96.7 \pm 4.1$   & $98.3 \pm 3.3$  & $98.3 \pm 3.3$    & $98.3 \pm 3.3$  \\
    \hline
    \end{tabular}
    \label{tab:world_model_ablation}
\end{table}

\clearpage
\subsection{Ablation: Contrastive Learning}\label{sec:contrastive}
In this section, we perform an ablation study on the contrastive and world modeling objectives used to train the context encoder. 
We compare two contrastive learning objectives commonly employed in OMRL: InfoNCE and FOCAL \cite{focal} (also referred to as distance metric learning). 
The FOCAL objective is defined as:
\begin{equation} \label{eq:focal}
    \mathrm{L}_{\text{FOCAL}}(\phi) = \mathbb{1} \{ i=j \} \| z^i - z^j\|^2_2 
    + \mathbb{1} \{ i \ne j \} \frac{\beta}{\| z^i - z^j\|^2_2 + \epsilon_0}.
\end{equation} 
The InfoNCE objective in our experiment is defined as:
\begin{align}
\mathcal{L}_{\text{Contrastive}}(\theta) =
- \sum_i \log
\frac{S(\mathbf{z}^i, \bar{\mathbf{z}}^i)}
{\sum_j S(\mathbf{z}^i, \bar{\mathbf{z}}^j)},
\label{eq:infonce}
\end{align}
where $\bar{\mathbf{z}}^i = \lambda \mathbf{z}^i + (1-\lambda) \bar{\mathbf{z}}^i$ is the moving average of task representations controlled by $\lambda$,
and $S(\mathbf{z}^i, \mathbf{z}^j) = \exp(-\|\mathbf{z}^i - \mathbf{z}^j\|_2^2 / \alpha)$ measures similarity.
Positive samples are obtained from the moving average of the same task representation, while negatives are drawn from other tasks. 
The moving average stabilizes training by smoothing updates. 
This objective provides a lower bound on the mutual information between tasks and task representations, $I(\mathbf{z};\mathcal{M})$ \citep{dora}. 

For a fair comparison, we convert the observation space to a latent space for all the methods with the same world model. 
\cref{tab:contrastive_id} reports few-shot in-distribution testing and \cref{tab:contrastive_ood} reports few-shot out-of-distribution testing.  
Here, \textit{TC} refers to training the context encoder with the world modeling objective (latent temporal consistency), and \textit{+} indicates the combination of objectives. 
For environments in Meta-World benchmarks (last 6 rows in \cref{tab:contrastive_id}), we observe no significant difference in the performance of different objectives. 
Training the context encoder solely with the temporal consistency is insufficient, as it fails to distinguish between different tasks.  
This limitation is particularly pronounced in environments where variation factors affect only the reward function, rather than the transition dynamics (\eg, Ant-dir, where the desired forward direction varies, or environments requiring the inference of desired speed).  
Only using contrastive learning results in good performance across most tasks, while InfoNCE outperforms FOCAL significantly in certain environments, especially on out-of-distribution testing. 
Adding the world modeling objective to the contrastive objective has an insignificant impact on in-distribution performance; however, it can improve generalization to out-of-distribution tasks for certain environments. 
We employ the same relative weighting of the contrastive objective with respect to the world modeling objective across all environments.  
Overall, combining InfoNCE with the world modeling objective produces more robust results across environments compared to combining FOCAL with the world modeling objective.

\begin{table}[ht]
    \caption{
    Ablation on contrastive learning and world modeling, few-shot in-distribution performance. 
    Average returns/success rates over 6 random seeds, $\pm$ represents $95\%$ confidence intervals. 
    \textbf{Bold} indicates the highest value with statistical significance according to the t-test with p-value $< 0.05$.
    }
    \begin{center}
    \small
    \begin{tabular}{llllll}
    \hline
     Environment    & FOCAL             & InfoNCE          & TC                & TC+FOCAL          & TC+InfoNCE      \\
    \hline
     Ant-dir       & $841.6 \pm 31.1$  & $857.7 \pm 42.3$ & $487.5 \pm 91.8$  & $859.0 \pm 20.5$  & $863.1 \pm 36.2$ \\
     Cheetah-LS    & $940.0 \pm 16.2$  & $937.7 \pm 17.7$ & $941.5 \pm 16.6$  & $933.2 \pm 15.0$  & $944.8 \pm 4.9$  \\
     Cheetah-speed & $721.1 \pm 54.2$  & $727.3 \pm 23.8$ & $395.0 \pm 36.2$  & $711.2 \pm 94.0$  & $764.1 \pm 39.2$ \\
     Finger-LS     & $971.0 \pm 10.5$  & $968.4 \pm 10.8$ & $974.6 \pm 5.5$   & $973.7 \pm 2.8$   & $968.0 \pm 5.5$  \\
     Finger-speed  & $789.7 \pm 189.6$ & $958.1 \pm 5.7$  & $706.1 \pm 150.3$ & $770.0 \pm 185.2$ & $\textbf{967}.4 \pm 2.0$  \\
     Walker-LS     & $929.8 \pm 24.4$  & $947.7 \pm 15.1$ & $904.5 \pm 24.9$  & $928.9 \pm 21.6$  & $934.6 \pm 20.1$ \\
     Walker-speed  & $622.6 \pm 65.9$  & $842.2 \pm 35.6$ & $522.1 \pm 95.4$  & $552.0 \pm 82.4$  & $835.7 \pm 37.3$ \\
    \hline
     Button-press  & $100.0 \pm 0.0$ & $100.0 \pm 0.0$ & $98.3 \pm 3.3$  & $100.0 \pm 0.0$ & $100.0 \pm 0.0$ \\
     Coffee-button & $100.0 \pm 0.0$ & $100.0 \pm 0.0$ & $100.0 \pm 0.0$ & $100.0 \pm 0.0$ & $100.0 \pm 0.0$ \\
     Dial-turn     & $93.3 \pm 4.1$  & $93.3 \pm 6.5$  & $98.3 \pm 3.3$  & $96.7 \pm 4.1$  & $98.3 \pm 3.3$  \\
     Door-open     & $100.0 \pm 0.0$ & $100.0 \pm 0.0$ & $100.0 \pm 0.0$ & $96.7 \pm 6.5$  & $100.0 \pm 0.0$ \\
     Door-unlock   & $98.3 \pm 3.3$  & $95.0 \pm 6.7$  & $98.3 \pm 3.3$  & $98.3 \pm 3.3$  & $100.0 \pm 0.0$ \\
     Handle-press  & $96.7 \pm 4.1$  & $96.7 \pm 4.1$  & $95.0 \pm 6.7$  & $95.0 \pm 6.7$  & $98.3 \pm 3.3$  \\
    \hline
    \end{tabular}
    \end{center}
    \label{tab:contrastive_id}
\end{table}

\begin{table}[ht]
    \caption{
    Ablation on contrastive learning and world modeling, few-shot out-of-distribution performance. 
    Average returns over 6 random seeds, $\pm$ represents $95\%$ confidence intervals. 
    \textbf{Bold} indicates the highest value with statistical significance according to the t-test with p-value $< 0.05$. 
    \textbf{Combining contrastive learning with latent temporal consistency enhances generalization to out-of-distribution tasks}. 
    }
    \begin{center}
    \small
    \begin{tabular}{llllll}
    \hline
     Environment    & FOCAL             & InfoNCE           & TC                & TC+FOCAL          & TC+InfoNCE      \\
    \hline
     Ant-dir       & $529.0 \pm 41.0$  & $363.2 \pm 173.9$ & $203.6 \pm 102.9$ & $540.8 \pm 111.2$ & $501.8 \pm 92.4$ \\
     Cheetah-LS    & $864.4 \pm 21.7$  & $864.7 \pm 29.9$  & $867.2 \pm 7.5$   & $876.1 \pm 24.7$  & $860.6 \pm 10.2$ \\
     Cheetah-speed & $729.0 \pm 42.4$  & $754.9 \pm 70.6$  & $486.8 \pm 111.1$ & $908.7 \pm 74.2$  & $\textbf{967.7} \pm 10.5$ \\
     Finger-LS     & $836.3 \pm 45.5$  & $809.6 \pm 53.1$  & $838.3 \pm 64.8$  & $860.0 \pm 42.7$  & $850.9 \pm 41.5$ \\
     Finger-speed  & $793.8 \pm 189.0$ & $868.5 \pm 10.3$  & $766.6 \pm 176.2$ & $755.7 \pm 185.0$ & $\textbf{978.5} \pm 6.0$  \\
     Walker-LS     & $740.2 \pm 52.9$  & $793.2 \pm 62.2$  & $757.0 \pm 65.2$  & $738.9 \pm 51.8$  & $792.3 \pm 41.3$ \\
     Walker-speed  & $619.4 \pm 111.4$ & $782.7 \pm 32.2$  & $507.7 \pm 103.4$ & $568.7 \pm 128.5$ & $\textbf{833.2} \pm 64.5$ \\
    \hline
    \end{tabular}
    \end{center}
    \label{tab:contrastive_ood}
\end{table}

\subsection{Ablation: Multi-Step Temporal Consistency} \label{sec:multi_step_tc}

In our framework, we jointly train the world model and the context encoder using a multi-step temporal consistency objective in the latent space (\cref{eq:worldmodel_objective}). In this section, we provide an ablation study on the training horizon $h$ used in this objective.

\Cref{fig:training_horizon} illustrates the impact of the training horizon $h$ on few-shot in-distribution adaptation performance. Overall, our framework achieves competitive performance across different training horizons. Increasing the training horizon initially improves performance, but performance degrades when the horizon becomes too large. We hypothesize that excessively long training horizons can make optimization more challenging due to backpropagation through time over longer sequences, potentially leading to unstable gradients and making the joint training of the world model and context encoder more difficult.
 
\begin{figure}[h]
    \centering
    \includegraphics[width=0.95\linewidth]{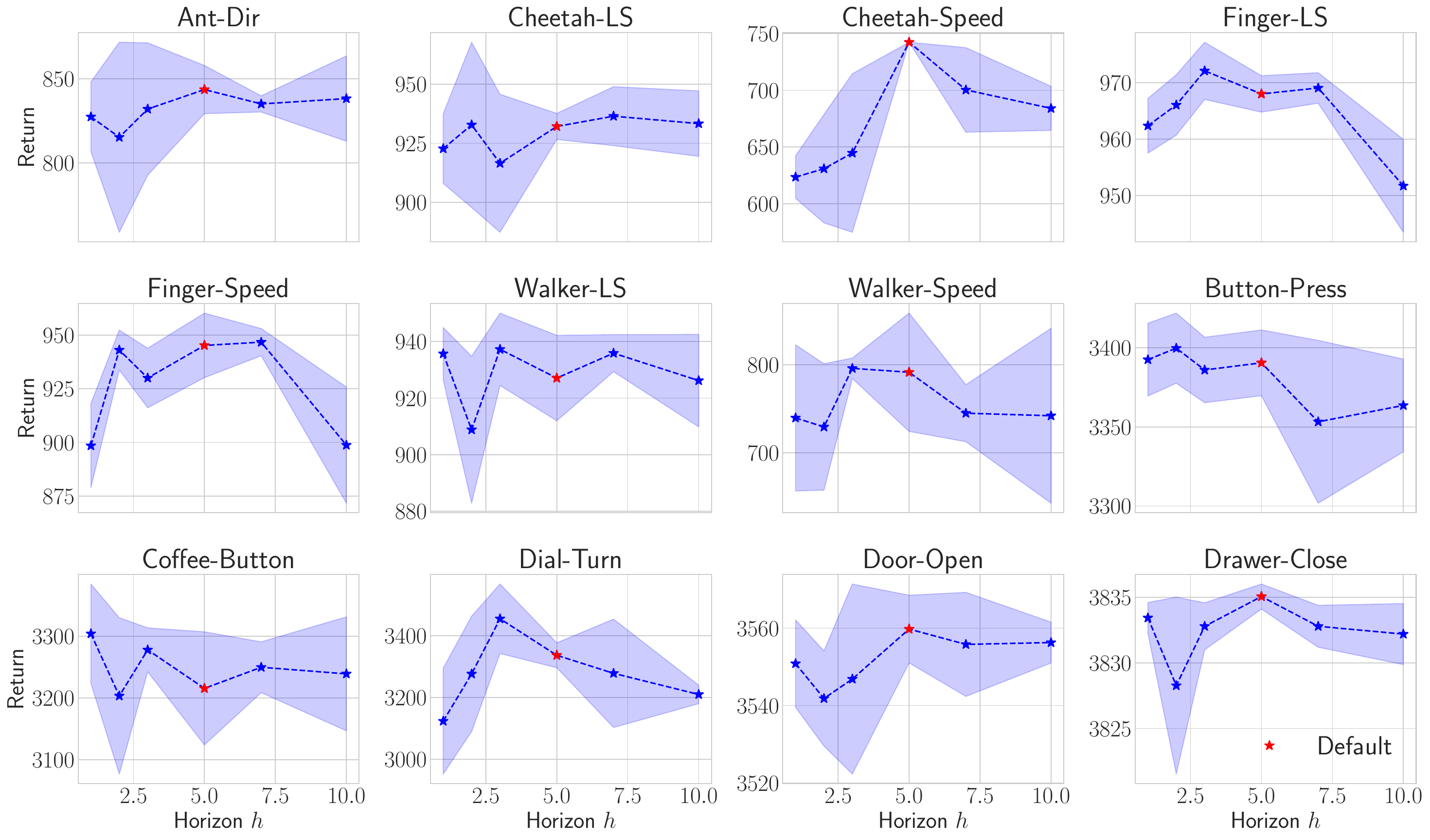}
    \caption{
        Impact of training horizon used in multi-step temporal consistency on few-shot performance.
        Results are averaged over 6 random seeds; shaded areas represent the standard deviation. 
    }
    \label{fig:training_horizon}
\end{figure}

\clearpage

\subsection{Ablation: Number of Parameters} \label{sec:scaling}
{
In this section, we investigate how increasing the number of trainable parameters affects performance across different methods. 
In our experiments, we jointly train the context encoder and world model, where the world model maps the observation space to a latent space, which increases the total number of parameters. 
\cref{fig:Scaling} illustrates the few-shot in-distribution performance of each method for different model sizes.
For the baselines, we vary the number of hidden layers in $\{2, 3\}$ and the number of hidden units in $\{256, 512, 1024\}$, resulting in six model sizes (default is two hidden layers with 256 neurons). 
To ensure architectural consistency, we apply Layer Normalization and the Mish activation function across all baseline networks. 
For our contextual world model, we set the latent dimension and the number of neurons in the dynamic head of the world model to $\{64, 128, 256, 512\}$ while setting the number of neurons in the offline RL (IQL) networks to $\{64, 128, 256, 512\}$ (6 combinations, default is 512 latent dimensions and neurons in the dynamic head, 256 neurons for the offline RL networks). 
Our contextual world model exhibits better scaling with model size: performance generally improves as the number of parameters increases. 
However, for smaller model sizes, our contextual world model underperforms the baselines in some environments.
}

\begin{figure}[ht]
    \centering
    \includegraphics[width=0.95\linewidth]{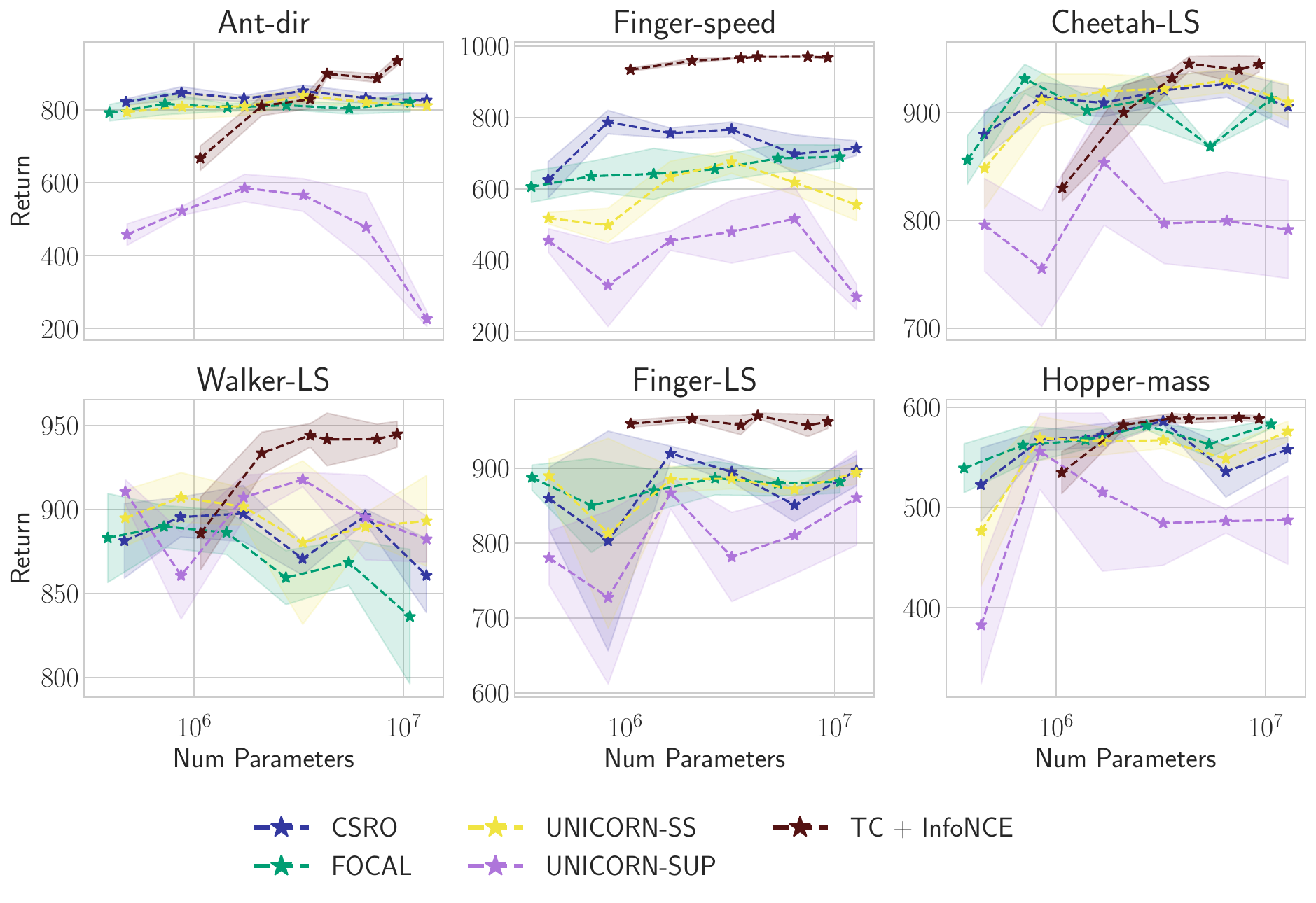}
    \caption{{
    \textbf{Jointly training the world model and context encoder scales more effectively.} Few-shot generalization to in-distribution tasks for different model sizes. 
    The shaded area represents the standard deviation over 6 random seeds. 
    }}
    \label{fig:Scaling}
\end{figure}

\clearpage
\subsection{Ablation: Offline RL} \label{sec:offline_rl}
{
As described in \cref{sec:offline_rl_IQL}, policy optimization with offline data requires regularization to avoid OOD action selection when computing the target for the value function. 
Offline RL methods address this issue in different ways, and in principle, any offline RL method can be used for policy optimization in our experiments. 
By default, we use Implicit Q-Learning (IQL) for all methods, which predicts an upper expectile of the TD targets in SARSA style without querying OOD actions. 
We also evaluate Conservative Q-Learning (CQL, \citep{cql}) and TD3+BC \citep{td3bc} for policy optimization, summarized in \cref{tab:offline_rl}. 
CQL regularizes the value function by reducing the Q-value for OOD actions, resulting in a pessimistic value function. 
TD3+BC, on the other hand, regularizes the policy to stay close to the behavior policy by adding a behavior cloning objective to the policy optimization. 
We used one set of hyperparameters (default values) for all methods without further fine-tuning. 
We find that IQL in general is more robust, performing well in diverse environments. 
CQL generally performs on par with IQL, even outperforming it significantly in two environments. 
However, the computation cost of CQL is generally higher than that of IQL. 
We sometimes observe a performance drop when training for a larger number of steps. 
TD3+BC generally has lower performance than CQL and IQL in our settings. 
We hypothesize that fine-tuning the regularization weight for each environment can increase the performance.
}

\begin{table}[ht]
    \caption{{
    \textbf{Ablation: different offline RL methods for policy optimizations.}
    Few-shot generalization to in-distribution tasks. 
    Average returns/success rates over 6 random seeds, $\pm$ represents $95\%$ confidence intervals. 
    \textbf{Bold} indicates the highest value with statistical significance according to the t-test with p-value $< 0.05$.
    }}
    \label{tab:offline_rl}
    \begin{center}
    \begin{small}
        
    \begin{tabular}{llll}
    \hline
     Environment     & CQL               & IQL (Def)        & TD3+BC            \\
    \hline
     Ant-dir         & $780.7 \pm 54.4$  & $\mathbf{863.1} \pm 36.2$   & $731.8 \pm 29.8$  \\
     Cheetah-LS      & $949.8 \pm 12.1$  & $943.0 \pm 11.5$ & $846.2 \pm 60.6$  \\
     Cheetah-speed   & $\mathbf{813.3} \pm 3.5$   & $751.2 \pm 27.9$ & $620.8 \pm 133.0$ \\
     Finger-LS       & $\mathbf{987.5} \pm 3.8$   & $957.1 \pm 23.9$ & $330.7 \pm 26.7$  \\
     Finger-speed    & $970.2 \pm 4.3$   & $962.0 \pm 9.2$  & $208.7 \pm 31.8$  \\
     Hopper-mass     & $584.2 \pm 1.1$   & $587.5 \pm 4.9$  & $159.4 \pm 26.1$  \\
     Walker-friction & $585.6 \pm 32.6$  & $563.6 \pm 33.5$ & $535.6 \pm 26.6$  \\
     Walker-LS       & $948.0 \pm 7.2$   & $937.1 \pm 16.6$ & $710.2 \pm 84.9$  \\
     Walker-speed    & $814.0 \pm 42.4$  & $827.6 \pm 34.6$ & $655.7 \pm 108.2$ \\
     \hline
     Button-press  & $95.0 \pm 9.8$  & $100.0 \pm 0.0$ & $80.0 \pm 5.1$  \\
     Coffee-button & $100.0 \pm 0.0$ & $100.0 \pm 0.0$ & $91.7 \pm 7.9$  \\
     Dial-turn     & $90.0 \pm 19.6$ & $91.7 \pm 10.6$ & $95.0 \pm 4.4$  \\
     Door-open     & $100.0 \pm 0.0$ & $100.0 \pm 0.0$ & $23.3 \pm 4.1$   \\
     Door-unlock   & $100.0 \pm 0.0$ & $100.0 \pm 0.0$ & $51.7 \pm 10.6$ \\
     Handle-press  & $60.0 \pm 19.6$ & $\mathbf{93.3} \pm 4.1$  & $76.7 \pm 6.5$  \\
    \hline
    \end{tabular}
    \end{small}
    \end{center}
\end{table}

\clearpage
\subsection{Ablation: Bounding the Task Representation} \label{sec:bounding_z}
In this section, we evaluate how bounding the task representation $z$ affects the generalization. 
By default, we utilize $\mathrm{Tanh}$ as the activation function for the context encoder, consistent with prior OMRL methods.  
\cref{tab:Ablation_task_representation} summarize the results. 
We compare unbounded representations (Identity), $\ell_2$-normalization, FSQ, and $\mathrm{Tanh}$ for bounding the latent space. 
We investigate whether discretizing the task representation with FSQ influences generalization.
The results indicate that bounding the task representation significantly enhances generalization in certain environments (\eg, Ant-dir, Cheetah-speed).
Discretizing the task representation with FSQ yields no advantages over $\mathrm{Tanh}$, and $\ell_2$-normalization is less robust across different environments. 

\begin{table}[ht]
    \caption{\textbf{Bounding the task representation enables better generalization in certain environments}, though discretizing with FSQ yields no advantages. Average returns and success rates across 6 random seeds; $\pm$ represents 95\% confidence intervals.}
    \small
    \begin{center}
    \begin{tabular}{lllll}
    \hline
     Environment   & Identity          & $\ell_2$-Norm           & FSQ                & Tanh   \\
    \hline
     Ant-dir       & $452.7 \pm 121.3$ & $838.0 \pm 72.9$  & $866.0 \pm 46.2$   & $863.1 \pm 36.2$     \\
     Cheetah-LS    & $933.8 \pm 10.5$  & $932.3 \pm 16.7$  & $938.3 \pm 5.8$    & $944.8 \pm 4.9$      \\
     Cheetah-speed & $670.3 \pm 75.1$  & $778.0 \pm 27.8$  & $730.7 \pm 30.2$   & $764.1 \pm 39.2$     \\
     Finger-LS     & $962.2 \pm 3.8$   & $956.6 \pm 11.4$  & $969.6 \pm 5.4$    & $968.0 \pm 5.5$      \\
     Finger-speed  & $715.4 \pm 133.6$ & $855.1 \pm 83.8$  & $950.4 \pm 30.4$   & $967.4 \pm 2.0$      \\
     Walker-LS     & $922.7 \pm 20.0$  & $924.7 \pm 15.7$  & $932.4 \pm 42.9$   & $934.6 \pm 20.1$     \\
     \hline
     Button-press  & $98.3 \pm 3.3$    & $100.0 \pm 0.0$   & $98.3 \pm 3.3$     & $100.0 \pm 0.0$      \\
     Coffee-button & $100.0 \pm 0.0$   & $100.0 \pm 0.0$   & $100.0 \pm 0.0$    & $100.0 \pm 0.0$      \\
     Dial-turn     & $96.7 \pm 4.1$    & $95.0 \pm 4.4$    & $78.0 \pm 33.6$    & $96.7 \pm 4.1$       \\
     Door-open     & $95.0 \pm 6.7$    & $98.3 \pm 3.3$    & $100.0 \pm 0.0$    & $100.0 \pm 0.0$      \\
     Door-unlock   & $98.3 \pm 3.3$    & $96.7 \pm 4.1$    & $98.3 \pm 3.3$     & $100.0 \pm 0.0$      \\
    \hline
    \end{tabular}
    \end{center}
    \label{tab:Ablation_task_representation}
\end{table}

\clearpage
\subsection{Computation Cost} \label{sec:computation_cost}
{
\cref{fig:time_comparison} compares the computation cost for different methods. 
All experiments are conducted using the same hardware, as described in \cref{sec:appendix_details}, to ensure a fair comparison. 
Although our contextual world model has a longer training time per step, it generally converges faster than the baselines, compensating for the higher per-step computational cost. 
During testing, our contextual world model is slightly slower because it first maps the observation to the latent space using the observation encoder, after which the policy produces actions.
}

{
UNICORN-SUP trains the context encoder solely using the prediction loss and has the lowest computational cost per training step. 
However, incorporating contrastive learning can improve task representation learning and, consequently, generalization to new tasks. 
CSRO and UNICORN-SS aim to reduce context distribution shift by minimizing a CLUB upper bound of mutual information and by adding a prediction loss, respectively. 
These approaches require additional networks, increasing their computational cost per training step. 
During test time, all baselines have the same computational cost since the policy and context encoder architectures are identical across the baseline methods.
}

\begin{figure}[ht]
    \centering
    \includegraphics[width=0.95\linewidth]{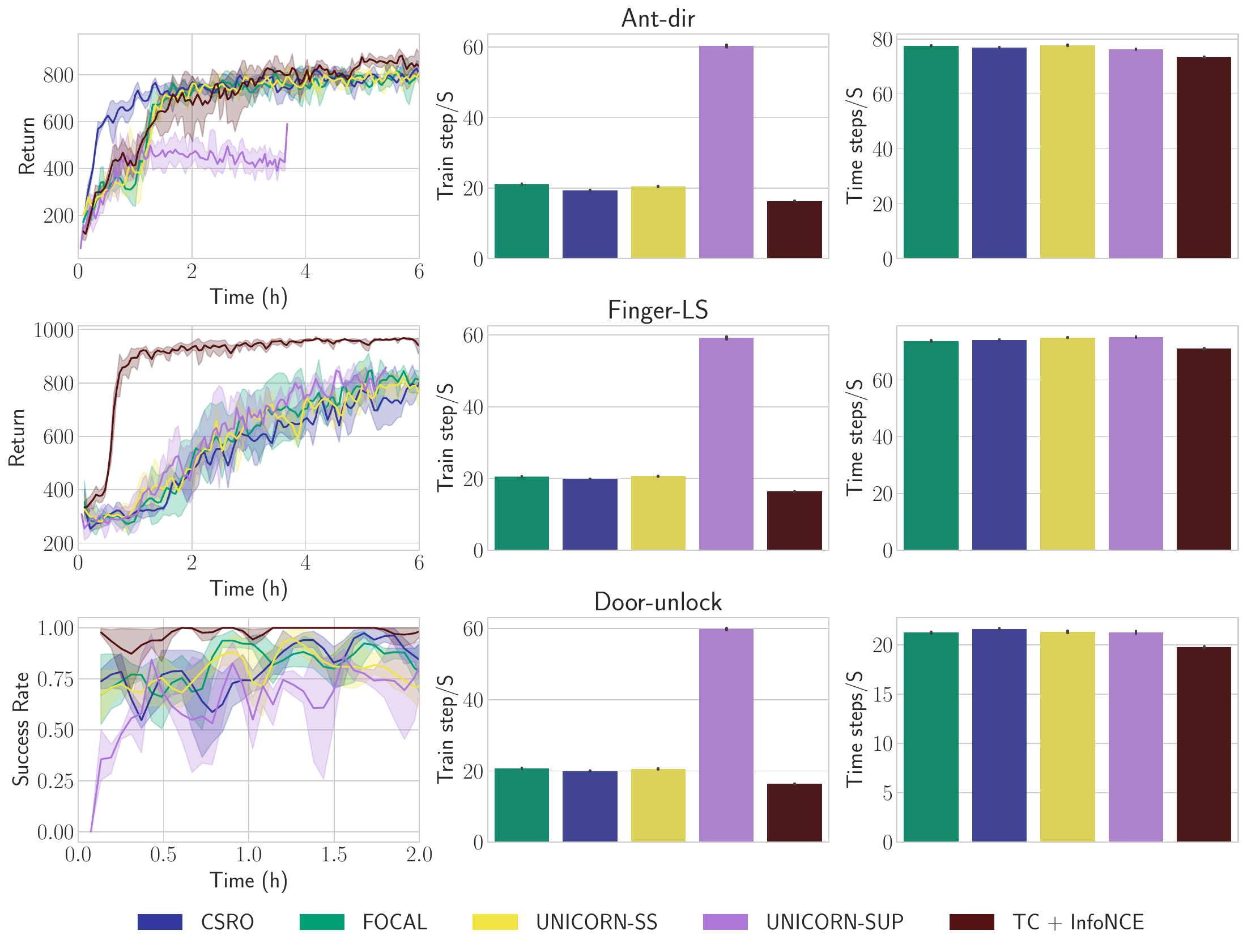}
    \caption{{
    Comparing the computation cost. \textbf{Left:} Few-shot performance on in-distribution tasks with training time.
    \textbf{Middle:} Number of training steps (full backpropagation and updating the networks) per second. 
    \textbf{Right:} Number of time steps per second during testing. 
    \textbf{Our contextual world model is computationally more expensive during training, yet it converges faster.} 
    Mapping the observation space to the latent space adds insignificant computational overhead during testing. 
    Results are averaged over 3 random seeds while considering the standard deviation. 
    }}
    \label{fig:time_comparison}
\end{figure}

\newpage

\end{document}